%% file: main.tex
\documentclass[11pt]{article}

\usepackage{fullpage}
\usepackage{natbib}

\usepackage{breakcites}

\usepackage[utf8]{inputenc} % allow utf-8 input
\usepackage[T1]{fontenc}    % use 8-bit T1 fonts

\usepackage{url}            % simple URL typesetting
\usepackage{booktabs}       % professional-quality tables
\usepackage{amsfonts}       % blackboard math symbols
\usepackage{nicefrac}       % compact symbols for 1/2, etc.
\usepackage{microtype}      % microtypography

\usepackage{afterpage}
\usepackage{subcaption} 

\usepackage{amssymb,amsthm,amsmath}
\usepackage{array,graphicx}
\usepackage{comment}
\usepackage{wrapfig}

\usepackage{authblk}

\usepackage{algorithm}
\usepackage[noend]{algpseudocode}
\usepackage{amsmath} 
\usepackage{amssymb} 

\usepackage{amsthm}
\usepackage{mathtools} 
\usepackage{stmaryrd}
\usepackage{bm}
\usepackage{tikz}
\usetikzlibrary{calc}
\usetikzlibrary{decorations}
\usepackage{longtable}
\usepackage{multirow}
\usepackage{makecell}

\usepackage{thmtools, thm-restate}
\usepackage{tikz}
\usepackage{pgfplots}
\usetikzlibrary{arrows}
\usetikzlibrary{positioning}
\usetikzlibrary{decorations.text}
\usetikzlibrary{decorations.markings}
\usetikzlibrary{decorations.shapes}
\usetikzlibrary{pgfplots.groupplots}
\usepgfplotslibrary{fillbetween}
\usetikzlibrary{patterns}
\usepackage{lscape}

\definecolor{myteal}{HTML}{2380D8}
\definecolor{myred}{HTML}{EA2222}
\definecolor{mypurple}{HTML}{9900CC}
\definecolor{mygrey}{HTML}{494949}
\definecolor{mygreen}{HTML}{008000}
\definecolor{myblack}{HTML}{000000}
\definecolor{mywhite}{HTML}{FFFFFF}
\definecolor{yahooviolet}{RGB}{66, 2, 176}
\definecolor{putblue}{RGB}{0, 103, 144}
\definecolor{allegroorange}{RGB}{255, 90, 0}
\definecolor{putred}{RGB}{204,33,69}

\usepackage{adjustbox}

\makeatletter

\newcommand*{\eifstartswith}{\@expandtwoargs\ifstartswith}
\newcommand*{\ifstartswith}[2]{%
  \if\@car#1.\@nil\@car#2.\@nil
    \expandafter\@firstoftwo
  \else
    \expandafter\@secondoftwo
  \fi}
\newcommand{\Algo}[1]{\textsc{#1}}

\renewcommand{\vec}[1]{\boldsymbol{#1}}

\newcommand{\bx}{\vec{x}}
\newcommand{\by}{\vec{y}}
\newcommand{\bz}{\vec{z}}
\newcommand{\bw}{\vec{w}}
\newcommand{\bh}{\vec{h}}
\newcommand{\bc}{\vec{c}}
\newcommand{\balpha}{\vec{\alpha}}

\newcommand{\btau}{\vec{\tau}}

\newcommand{\calX}{\mathcal{X}}
\newcommand{\calY}{\mathcal{Y}}

\newcommand{\calD}{\mathcal{D}}
\newcommand{\calQ}{\mathcal{Q}}
\newcommand{\calL}{\mathcal{L}}
\newcommand{\calH}{\mathcal{H}}
\newcommand{\calS}{\mathcal{S}}

\newcommand{\dataset}{{\cal D}}

\newcommand{\heta}{\hat{\eta}}
\newcommand{\hy}{\hat{y}}

\newcommand{\pa}[1]{\mathrm{pa}(#1)}

\newcommand{\Path}[1]{\mathrm{Path}(#1)}
\newcommand{\lenpath}{\mathrm{len}}
\newcommand{\depth}{\mathrm{depth}}

\newcommand{\precat}[1]{precision$@#1$}
\newcommand{\Precat}[1]{Precision$@#1$}
\newcommand{\precatk}{precision$@k${}}
\newcommand{\Precatk}{Precision$@k${}}
\newcommand{\psimacro}{$\Psi_{\mathrm{macro}(\bx)}$ }
\newcommand{\psimicro}{$\Psi_{\mathrm{micro}(\bx)}$ }

\newcommand{\prob}{\mathbf{P}} 

\newcommand{\reg}{\mathrm{reg}}

\newcommand{\riskcond}{\mathrm{risk}}
\newcommand{\loss}{R}

\newcommand{\lossfunc}{\ell}

\newcommand{\lonemarginalserror}{$|\eta_j(\bx) - \hat{\eta}_j(\bx)|$}

\newcommand{\mM}{\boldsymbol{M}}
\newcommand{\mW}{\boldsymbol{W}}

\newcommand{\childs}[1]{\mathrm{Ch}(#1)}

\renewcommand{\root}{r}
\newcommand{\nodes}{V}
\newcommand{\vertices}{V}
\newcommand{\tree}{T}
\newcommand{\leaves}{L}
\newcommand{\leafnode}{l}
\newcommand{\labels}{\calL}
\newcommand{\vertex}{v}
\newcommand{\node}{v}

\newcommand{\FP}{\mathrm{FP}}
\newcommand{\FN}{\mathrm{FN}}

\newcommand{\multilabel}{multi-label}
\newcommand{\multiclass}{multi-class}

\newcommand{\eurlex}{EurLex-4K}
\newcommand{\amazoncatsmall}{AmazonCat-13K}

\newcommand{\wikiten}{Wiki10-30K}
\newcommand{\deliciouslarge}{DeliciousLarge-200K}
\newcommand{\wikilshtc}{WikiLSHTC-325K}
\newcommand{\wikipedia}{WikipediaLarge-500K}
\newcommand{\amazon}{Amazon-670K}
\newcommand{\amazonlarge}{Amazon-3M}

\algnewcommand{\IIf}[1]{\State\algorithmicif\ #1\ \algorithmicthen}
\algnewcommand{\IElse}[1]{\State\algorithmicelse\ #1\ }
\algnewcommand{\IElseIf}[1]{\State\algorithmicelse \algorithmicif\ #1\  \algorithmicthen}
\algnewcommand{\EndIIf}{\unskip\ \algorithmicend\ \algorithmicif}
\algnewcommand{\IfThen}[2]{\State\algorithmicif\ #1\ \algorithmicthen\ #2\ }
\algnewcommand{\ForDo}[2]{\State\algorithmicfor\ #1\ \algorithmicdo\ #2\ }

\newcommand{\assert}[1]{\llbracket #1 \rrbracket}

\newcommand{\given}{\, | \,}

\DeclareMathOperator*{\argmax}{\arg \max}
\DeclareMathOperator*{\argmin}{\arg \min}

\newtheorem{lemma}{Lemma}

\newtheorem{proposition}{Proposition}
\newtheorem{corollary}{Corollary}
\newtheorem{remark}{Remark}
\newtheorem{definition}{Definition}

\newif\ifjmlr

\title{Probabilistic Label Trees for Extreme Multi-label Classification}

\begin{document}

\jmlrfalse

\author[1]{Kalina Jasinska-Kobus\thanks{kjasinska@cs.put.poznan.pl}}
\author[1]{Marek Wydmuch\thanks{mwydmuch@cs.put.poznan.pl}}
\author[1,2]{Krzysztof Dembczy\'{n}ski\thanks{kdembczynski@cs.put.poznan.pl}}
\author[2]{Mikhail Kuznetsov\thanks{kuznetsov@verizonmedia.com}}
\author[3]{R\'{o}bert Busa-Fekete\thanks{busarobi@google.com}}
\affil[1]{Institute of Computing Science, Poznan University of Technology, Poland}
\affil[2]{Yahoo! Research, New York, USA}
\affil[3]{Google Research, New York, USA}

\date{}
\maketitle

\begin{abstract}
  \noindent
Extreme \multilabel{} classification (XMLC) is a learning task of tagging instances
with a small subset of relevant labels chosen from an extremely large pool of possible labels. 
Problems of this scale can be efficiently handled by organizing labels as a tree, 
like in hierarchical softmax used for \multiclass{} problems. 
In this paper, we thoroughly investigate probabilistic label trees (\Algo{PLT}s) 
which can be treated as a generalization of hierarchical softmax for \multilabel{} problems. 
We first introduce the \Algo{PLT} model and discuss training and inference procedures and their computational costs. 
Next, we prove the consistency of \Algo{PLT}s for a wide spectrum of performance metrics. 
To this end, we upperbound their regret  
by a function of surrogate-loss regrets of node classifiers.
Furthermore, we consider a problem of training \Algo{PLT}s in a fully online setting, 
without any prior knowledge of training instances, their features, or labels.
In this case, both node classifiers and the tree structure are trained online.
We prove a specific equivalence between the fully online algorithm 
and an algorithm with a tree structure given in advance.
Finally, we discuss several implementations of \Algo{PLT}s 
and introduce a new one, \Algo{napkinXC}, 
which we empirically evaluate and compare with state-of-the-art algorithms.
\end{abstract}

\thispagestyle{empty}
\addtocounter{page}{-1}
\newpage    

\input{01-intro.tex}
\input{03-problem.tex}

\input{04-plt.tex}

\input{05-theory.tex}
\input{06-oplt.tex}
\input{07-implementation.tex}

\input{08-experiments.tex}
\input{09-summary.tex}

\paragraph{Acknowledgements}{The work of Kalina Jasinska-Kobus was supported by the Polish National Science Center under grant 
no.~2017/25/N/ST6/00747. Computational experiments have been performed in Poznan Supercomputing and Networking Center.
}

\clearpage
\input{10-appendix.tex}
\vskip 0.2in
\bibliography{references_jmlr}
\bibliographystyle{abbrvnat}

\end{document}

%% file: 01-intro.tex
\section{Introduction}
\label{sec:intro}

%%%%%%%%%%%%%%%%%%%%%%%%%%%%%%%%%%%%%%%%%%%%%%%%%%%%%%%%%%%%%%%%%%%%%%%%%%%%%%%%%%%%%%%%%%%%%%%%
%%% Introduction
%%%%%%%%%%%%%%%%%%%%%%%%%%%%%%%%%%%%%%%%%%%%%%%%%%%%%%%%%%%%%%%%%%%%%%%%%%%%%%%%%%%%%%%%%%%%%%%%
In modern machine learning applications, the output space can be enormous
containing millions of labels. 
Some notable examples of such problems are image and video annotation for multimedia search~\citep{Deng_et_al_2011}, 
tagging of text documents~\citep{Dekel_Shamir_2010}, 
online advertising~\citep{Beygelzimer_et_al_2009b,Agrawal_et_al_2013}, 
recommendation of bid words for online ads~\citep{Prabhu_Varma_2014}, 
video recommendation~\citep{Weston_et_al_2013}, 
or prediction of the next word in a sentence~\citep{Mikolov_et_al_2013}. 
The number of labels in all these applications is extremely large. 
Therefore, problems of this kind are often referred to as \emph{extreme classification}.
To give a more detailed example let us consider the task of tagging Wikipedia articles. 
In this case, each article is an instance, 
words appearing in the text of articles can be considered as features, 
and categories to which articles are assigned as labels. 
By creating a data set from the current content of Wikipedia,
we very easily end up with an enormous problem with millions of examples and features, 
but also with more than one million labels, 
because that many categories are used in Wikipedia today.

%%%%%%%%%%%%%%%%%%%%%%%%%%%%%%%%%%%%%%%%%%%%%%%%%%%%%%%%%%%%%%%%%%%%%%%%%%%%%%%%%%%%%%%%%%%%%%%%
%%% Motivations and contributions
%%%%%%%%%%%%%%%%%%%%%%%%%%%%%%%%%%%%%%%%%%%%%%%%%%%%%%%%%%%%%%%%%%%%%%%%%%%%%%%%%%%%%%%%%%%%%%%%

The extreme classification problems have posed new computational and statistical challenges.
A naive solution is to train an independent model 
(for example, a linear classifier) for each label individually. 
Such an approach, usually referred to as \Algo{1-vs-all}, 
has time and space complexity linear in the number of labels. 
This is, unfortunately, too costly in many practical applications.
Hence there is a need for more advanced solutions 
characterized by both good predictive performance and sublinear complexity.
To tackle extreme classification problems efficiently, one can organize labels as a tree
in which each label corresponds to one and only one path from the root to a leaf. 
A prominent example of such \emph{label tree} model is hierarchical softmax (\Algo{HSM})~\citep{Morin_Bengio_2005},
often used with neural networks to speed up computations in \multiclass{} problems. 
For example, it is commonly applied in natural language processing~\citep{Mikolov_et_al_2013}. 
Interestingly, similar algorithms have been introduced independently in many different research fields. 
In statistics, they are known as nested dichotomies~\citep{Fox_1997}, 
in multi-class regression as conditional probability trees (\Algo{CPT}s)~\citep{Beygelzimer_et_al_2009a},
and in pattern recognition as multi-stage classifiers~\citep{Kurzynski_1988}.

In this paper, we thoroughly investigate \emph{probabilistic label trees} (\Algo{PLT}s) 
which can be treated as a generalization of the above approaches to \multilabel{} problems. 
In a nutshell, \Algo{PLT}s use a label tree to factorize conditional label probabilities.
Classification of a test example relies on a sequence of decisions made by node classifiers, 
leading the test example from the root to the leaves of the tree. 
Since \Algo{PLT}s are designed for \multilabel{} classification, 
each internal node classifier decides whether or not to continue the path by moving to the child nodes. 
This is different from typical left/right decisions made in tree-based classifiers. 
Moreover, a leaf node classifier needs to make a final decision regarding
the prediction of a label associated with this leaf. 
\Algo{PLT}s use a class probability estimator in each node of the tree, 
such that an estimate of the conditional probability of a label associated with a leaf is given by 
the product of the probability estimates on the path from the root to that leaf. 
This requires specific conditioning in the tree nodes which leads to small and independent learning problems. 
For efficient prediction one needs to follow a proper tree search policy.

Let us also emphasize that extreme classification problems are characterized
by many additional issues not present in standard learning problems. 
For example, for many labels only a small number of training examples is usually available. 
This leads to the problem of long-tail labels and, ultimately, to zero-shot learning, 
where there are no training examples for some labels in the training set. 
Furthermore, training data might be of low quality as no one could go through all labels 
to verify whether they have been correctly assigned even to a single training example. 
The training information is therefore usually obtained from implicit feedback
and we often deal with the so-called counterfactual learning. 
In this paper, however, we do not touch all these challenging and actual problems, 
but focus on computational issues and statistical properties of \Algo{PLT}s.
%Thus, we also do not discuss any other methods recently introduced for solving XMLC problems. 

\subsection{Main contribution}

% \ifjmlr
% \else
% \newline
% \noindent
% \fi
\begin{sloppypar}
This article summarizes and extends our previous work 
on probabilistic label trees published in%
~\citep{Jasinska_et_al_2016, Wydmuch_et_al_2018, Busa-Fekete_et_al_2019, Jasinska_et_al_2020}. 
In the following points, we describe our main contribution and its relation to existing work. 
\end{sloppypar}

\subsubsection{The \Algo{PLT} model}

The \Algo{PLT} model has been introduced in~\citep{Jasinska_et_al_2016}. 
It uses the chain rule along the paths in the tree to factorize conditional probabilities of labels. 
%The factorization is determined by the structure of the tree and the assignment of labels to leaves. 
%Each leaf in the tree corresponds to one and only one label.
In this way, the model reduces the original \multilabel{} problem to a number of binary classification (estimation) problems.
From this point of view, it follows the learning reductions framework~\citep{learning_reductions}.
As mentioned above, similar label tree approaches have already been used for solving \multiclass{} problems.
Also for~\multilabel{} problems under the subset 0/1 loss a method, called probabilistic classifier chains,
has been introduced~\citep{Dembczynski_et_al_2010c}. 
It can be interpreted as a specific label tree, but with paths corresponding to subsets of labels. 
So, the size of the tree is exponential in the number of labels.  
Therefore, the \Algo{PLT} model is different.
It can be treated as a proper generalization of 
\Algo{HSM}, \Algo{CPT}s, or nested dichotomies to 
\multilabel{} estimation of conditional probabilities of single labels.
Since the first article~\citep{Jasinska_et_al_2016} 
the \Algo{PLT} model has been used in many other algorithms such as 
\Algo{Parabel}~\citep{Prabhu_et_al_2018}, 
\Algo{Bonsai Tree}~\citep{Khandagale_et_al_2019},
\Algo{extremeText}~\citep{Wydmuch_et_al_2018},
and \Algo{AttentionXML}~\citep{You_et_al_2019}.

There also exist other label tree approaches 
which mainly differ from \Algo{PLT}s and other methods mentioned above in that 
they do not use the probabilistic framework. 
The most similar to \Algo{PLT}s is \Algo{Homer} introduced by~\citet{Tsoumakas_et_al_2008}. 
Tournament-based or filter trees have been considered 
in \citep{Beygelzimer_et_al_2009a} and \citep{Li_Lin_2014}
to solve respectively \multiclass{} and \multilabel{} problems. 
Another example is label embedding trees introduced in~\citep{Bengio_et_al_2010}. 
None of these methods, however, has been thoroughly tested in the extreme classification setting, 
but initial studies suggest that they neither scale to problems of this scale
nor perform competitively to~\Algo{PLT}s. 

To define the \Algo{PLT} model we can follow one of two different frameworks. 
The first one is based on the standard concept of a tree. 
The second one uses the prefix codes, similarly as in~\citep{Morin_Bengio_2005},
as each path in the tree indicating a label can be seen as a code. 
We use the first framework. 
We discuss the general training schema that can be used for both batch and online learning.
%point out two different approaches to obtain probabilities of internal nodes. 
Since internal node probabilities are conditioned on parent nodes, 
\Algo{PLT}s need to properly assign training examples to nodes. 
We show an efficient procedure for this assignment which follows a bottom-up strategy. 
For prediction, we consider two tree search procedures.
The first one stops exploring a tree 
whenever the probability of reaching a given node is less than a given threshold. 
The second one uses a priority queue in order to find 
labels with the highest estimates of conditional label probabilities.

\citet{Busa-Fekete_et_al_2019} discuss in depth the computational complexity of \Algo{PLT}s.
The training complexity of node classifiers, as well as the prediction complexity,
obviously depends on the tree structure. 
They prove that building an optimal tree structure in terms of computational cost is an NP-hard problem.
Here, we include some of their results, which show that training of node classifiers 
(with a tree structure given in advance) can be done in logarithmic time under additional assumptions
concerning the tree structure and the maximum number of labels per training instance.
Moreover, with additional assumptions on estimates of label probabilities
also the prediction time is logarithmic in the number of labels. 
This result is not trivial as a prediction method needs to use a tree search algorithm
which, in the worse case, explores the entire tree leading to linear complexity.

\subsubsection{Consistency and regret bounds}

Our main theoretical results concern the consistency of \Algo{PLT}s 
for a wide spectrum of performance metrics. 
We show this by following the learning reductions framework~\citep{learning_reductions}.
We upper bound the regret of a \Algo{PLT} by a function of surrogate regrets of node classifiers. 
We first prove bounds for the $L_1$ estimation error of label conditional probabilities. 
These results are the building blocks for all other results, 
as optimal predictions for many performance metrics can be determined 
through conditional probabilities of labels%
~\citep{Dembczynski_et_al_2010c,Kotlowski_Dembczynski_2016,Koyejo_et_al_2015, Wydmuch_et_al_2018}.
A part of these results has been first published in~\citep{Wydmuch_et_al_2018}. 
They follow similar results obtained by~\citep{Beygelzimer_et_al_2009b} for \multiclass{} problems,
however, our proofs seem to be simpler and 
the bounds tighter as they weight errors of tree nodes by 
a probability mass of examples assigned to these nodes.
We also consider a wider spectrum of surrogate loss functions minimized in tree nodes,
namely a class of strongly proper composite losses~\citep{Agarwal_2014},
which includes squared error loss, squared hinge loss, logistic loss, or exponential loss.

Next, we show the regret bounds for generalized performance metrics of the linear-fractional form. 
This class of functions contains among others Hamming loss, micro and macro F-measure. 
The optimal prediction for metrics from this class relies on thresholding 
conditional probabilities of labels. 
For some metrics, this threshold has a simple form independent of a concrete data distribution 
(for example, it is 0.5 for Hamming loss),
or it requires some easy-to-estimate quantities such as priors of labels. 
Unfortunately, for metrics such as F-measures, 
an additional step of threshold tuning on a validation set is required.
The bounds for \Algo{PLT}s are based on results obtained by~\citep{Kotlowski_Dembczynski_2017}
for the \Algo{1-vs-All} approach.
%We have appropriately adjusted them to our setting.
The analysis can also be seen as 
a theoretical extension of the results published in~\citep{Jasinska_et_al_2016},
where we have discussed algorithmic challenges of the macro F-measure optimization in XMLC.

The last metric considered is \precatk. 
The regret bounds presented here extend the results published before in~\citep{Wydmuch_et_al_2018}.
We first prove the Bayes optimal prediction for \precat{k}
to be a set of $k$ labels with the highest conditional probabilities. 
Next, by using the definition of the regret and the bound for the $L_1$ estimation error 
we obtain conditional and unconditional bounds for \precatk.
This result combined with the priority-queue-based inference 
shows that \Algo{PLT}s are well-tailored for this metric. 

We also study the relation of \Algo{PLT}s to \Algo{HSM}. 
We show that the former are indeed a proper generalization of the latter to \multilabel{} problems.
Namely, for \multiclass{} distribution \Algo{PLT}s reduce to the \Algo{HSM} model.
Moreover, we show that a specific heuristic, 
often used in the deep network community to adapt \Algo{HSM} to \multilabel{} classification,
is not consistent in terms of $L_1$ estimation error and \precatk. 
This heuristic, 
used for example in \Algo{fastText}~\citep{Joulin_et_al_2016} 
and \Algo{Learned Tree}~\citep{Jernite_et_al_2017},
randomly picks one of labels from a multi-label training example and 
treats the example as a multi-class one.
Interestingly, this specific reduction to \multiclass{} classification is consistent for recall$@k$ 
as shown in \citep{Menon_et_al_2019}.

\subsubsection{Online probabilistic label trees}

We also consider a challenging problem of training \Algo{PLT}s in a fully online setting, 
without any prior knowledge about the number of training instances, their features, and labels.
In this case, not only node classifiers but also the tree structure is trained online.
This framework is similar to the one of \Algo{CPT}s 
for \multiclass{} problems~\citep{Beygelzimer_et_al_2009b}.
To formalize this setting, we define two properties that a fully online algorithm should satisfy.
The first one is a specific equivalence condition between the fully online algorithm 
and an incremental algorithm operating on a tree structure given in advance.
The second one concerns the relative complexity of the fully online algorithm.
We prove that a general algorithmic framework we introduce satisfies both properties. 
Here, we use a simple tree building policy which constructs a complete binary tree. 
Recently, we have experimented with a more sophisticated method~\citep{Jasinska_et_al_2020}
leading to better results.

\subsubsection{napkinXC and comparison to the state-of-the-art}

\begin{sloppypar}
Beside our theoretical findings, 
we discuss several existing implementations of the general \Algo{PLT} scheme,
such as \Algo{XMLC-PLT}~\citep{Jasinska_et_al_2016}, 
\Algo{PLT-vw},
\Algo{Parabel}~\citep{Prabhu_et_al_2018}, 
\Algo{Bonsai Tree}~\citep{Khandagale_et_al_2019},
\Algo{extremeText}~\citep{Wydmuch_et_al_2018},
and \Algo{AttentionXML}~\citep{You_et_al_2019}.
We compare them from the perspective of 
feature and model representation, 
batch and incremental learning, 
prediction algorithms,
or tree structure choice.
We also introduce a new library, referred to as \Algo{NapkinXC}, 
which can be easily adapted to different settings thanks to its modular design. 
Finally, in a wide empirical study we thoroughly analyze different instances of \Algo{PLT}s 
and relate their results to other state-of-the-art algorithms such as 
\Algo{FastXML}~\citep{Prabhu_Varma_2014},
\Algo{PfastreXML}~\citep{Jain_et_al_2016}
\Algo{DiSMEC}~\citep{Babbar_Scholkopf_2017}, 
and \Algo{PDDSparse}~\citep{Yen_et_al_2017}.
The first two algorithms are decision tree approaches adapted to the XMLC setting.
Their main difference to label trees is that they split the feature space, 
not the set of labels. 
The last two algorithms are efficient variants of the \Algo{1-vs-All} approach,
which are known to obtain the top empirical performance. 
Our experiments indicate that \Algo{PLT}s are very competitive 
reaching the best \precat{1} on majority of benchmark data sets,
being at the same time even thousand times faster 
in training and prediction than the \Algo{1-vs-All} approaches.
\end{sloppypar}

%%%%%%%%%%%%%%%%%%%%%%%%%%%%%%%%%%%%%%%%%%%%%%%%%%%%%%%%%%%%%%%%%%%%%%%%%%%%%%%%%%%%%%%%%%%%%%%%
%%% Paper structure
%%%%%%%%%%%%%%%%%%%%%%%%%%%%%%%%%%%%%%%%%%%%%%%%%%%%%%%%%%%%%%%%%%%%%%%%%%%%%%%%%%%%%%%%%%%%%%%%

\subsection{Organization of the paper}

%First we discuss the related work and situate our approach in the context of other works. 
In Section~\ref{sec:problem-setting} we formally state the XMLC problem. 
%and present theoretical framework for insights related to optimization of performance measures used in extreme \multilabel{} classification. 
Section~\ref{sec:plt} defines the \Algo{PLT} model and discusses training and prediction procedures,
as well their computational complexities. 
Section~\ref{sec:analysis} contains the theoretical analysis of \Algo{PLT}s, 
which includes regret bounds and relation to \Algo{HSM}.
In Section~\ref{sec:oplt} we discuss online probabilistic label trees. 
Section~\ref{sec:implementation} provides an in-depth discussion 
on the implementation choices of the \Algo{PLT} model. 
The experimental results are presented in Section~\ref{sec:experiments}. 
The last section concludes the paper.

%% file: 03-problem.tex
\section{Problem statement}
\label{sec:problem-setting}

Let $\calX$ denote an instance space, and let $\labels = [m]$ be a finite set of $m$ class labels. 
We assume that an instance $\bx \in \calX$ is associated with a subset of labels $\labels_{\bx} \subseteq \calL$ 
(the subset can be empty); 
this subset is often called the set of \emph{relevant} or \emph{positive} labels,
while the complement $\labels \backslash \labels_{\bx}$ is considered as \emph{irrelevant} or \emph{negative} for $\bx$. 
We identify the set $\labels_{\bx}$ of relevant labels with the binary vector $\by = (y_1,y_2, \ldots, y_m)$, 
in which $y_j = 1 \Leftrightarrow j \in \labels_{\bx}$. 
By $\calY = \{0, 1\}^m$ we denote the set of all possible label vectors.
We assume that observations $(\bx, \by)$ are generated independently and identically
according to a probability distribution $\prob(\bx, \by)$ defined on $\calX \times \calY$. 
Notice that the above definition concerns not only multi-label classification, 
but also multi-class (when $\|\by\|_1=1$) 
and $k$-sparse multi-label (when $\|\by\|_1\le k$) problems as special cases. 
In case of extreme multi-label classification (XMLC) we assume $m$ to be a large number 
(for example $\ge 10^5$), 
and $k$ to be much smaller than $m$, $k \ll m$.%
\footnote{We use $[n]$ to denote the set of integers from $1$ to $n$, and $\|\bx\|_1$ to denote the $L_1$ norm of $x$.}

The problem of extreme \multilabel{} classification can be defined as finding a \emph{classifier} $\bh(\bx) = (h_1(\bx), h_2(\bx),\ldots, h_m(\bx))$, from a function class $\calH^m: \calX \rightarrow \mathbb{R}^m$, that minimizes the \emph{expected loss} or \emph{risk}:  %\todo{add citations to the framework we use}
$$
\loss_\ell(\bh) = \mathbb{E}_{(\bx,\by) \sim \prob(\bx,\by)} (\ell(\by, \bh(\bx))\,,
$$
where $\ell(\by, \hat{\by})$ is the  (\emph{task}) \emph{loss}.
The optimal classifier,  the so-called \emph{Bayes classifier},  for a given loss function $\ell$ is:
$$
\bh^*_\ell = \argmin_{\bh}  \loss_\ell(\bh) \,.
$$
The \emph{regret} of a classifier $\bh$ with respect to $\ell$ is defined as:
 $$
\reg_\ell(\bh) = \loss_\ell(\bh) - \loss_\ell(\bh_{\ell}^*) = \loss_\ell(\bh) - \loss_\ell^* \,.
$$
The regret quantifies the suboptimality of $\bh$ compared to the optimal classifier $\bh_{\ell}^*$. The goal could be then defined as finding $\bh$ with a small regret, ideally equal to zero.

We are interested in multi-label classifiers that estimate conditional probabilities of labels, $\eta_j = \prob(y_j = 1 \vert \bx)$, $j \in \calL$, 
as accurately as possible, that is, with possibly small $L_1$-estimation error, 
\begin{equation}
|\eta_j(\bx) - \heta_j(\bx)| \,,   
\label{eqn:l1-error}
\end{equation}
where $\heta_j(\bx)$ is an estimate of $\eta_j(\bx)$. 
This statement of the problem is justified by the fact that optimal predictions 
for many performance measures used in multi-label classification, 
such as the Hamming loss, precision@k, and the micro- and macro F-measure, 
are determined through the conditional probabilities of labels%
~\citep{Dembczynski_et_al_2010c,Kotlowski_Dembczynski_2016,Koyejo_et_al_2015}. 
In Section~\ref{sec:analysis} we derive statistical guarantees for \Algo{PLT}s, in a form of regret bounds, 
for a wide spectrum of losses including those mentioned above. 
Here, we only mention that to obtain estimates $\heta_j(\bx)$ 
one can use the label-wise logistic loss, 
sometimes referred to as binary cross-entropy: 
$$
\ell_{\log}(\by, \bh(\bx))  = \sum_{j=1}^m \ell_{\log}(y_j, h_j(\bx)) = \sum_{j=1}^m  \left ( y_j \log(h_j(\bx)) + (1-y_j) \log(1-h_j(\bx)) \right) \,.
$$
The expected label-wise logistic loss for a single $\bx$ (the so-called \emph{conditional risk}) is:
$$
\mathbb{E}_{\by} \ell_{\log}(\by, \bh(\bx)) =  \sum_{j=1}^m \mathbb{E}_{\by}{\ell_{\log}(y_j, h_j(\bx))} = \sum_{j=1}^m \loss_{\log}(h_j(\bx) \given \bx) \,. %\\ 
$$
This is a sum of conditional risks of binary problems under the logistic loss. 
Therefore, it is easy to show that the pointwise optimal prediction for the $j$-th label is given by:
$$
h_j^*(\bx)  = \argmin_h \loss_{\log}(h_j(\bx)\given \bx) = \eta_j(\bx) \,.
$$
%By taking the derivative of $\loss_{\log}(h_j(\bx))$, defined as:
%$$
%-\eta_j(\bx) \log(h_j(\bx)) - (1-\eta_j(\bx) ) \log(1 \!-\! h_j(\bx)) \,,
%$$
%with respect to $h_j(\bx)$ and equating it to zero we get: 
%$$
%h_j^*(\bx)  = \eta_j(\bx) \,.
%$$

The above loss function corresponds to the vanilla \Algo{1-vs-All} approach. 
Unfortunately, it is too costly in the extreme setting 
as training and prediction is linear in the number of labels. 
In the following sections, we discuss an alternative approach based on label trees, 
which estimates the marginal probabilities with a competitive accuracy, 
but in a much more efficient way.

%% file: 04-plt.tex
\section{Probabilistic label trees (\Algo{PLT}s)}
\label{sec:plt}

Probabilistic label trees (\Algo{PLT}s) follow a label-tree approach 
to efficiently estimate the marginal probabilities of labels. % in \multilabel{} problems. 
They reduce the original problem to a set of binary estimation problems organized in the form of a rooted, 
leaf-labeled tree with $m$ leaves. 
We denote a single tree by $\tree$, a root node by $\root_T$, and the set of leaves by $\leaves_T$. 
The leaf $\leafnode_j \in \leaves_T$ corresponds to the label $j \in \labels$. 
The set of leaves of a (sub)tree rooted in an inner node $v$ is denoted by $\leaves_v$. 
The set of labels corresponding to leaf nodes in $\leaves_v$ is denoted by $\calL_{v}$.
The parent node of $v$ is denoted by $\pa{\node}$, and the set of child nodes by $\childs{\node}$. 
A pre-leaf is a parent node whose all children are leaves.
The path from node $v$ to the root is denoted by $\Path{\node}$. 
The length of the path, that is, the number of nodes on the path, is denoted by $\lenpath_\node$. 
The set of all nodes is denoted by $\nodes_T$. 
The degree of a node $\node \in \nodes_T$,  being the number of its children, is denoted by $\deg_\node=|\childs{\node}|$. 
An example of a label tree is given in Figure~\ref{fig:plt-model-z}.

\begin{figure}[H]
    \centering
    \input{pics/plt-model-z.tex}
    \caption{An example of a label tree $T$ with labels $\labels = \{y_1, y_2, y_3, y_4\}$ assigned to the leaf nodes.}
    \label{fig:plt-model-z}
\end{figure}
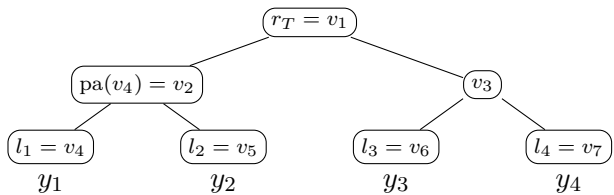

% Let us notice that 
The assignment of labels to tree leaves corresponds to encoding them by a prefix code, 
as any such code can be given in the form of a tree. 
Under the coding, each label $y_j$ is uniquely represented by a code word $\bc_j = (1, c_{j1}, \ldots, c_{jd})$ %\in \calC$
corresponding to a path from the root to leaf $l_j$.
%where $\calC$ is a set of all code words.
Obviously, the length of the code equals the length of the path, 
that is, $|\bc_j| = d+1 = \lenpath_{\leafnode_j}$. 
The zero position of the code allows one to indicate a situation in which there is no label assigned to an instance. 
Therefore, each label code starts with~1. 
For $c_{ji} \in \{0,1\}$, the code and the label tree are binary. 
In general, the code alphabet can contain more than two symbols. 
Furthermore, $c_{ji}$s can take values from different sets of symbols depending on the prefix of the code word. 
In other words, the code can result in nodes of a different arity, like in~\citep{Grave_et_al_2017} and \citep{Prabhu_et_al_2018}. 
Notice that any node $\node$ in the tree can be uniquely identified by the partial code word $\bc_{\node} = (1, c_{\node 1}, \ldots, c_{\node d_v})$. 
%The root node gets the code word $\bc_{r_T} = (1)$. 
%Partial code $\bc_{\node}$ for an instance says that there is at least one positive label in a subtree rooted in node $\node$. 
An example of the coding is visualized in Figure~\ref{fig:plt-model}.
% The conditional probability of a label can be determined by a sequence of decisions 
%made by node classifiers predicting subsequent values of the code word. 
This coding perspective has been used in the original paper introducing the \Algo{HSM} model~\citep{Morin_Bengio_2005},
as well as in some later articles~\citep{Dembczynski_et_al_2016}.  % \todo{Coding perspective used in HSM papers}
In the following, however, we mainly use the tree notation introduced in the paragraph before.
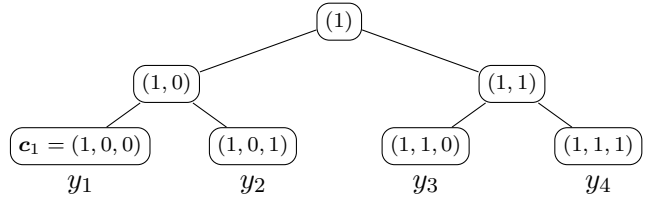
\begin{figure}[h]
    \centering
    \input{pics/plt-model.tex}
    \caption{{Example of assignment of codes to nodes and labels $\labels = \{y_1, y_2, y_3, y_4\}$.}}
    \label{fig:plt-model}
\end{figure}

%The probability of a given label is determined by a sequence of decisions made by node classifiers that predict subsequent values of the code word. 

A \Algo{PLT} uses tree $T$ to factorize the conditional probabilities of labels, 
$\eta_j(\bx) = \prob(y_j = 1 \vert \bx)$, for all $j \in \labels$. 
To this end let us define for every $\by$ a corresponding vector $\bz$ of length $|V_T|$,%
\footnote{Note that $\bz$ depends on $T$, but $T$ will always be obvious from the context.} 
whose coordinates, indexed by $v \in V_T$, % \footnote{We will also use leaves $v \in L_T$ to index the elements of vector $\by$.}
are given by: 
\begin{equation}
z_v = \assert{\textstyle \sum_{j \in \labels_v} y_{j} \ge 1} \,, \quad \textrm{or equivalently by~} z_v = \textstyle \bigvee_{j \in \labels_v} y_{j} \, .
\label{eqn:z}
\end{equation}
In other words, the element $z_v$ of $\bz$, corresponding to the node $v \in V_{\tree}$, 
is set to one iff $\by$ contains at least one label in $\labels_v$.
With the above definition, it holds based on the chain rule that for any node $v \in V_T$:
\begin{equation}
\eta_v(\bx) = \prob(z_v = 1 \given \bx) =  \prod_{v' \in \Path{v}} \eta(\bx, v') \,,
\label{eqn:plt_factorization}
\end{equation}
where $\eta(\bx, v) = \prob(z_v = 1 \vert z_{\pa{v}} = 1, \bx)$ for non-root nodes, 
and $\eta(\bx, v) = \prob(z_v = 1 \given \bx)$ for the root.
Notice that for leaf nodes we get the conditional probabilities of labels, that is, 
\begin{equation}
\eta_{\leafnode_j}(\bx) = \eta_j(\bx) \,, \quad \textrm{for~} l_j \in L_T \,.
\label{eqn:plt_leaf_prob}
\end{equation}
Remark that (\ref{eqn:plt_factorization}) can also be stated as recursion:
\begin{equation}
\eta_v(\bx) = \eta(\bx, v) \eta_{\pa{\node}}(\bx) \,,
\label{eqn:plt_factorization_recursion}
\end{equation}
with the base case $\eta_{\root_T}(\bx) = \eta(\bx, \root_T) = \prob(z_{\root_T} = 1 \given \bx)$.

As we deal here with \multilabel{} distributions, 
the relation between probabilities of a parent node and its children is not obvious. 
The following result characterizes this relation precisely. 
\begin{proposition}
\label{prop:node_cond_prob_interval}
For any $T$, $\prob(\by \vert \bx)$, and internal node $v \in V_T \setminus L_T$ we have that:
\begin{equation}
\sum_{v' \in \childs{v}} \eta(\bx, v') \ge 1 \,.
\label{eqn:sum_of_child_nodes}
\end{equation}
Moreover, the probability $\eta_v(\bx)$ satisfies:
\begin{equation}
\max \left \{ \eta_{v'}(\bx): v' \in \childs{v}\right \} \le \eta_v(\bx) \le \min \left \{ 1, \textstyle \sum_{v' \in \childs{v}} \eta_{v'}(\bx) \right \}\,.
\label{eqn:pop_child_cond_prob}
\end{equation}
\end{proposition}
\begin{proof}
We first prove the first inequality. 
From the definition of tree $T$ and $z_v$ we have that if $z_{v} = 1$, then there exists at least one $v' \in \childs{v}$ for which $z_{v'} = 1$. 
This gives that $\sum_{v' \in \childs{v}} z_{v'} \ge 1$, if $z_{v} = 1$. 
By taking expectation and recalling that $\eta(\bx, v') = \prob(z_{v'} = 1 \vert z_{v} = 1, \bx)$, for $v' \in \childs{v}$, we get:
$$
\sum_{v' \in \childs{v}} \eta(\bx, v') \ge 1 \,.
$$
To prove (\ref{eqn:pop_child_cond_prob}) we use the above result and (\ref{eqn:plt_factorization_recursion}). 
Obviously $\eta(\bx, v) \in [0,1]$, for any $v \in V_T$, therefore $\eta_{v'}(\bx) \le \eta_{v}(\bx)$ for every $v' \in \childs{v}$. 
Moreover, from (\ref{eqn:plt_factorization_recursion}) we have that 
$$
\eta(\bx, v') = \eta_v'(\bx)/ \eta_v(\bx) \,, 
$$
for every  $v' \in \childs{v}$. 
Substituting this to (\ref{eqn:sum_of_child_nodes}) gives $\eta_v(\bx) \le \sum_{v' \in \childs{v}} \eta_{v'}(\bx)$. 
Since we obviously have $\eta_{v}(\bx) \le 1$, we get the final result. 
\end{proof}

\subsection{Training}

Let $\dataset = \{ (\bx_{i},\by_{i})\}_{i=1}^n$ be a training set 
consisting of $n$ tuples consisting of a feature vector $\bx_i \in \calX$ and a label vector $\by_i \in \calY$.
Depending on the context we also use the set notation for label vectors, 
that is, $\by_i \equiv \calL_{\bx_i}$.
From factorization (\ref{eqn:plt_factorization}) we see 
that we need to train classifiers estimating $\eta(\bx, v)$, for $v \in V_T$.   
We use a function class $\calH^1_{\textrm{prob}} : \calX \mapsto [0,1]$ 
which contains probabilistic classifiers of choice, for example, logistic regression. 
We assign a classifier from $\calH^{1}_{\textrm{prob}}$ to each node of the tree $T$. 
We index this set of classifiers by elements of $V_T$ as $H = \{ \heta(v) \in \calH^{1}_{\textrm{prob}} : v \in V_T \}$. 
We denote by $\heta(\bx, v)$ the prediction made by $\heta(v)$ for some $\bx$, 
which is the estimate of $\eta(\bx, v)$. 
The training algorithm for a \Algo{PLT} is given in Algorithm~\ref{alg:plt-learning}. 
For simplicity we discuss here a batch procedure, but an online counterpart can be easily obtained based on it
(see Section~\ref{sec:oplt}).

To train probabilistic classifiers $\heta(v)$, $v \in V_T$, %in all nodes of a tree $T$, 
we need to properly filter training examples as given in (\ref{eqn:plt_factorization}).
The \textsc{Train} procedure first initializes the sets of training example in all nodes of $T$. 
Then, for each training example it identifies the set of \emph{positive} and \emph{negative nodes},
that is, the nodes for which the training example is treated respectively as positive or negative. 
The \textsc{AssignToNodes} method, given in Algorithm~\ref{alg:plt-assign}, 
initializes the positive nodes to the empty set and the negative nodes to the root node (to deal with $\by$ of all zeros). 
Next, it traverses the tree from the leaves, corresponding to the labels of the training example,
to the root adding the visited nodes to the set of positive nodes. 
It also removes each visited node from the set of negative nodes, if it has been added to this set before. 
All children of the visited node, which are not in the set of positive nodes, are then added to the set of negative nodes. 
If the parent node of the visited node has already been added to positive nodes, the traversal of this path stops. 
Using the result of the \textsc{AssignToNodes} method the algorithm distributes the training example to the corresponding nodes. 
Finally, a probabilistic classifier $\heta(v)$ is trained in each node $v$ using algorithm $A$ of choice. 
Notice that training of each node classifier can be performed simultaneously as an independent task.
%Note that all node classifiers can be trained simultaneously as they are completely independent. 
The output of the algorithm is a set of probabilistic classifiers $H$.

\begin{algorithm}[H]
\caption{\Algo{PLT.Train}$(T, A, \dataset)$}%Learning of a Probabilistic Label Tree}
\label{alg:plt-learning}
\begin{small}
\begin{algorithmic}[1]
%\State \textbf{input:} a label tree $T$, a learning algorithm $A$, and a training set $\mathcal{S}$
%\State \textbf{output:} a set of probability estimation classifiers $\mathcal{Q}$
\State $H = \emptyset$ \Comment{Initialize a set of node probabilistic classifiers}
\For{each node $v \in V_T$} \Comment{For each node in the tree}
\State $\dataset(v) = \emptyset$ \Comment{Initialize its set of training example in $\calD$}
\EndFor
\For{$i=1 \to n$}  \Comment{For each training example}
\State $(P, N) = \mathrm{\textsc{AssignToNodes}}(T, \bx_i, \calL_{\bx_i})$ \Comment{Compute its positive and negative nodes}
\For{$v \in P$} \Comment{For all positive nodes}
\State $\dataset(v) = \dataset(v) \cup \{(\bx_i, z_v = 1)\}$  \Comment{Add the modified example to the training set of node $t$}
\EndFor
\For{$v \in N$} \Comment{For each negative node}
\State $\dataset(v) = \dataset(v) \cup \{(\bx_i, z_v = 0)\}$ \Comment{Add the modified example to the training set of node $t$}
\EndFor
\EndFor
\For{each node $v \in T$} \Comment{For all nodes in the tree}
\State $\heta(v) = A(\dataset(v))$, $H = H \cup \{\heta(v)\}$ \Comment{Train a node classifier with algorithm $A$}  
\EndFor
\State \textbf{return} $H$ \Comment{Return the set of node probabilistic classifiers} 
\end{algorithmic}
\end{small}
\end{algorithm} 
\begin{algorithm}[H]
\caption{\Algo{PLT.AssignToNodes}$(T, \bx, \calL_{\bx})$}
\label{alg:plt-assign}
\begin{small}
\begin{algorithmic}[1]
\State $P = \emptyset$, $N = \{r_T\}$ \Comment{Initialize sets of positive and negative nodes}
\For{$j \in \calL_{\bx}$} \Comment{For all labels of the training example}
\State $v = \ell_j$  \Comment{Set $v$ to a leaf corresponding to label $j$}
\While{$v$ not null \textbf{and} $v \not \in P$} \Comment{On a path to the root or the first positive node (excluded)}
\State $P = P \cup \{v\}$ \Comment{Assign a node to positive nodes} 
\State $N = N \setminus \{v\}$ \Comment{Remove the node from negative nodes if added there before} 
\For{$v' \in \childs{v}$} \Comment{For all its children}
\If{$v' \not \in P$} \Comment{If a child is not a positive node}
\State $N = N \cup \{v'\}$ \Comment{Assign it to negative nodes} 
\EndIf
\EndFor
\State $v = \pa{v}$ \Comment{Move up along the path} 
\EndWhile
\EndFor
\State \textbf{return} $(P,N)$ \Comment{Return a set of positive and negative nodes for the training example}
\end{algorithmic}
\end{small}
\end{algorithm}

\subsection{Prediction}
\label{subsec:plt-prediction}

For test example $\bx$, the estimate of the marginal probability of label $j$  can be readily computed 
as a product of probability estimates on the path from the root to leaf $l_j \in L_T$:
\begin{equation}
\heta_j(\bx) = \prod_{v \in \Path{l_j}} \heta(\bx, v) \,,
\label{eqn:plt-factorization-prediction}
\end{equation}
where we assume $\heta(\bx, v) \in [0,1]$.
Obviously, the recursive dependency (\ref{eqn:plt_factorization_recursion}) also holds for the estimates. 
We have, for any $\node \in \nodes_{\tree}$, that:
\begin{equation}
\heta_v(\bx) = \heta(\bx, v) \heta_{\pa{v}}(\bx) \,,
\label{eqn:plt_estimates_factorization_recursion}
\end{equation}
with the base case $\heta_{\root_{\tree}}(\bx) = \heta(\bx, \root_{\tree})$. 
However, the estimates may not satisfy property (\ref{eqn:sum_of_child_nodes}) 
given in Proposition~\ref{prop:node_cond_prob_interval}. 
Namely, it may not hold, for $\node \in \nodes_T$, that: 
$$
\sum_{v' \in \childs{v}} \heta(\bx, v') \ge 1 \,,
$$
since the node classifiers are trained independently from each other.  
The remedy relies on an additional normalization step during prediction, 
which may take the following form, for each child node $\node' \in \childs{\node}$:
%\begin{equation}
\begin{eqnarray}
\heta(\bx, \node') & \leftarrow & \frac{ \heta(\bx, \node')}{\sum_{v'' \in \childs{v}} \heta(\bx, \node'')}\,, \quad \textrm{if~} \sum_{v'' \in \childs{v}} \heta(\bx, \node'') < 1 \,.
\label{eqn:plt-normalization}
\end{eqnarray}
Nevertheless, this normalization is not always necessary. 
The theoretical results presented in Section~\ref{sec:analysis} hold without it. 
Also empirically an algorithm without normalization performs similarly, being often slightly better.
However, the complexity analysis of the prediction algorithms, presented later, 
requires this normalization.

The estimation of the label probabilities is only a part of the solution 
as we usually need a prediction algorithm that 
delivers a set of labels being as similar as possible to the actual one 
with respect to some application-specific loss function. 
Below we introduce two such algorithms based on tree search.
%for which we present statistical guarantees in Section~\ref{sec:analysis}. 
%
Let us first consider a prediction algorithm which finds, for a test example $\bx$, all labels such that:
$$
\heta_j(\bx) \ge \tau_j \,, \quad j \in \calL \,,  
$$
where $\tau_j \in [0,1]$ are label-specific thresholds. 
The threshold-based predictions are inline with the theoretical analysis given in the next section,
as for many performance metrics they lead to optimal decisions.  
Here, we present the algorithmic solution assuming 
that the particular values of $\tau_j$, for all $j \in \calL$, 
have been provided. %
Consider the tree search procedure presented in Algorithm~\ref{alg:tbs_prediction}. 
It starts with the root node 
and traverses the tree by visiting the nodes $v \in V_T$ for which $\heta_{pa(v)}(\bx)\geq \tau_v$, 
where $\tau_v = \min \{\tau_j : \leafnode_j \in \leaves_v \}$. 
It uses a simple stack $\calQ$ to guide the search. 
Obviously, the final prediction consists of labels corresponding to the visited leaves 
for which $\heta_{\ell_j}(\bx)\geq \tau_j$.
\begin{algorithm}[ht]
\caption{\Algo{PLT.PredictWithThresholds}$(T, H, \vec{\tau}, \bx)$}
\label{alg:tbs_prediction}
\begin{small}
\begin{algorithmic}[1] 
\State $\hat\by = \vec{0}$, $\calQ = \emptyset$ \Comment{Initialize prediction vector to all zeros and a stack}
\State $\calQ\mathrm{.add}((r_T, \heta(\bx, r_T)))$ \Comment{Add the tree root with the corresponding estimate of probability}
\While{$\calQ \neq \emptyset$}  \Comment{In the loop}
	\State $(v, \heta_v(\bx)) = \calQ\mathrm{.pop}()$ \Comment{Pop an element from the stack}
	\If{$\heta_v(\bx) \ge \tau_v$} \Comment{If the probability estimate is greater or equal $\tau_v$}
	\If{$v$ is a leaf}  \Comment{If the node is a leaf}
		\State $\hy_v = 1$ \Comment{Set the corresponding label in the prediction vector}
	\Else \Comment{If the node is an internal node}
	\For{$v' \in \childs{v}$} \Comment{For all child nodes}
	    \State $\heta_{v'}(\bx) = \heta_v(\bx) \times \heta(\bx,v')$ \Comment{Compute $\heta_{v'}(\bx)$ using $\heta(v') \in H$}
		\State $\calQ\mathrm{.add}((v', \heta_{v'}(\bx)))$  \Comment{Add the node and the computed probability estimate}
	\EndFor	
	\EndIf 
	\EndIf 
\EndWhile
\State \textbf{return} $\hat\by$ \Comment{Return the prediction vector}
\end{algorithmic}
\end{small}
\end{algorithm}

The next algorithm finds the top $k$ labels with the highest $\heta_j(\bx)$, $j \in \labels$.
Consider a variant of uniform-cost search~\citep{Russell_Norvig_2009} presented in Algorithm~\ref{alg:ucs_prediction}.
It uses, in turn, a priority queue $\calQ$ to guide the search. 
In each iteration a node with the highest $\heta_v(\bx)$ is popped from the queue. 
If the node is not a leaf, then its each child node $\node'$ is added to the priority queue with its $\heta_{\node'}(\bx)$. 
Otherwise, a label corresponding to the visited leaf is added to the prediction. 
If the number of predicted labels is equal $k$, then the procedure returns prediction $\hat\by$. 
Clearly, the priority queue guarantees that $k$ labels with the highest $\heta_j(\bx)$, $j \in \labels$, are predicted. 
\begin{algorithm}[h!]
\caption{\Algo{PLT.PredictTopLabels}$(T, H, k, \bx)$}
\label{alg:ucs_prediction}
\begin{small}
\begin{algorithmic}[1] 
\State $\hat\by = \vec{0}$, $\calQ = \emptyset$, \Comment{Initialize prediction vector to all zeros and a priority queue}
\State $k' = 0$ \Comment{Initialize counter of predicted labels}
\State $\calQ\mathrm{.add}((r_T, \heta(\bx, r_T)))$ \Comment{Add the tree root with the corresponding estimate of probability}
\While{$k' < k$}   \Comment{While the number of predicted labels is less than $k$}
	\State $(v, \heta_v(\bx)) = \calQ\mathrm{.pop}()$ \Comment{Pop the top element from the queue}
	\If{$v$ is a leaf}  \Comment{If the node is a leaf}
		\State $\hy_v = 1$ \Comment{Set the corresponding label in the prediction vector}
		\State $k' = k' + 1$ \Comment{Increment the counter}
	\Else \Comment{If the node is an internal node}
	\For{$v' \in \childs{v}$} \Comment{For all its child nodes}
	    \State $\heta_{v'}(\bx) = \heta_v(\bx) \times \heta(\bx,v')$ \Comment{Compute $\heta_{v'}(\bx)$ using $\heta(v') \in H$}
		\State $\calQ\mathrm{.add}((v', \heta_{v'}(\bx)))$  \Comment{Add the node and the computed probability estimate}
	\EndFor	
	\EndIf 
\EndWhile
\State \textbf{return} $\hat\by$ \Comment{Return the prediction vector}
\end{algorithmic}
\end{small}
\end{algorithm}

Both algorithms presented above in the worst case can visit all nodes of the tree. 
However, in many practical applications they work in time logarithmic in the number of labels $m$. 
%The exhaustive complexity analysis of the algorithm is given in \citep{Busa-Fekete_et_al_2019}. 
In the following subsection, 
we briefly discuss some basic results related to the computational complexity of \Algo{PLT} algorithms.

\subsection{Computational complexity of \Algo{PLT}s}
\label{subsec:complexity_analysis}

Following \citep{Busa-Fekete_et_al_2019}, 
we define the training complexity of \Algo{PLT}s in terms of the number of nodes 
in which a training example $(\bx,\by)$ is used. 
From the definition of the tree and the \Algo{PLT} model~(\ref{eqn:plt_factorization}),
we have that each training example is used in the root, to estimate $\prob(z_{r_T} = 1\vert \bx)$, 
and in each node $v$ for which $z_{\pa{v}} = 1$, to estimate $\prob(z_v = 1\vert z_{\pa{v}} = 1, \bx)$.  
Therefore, the training cost for a single observation $(\bx, \by)$ can be given as:
\begin{equation}
c(T, \by) = 1 + \sum_{v \in V_T \setminus r_T} z_{\pa{v}} \,.
\label{eqn:learning_cost}
\end{equation}
This definition agrees with the time complexity of the \textsc{AssignToNodes} procedure, 
which is $O(c(T, \by))$ if the set operations are performed in $O(1)$ time.
From the perspective of the training complexity of node classifiers,
the above definition of cost is only justified for learning algorithms
that scale linearly in the number of training examples.
There exists, however, plenty of such algorithms
with a prominent example of stochastic gradient descent.

The next proposition determines the upper bound for the cost $c(T, \by)$.
\begin{restatable}{proposition}{costupperbound}
\label{prop:cost_upperbound}
For any tree $T$ and vector $\by$ it holds that:
$$
c(T, \by) \le 1 + \|\by\|_1 \cdot \depth_T\cdot \deg_T\,,
$$
where $\depth_T = \max_{v \in L_T} \lenpath_v - 1$ is the depth of the tree, and $\deg_T = \max_{v \in V_T} \deg_v$ is the highest degree of a node in $T$.
\end{restatable}
This result has been originally published in~\citep{Busa-Fekete_et_al_2019}. 
For completeness, we present the proof in the Appendix. 
The immediate consequence of this result is the following remark 
which states  that the training complexity of \Algo{PLT}s can scale logarithmically in the number of labels.
\begin{remark}
Consider $k$-sparse multi-label classification (for which, $\|\by\|_1  \le k$). 
For a balanced tree of constant $\deg_T=\lambda~(\ge 2)$ and $\depth_T=\log_{\lambda}{m}$, 
the training cost is $c(T, \by)=O(k\log m)$. 
\end{remark}
\noindent
Interestingly, the problem of finding the optimal tree structure in terms of the training cost is NP-hard, 
as proven in~\citep{Busa-Fekete_et_al_2019}. 
However, the balanced trees achieve a logarithmic  approximation of the optimal tree in the number of labels. 

For the prediction cost, we use a similar definition. 
We define it as the number of calls to node classifiers for a single observation~$\bx$. 
Let us first consider Algorithm~\ref{alg:tbs_prediction} with a threshold vector $\vec{\tau}$.
Its prediction cost is clearly given by: 
%In case of prediction with threshold $\tau$, Algorithm~\ref{alg:tbs_prediction}, this value is:
$$
c_{\tau}(T,\bx)= 1 + \sum\limits_{v\in V_T\setminus r_T}\assert{\heta_{pa(v)}(\bx)\geq \tau_v} = 1 + \sum\limits_{v\in V_T}\assert{\heta_{v}(\bx)\geq \tau_v} \cdot \deg_v.
$$
Analogously to Proposition~\ref{prop:cost_upperbound}, 
we determine the upper bound for the prediction cost. 
To this end, let us upper bound $\sum_{j=1}^m \heta_j(\bx)$ by a constant $\hat{P}$.
Moreover, we assume that $\heta_v(\bx)$ are properly normalized 
to satisfy the same requirements as true probabilities expressed in Proposition~\ref{prop:node_cond_prob_interval}.  
For simplicity, we set all $\tau_v$, $\node \in \nodes_T$, to $\tau$.
Then, we can prove the following result. 
\begin{restatable}{theorem}{compthreshpred}
\label{thm:comp_tresh_pred}
For Algorithm~\ref{alg:tbs_prediction} with all thresholds $\tau_v$, $\node \in \nodes_T$, 
set to $\tau$ and any $\bx \in \calX$, we have that:
\begin{equation}
c_{\tau}(T,\bx) \le  1 + \lfloor \hat{P}/\tau \rfloor \cdot \depth_T \cdot \deg_T \,,    
\label{eq:comp_tresh_pred}
\end{equation}
where $\hat{P}$ is a constant upperbounding $\sum_{j=1}^m \heta_j(\bx)$,
$\depth_T = \max_{v \in L_T} \lenpath_v - 1$, and $\deg_T = \max_{v \in V_T} \deg_v$.
\end{restatable}
We present the proof of this result in Appendix~\ref{app:complexity_analysis}. 
Similarly as in the case of the training cost, we can conclude the logarithmic cost in the number of labels.
\begin{remark}
For a tree of constant $\deg_T=\lambda (\ge 2)$ and $\depth_T=\log_\lambda{m}$, the cost of Algorithm~\ref{alg:tbs_prediction} is $O(\log m)$.  
\end{remark}
\noindent   
The above result can also be related to the sum of true conditional label probabilities, $\sum_{j=1}^m \eta_j(\bx)$.
In this case one needs to take into account the $L_1$-estimation error of $\eta_j(\bx)$. 
In the next section, we discuss the upper bound of the $L_1$-estimation error 
that can be incorporated into the above result.
A theorem of this form has been published before in~\citep{Busa-Fekete_et_al_2019}. 

The analysis of Algorithm~\ref{alg:ucs_prediction} which predicts the top $k$ labels is more involved. 
Its prediction cost defined as before is upperbounded by $c_{\tau_k}(T, \bx)$, 
where $\tau_k$ is equal to $\heta_j(\bx)$ 
being the $k$-th highest estimate of conditional label probabilities.
However, this algorithm uses a priority queue which operations require $O(\log m)$ time. 
Therefore, each call of a node classifier is not associated with $O(1)$ cost. 
Nevertheless, in practical scenarios the maintenance of the priority queue is almost negligible 
as only in the worst case scenario its size approaches $m$. 

Finally, let us shortly discuss the space complexity. 
The space needed for storing the final model can also be expressed in terms of the number of nodes.
As the number of nodes of a label tree is upperbounded by $2m-1$, 
that is, the maximum number of nodes of the tree with $m$ leaves,
the space complexity is $O(m)$. 
During training or prediction there are no other structures with a higher space demand.
Nevertheless, different design choices impact the space requirements of \Algo{PLT}s. 
We discuss some of them in Section~\ref{sec:implementation}.

%% file: pics/plt-model-z.tex
\begin{tikzpicture}[scale = 1,every node/.style={scale=1},
		regnode/.style={circle,draw,minimum width=1.5ex,inner sep=0pt},
		leaf/.style={circle,fill=black,draw,minimum width=1.5ex,inner sep=0pt},
		pleaf/.style={rectangle,rounded corners=1ex,draw,font=\scriptsize,inner sep=3pt},
		pnode/.style={rectangle,rounded corners=1ex,draw,font=\scriptsize,inner sep=3pt},
		rootnode/.style={rectangle,rounded corners=1ex,draw,font=\scriptsize,inner sep=3pt},
		level/.style={sibling distance=12em/#1, level distance=5ex}
	]
	\node (z) [rootnode] {$r_T = v_1$}
	child {node (a) [pnode] {$\pa{v_4} = v_2$} 
		child {node [label=below:{\makebox[1cm][c]{$y_1$}}] (b) [pleaf] {$\leafnode_1 = v_4$} edge from parent node[above left]{}}
		child {node [label=below:{\makebox[1cm][c]{$y_2$}}] (g) [pleaf] {$\leafnode_2 = v_5$} edge from parent node[above right]{}}
		edge from parent node[above left]{}
	}
	child {node (j) [pnode] {$v_3$}
		child {node [label=below:{\makebox[1cm][c]{$y_3$}}] (k) [pleaf] {$\leafnode_3 = v_6$} edge from parent node[above left]{}}
		child {node [label=below:{\makebox[1cm][c]{$y_4$}}] (l) [pleaf] {$\leafnode_4 = v_7$}
			{
				child [grow=right] {node (s) {} edge from parent[draw=none]
					child [grow=up] {node (t) {} edge from parent[draw=none]
						child [grow=up] {node (u) {} edge from parent[draw=none]}
					}
				}
			}
			edge from parent node[above right]{}
		}
		edge from parent node[above right]{}
	};
	\end{tikzpicture}

%% file: pics/plt-model.tex
\begin{tikzpicture}[scale = 1,every node/.style={scale=1},
		regnode/.style={circle,draw,minimum width=1.5ex,inner sep=0pt},
		leaf/.style={circle,fill=black,draw,minimum width=1.5ex,inner sep=0pt},
		pleaf/.style={rectangle,rounded corners=1ex,draw,font=\scriptsize,inner sep=3pt},
		pnode/.style={rectangle,rounded corners=1ex,draw,font=\scriptsize,inner sep=3pt},
		rootnode/.style={rectangle,rounded corners=1ex,draw,font=\scriptsize,inner sep=3pt},
		level/.style={sibling distance=12em/#1, level distance=5ex}
	]
	\node (z) [rootnode] {$(1)$}
	child {node (a) [pnode] {$(1,0)$} 
		child {node [label=below:{\makebox[1cm][c]{$y_1$}}] (b) [pleaf] {$\bc_1 = (1,0,0)$} edge from parent node[above left]{}}
		child {node [label=below:{\makebox[1cm][c]{$y_2$}}] (g) [pleaf] {$(1,0,1)$} edge from parent node[above right]{}}
		edge from parent node[above left]{}
	}
	child {node (j) [pnode] {$(1,1)$}
		child {node [label=below:{\makebox[1cm][c]{$y_3$}}] (k) [pleaf] {$(1,1,0)$} edge from parent node[above left]{}}
		child {node [label=below:{\makebox[1cm][c]{$y_4$}}] (l) [pleaf] {$(1,1,1)$}
			{
				child [grow=right] {node (s) {} edge from parent[draw=none]
					child [grow=up] {node (t) {} edge from parent[draw=none]
						child [grow=up] {node (u) {} edge from parent[draw=none]}
					}
				}
			}
			edge from parent node[above right]{}
		}
		edge from parent node[above right]{}
	};
	\end{tikzpicture}

%% file: 05-theory.tex
\section{Statistical analysis of \Algo{PLT}s}
\label{sec:analysis}

In this section, we thoroughly analyze the \Algo{PLT} model in terms of its statistical properties. 
The main results concern the regret bounds for several performance measures commonly used in XMLC. 
We first upperbound the $L_1$ estimation error of marginal probabilities of labels, \lonemarginalserror{}, 
by the $L_1$ error of the node classifiers, $|\eta(\bx, \node) - \heta(\bx, \node)|$. 
We then generalize this result to a wide class of strongly proper composite losses~\citep{Agarwal_2014} 
to make a direct connection between the marginal probability estimates and a learning algorithm used in the tree nodes.
We then analyze a wide class of generalized performance metrics. 
Instances of such metrics are Hamming loss, being a canonical loss function for multi-label classification,
the AM metric which weights labels by their priors, 
or the macro and micro $F_\beta$-measures.
Next, we present the regret analysis for \precatk{} which is the most popular metric used in XMLC. 
Finally, we discuss the relation of \Algo{PLT}s to hierarchical softmax. 

\subsection{$L_1$ estimation error}
\label{subsec:analysis-marginal}

\begin{sloppypar}
We start with a bound that express the quality of probability estimates $\heta_v(\bx)$, $v \in \nodes_T$.
The lemma and corollary below generalize a similar result obtained for multi-class classification 
in~\citep{Beygelzimer_et_al_2009a}. 
\begin{lemma}
\label{lma:node_estimation_regret}
For any tree $T$ and distribution $\prob(\by \vert \bx)$ the following holds for each $v \in V_T$:
\begin{equation}
\left | \eta_v(\bx) - \heta_v(\bx) \right |  \leq  \sum_{v' \in \Path{v}} \eta_{\pa{v'}}(\bx) \left | \eta(\bx,v') - \heta(\bx, v')  \right | \,,
\label{eqn:estimation_bound_known}
\end{equation}
where we assume $\heta(\bx, v) \in [0,1]$, for each $v \in V_T$, and $\eta_{\pa{r_\tree}}(\bx) = 1$, for the root node $r_\tree$.
\end{lemma}
\end{sloppypar}

From this lemma we immediately get guarantees for estimates of the marginal probabilities,
for each label $j \in \labels$ corresponding to leaf node $\leafnode_j$.
\begin{corollary}\label{cor:marginal-conditional-rec}
For any tree $T$ and distribution $\prob(\by \vert \bx)$, the following holds for each label $j \in \labels$:
\begin{equation}
\left | \eta_{j}(\bx) - \heta_{j}(\bx) \right |  \leq  \sum_{v \in \Path{\leafnode_j}} \eta_{\pa{v}}(\bx) \left | \eta(\bx,v) - \heta(\bx, v)  \right | \,,
\label{eqn:estimation_bound}
\end{equation}
where we assume $\eta_{\pa{r_T}}(\bx) = 1$ for the root node $r_T$.
\end{corollary}
It is worth to notice that the above bounds are tighter 
than the one in~\citep{Beygelzimer_et_al_2009a}, 
since the $L_1$ estimation error of the node classifiers is additionally multiplied 
by the probability of the parent node $\eta_{\pa{v'}}(\bx)$. 
Our results are also obtained using different arguments.
Because of that, we present the entire proof below in the main text.
\begin{proof}
Recall the recursive factorization of probability $\eta_v(\bx)$ given in~(\ref{eqn:plt_factorization_recursion}):
$$
\eta_{\node}(\bx) = \eta(\bx,\node) \eta_{\pa{\node}}(\bx) \,.
$$
As the same recursive relation holds for $\heta_v(\bx)$, see (\ref{eqn:plt_estimates_factorization_recursion}), we have that
$$
\left | \eta_\node(\bx) - \heta_\node(\bx) \right | = \left | \eta(\bx,\node) \eta_{\pa{\node}}(\bx) - \heta(\bx,\node) \heta_{\pa{\node}}(\bx) \right | \,.
$$
By adding and subtracting $\heta(\bx, \node) \eta_{\pa{\node}}(\bx) $, using the triangle inequality $|a + b| \leq |a| + |b|$
and the assumption that $\heta(\bx, \node) \in [0 ,1]$, we obtain:
\begin{eqnarray*}
\left | \eta_v(\bx) - \heta_v(\bx) \right | &\!\!\!=\!\!\!& \left | \eta(\bx,\node) \eta_{\pa{\node}}(\bx)- \heta(\bx, \node) \eta_{\pa{\node}}(\bx) + \heta(\bx, \node) \eta_{\pa{\node}}(\bx) - \heta(\bx,\node) \heta_{\pa{\node}}(\bx) \right | \\
 &\!\!\!\leq\!\!\!& \left | \eta(\bx,\node) \eta_{\pa{\node}}(\bx)- \heta(\bx, \node) \eta_{\pa{\node}}(\bx)\right |  + \left | \heta(\bx, \node) \eta_{\pa{\node}}(\bx) - \heta(\bx,\node) \heta_{\pa{\node}}(\bx) \right | \\
&\!\!\!\leq\!\!\!& \eta_{\pa{\node}}(\bx) \left | \eta(\bx, \node) - \heta(\bx, \node) \right | + \heta(\bx,\node) \left | \eta_{\pa{\node}}(\bx) - \heta_{\pa{\node}}(\bx) \right | \\
&\!\!\!\leq\!\!\!& \eta_{\pa{\node}}(\bx) \left | \eta(\bx, \node) - \heta(\bx, \node) \right | + \left | \eta_{\pa{\node}}(\bx) - \heta_{\pa{\node}}(\bx) \right | 
\end{eqnarray*}
Since the rightmost term corresponds to the $L_1$ error of the parent of $\node$, we use recursion to get the result of Lemma~\ref{lma:node_estimation_regret}:
$$
\left | \eta_v(\bx) - \heta_v(\bx) \right |  \leq  \sum_{v' \in \Path{v}} \eta_{\pa{v'}}(\bx) \left | \eta(\bx,v') - \heta(\bx, v')  \right | \,,
$$
where for the root node $\eta_{\pa{r_T}}(\bx) = 1$.
As the above holds for any $\node \in \nodes$, the result also applies to marginal probabilities of labels as stated in Corollary~\ref{cor:marginal-conditional-rec}.
\end{proof}

The above results are conditioned on $\bx$ and concern a single node $\node \in \nodes$.
The next theorem gives the understanding of the average performance over all labels and the entire distribution $\prob(\bx)$. 
We present the result in a general form of a weighted average as 
this form we use later to prove bounds for the generalized performance metrics.
\begin{restatable}{theorem}{thmtotalestimationbound}
\label{thm:total_estimation_bound}
For any tree $T$, distribution $\prob(\bx,\by)$, and weights $W_j \in R$, $j \in \{1, \ldots, m\}$, the following holds:
\begin{eqnarray}
& & \frac{1}{m}\sum_{j=1}^m W_j \mathbb{E}_{\bx \sim \prob(\bx)} \left [ \left | \eta_j(\bx) - \heta_j(\bx) \right | \right ] \leq \nonumber \\ 
& & \quad \quad \frac{1}{m} \sum_{\node \in \nodes} \prob(z_{\pa{v}} = 1) \mathbb{E}_{\bx \sim \prob(\bx \vert z_{\pa{v}} =1)} \left [ | \eta(\bx,v') - \heta(\bx, v') | \right ]  \sum_{j \in \leaves_\node} W_j  \,,
\label{eqn:total_estimation_bound}
\end{eqnarray}
where for the root node $\prob(z_{\pa{r_T}} = 1) = 1$. 
For $W_j = 1$,  $j \in \{1, \ldots, m\}$, we have:
\begin{eqnarray*}
& & \frac{1}{m}\sum_{j=1}^m \mathbb{E}_{\bx \sim \prob(\bx)} \left [ \left | \eta_j(\bx) - \heta_j(\bx) \right | \right ] \leq \nonumber \\ 
& & \quad \quad \frac{1}{m} \sum_{\node \in \nodes} \prob(z_{\pa{v}} = 1) \mathbb{E}_{\bx \sim \prob(\bx \vert z_{\pa{v}} =1)} \left [ | \eta(\bx,v') - \heta(\bx, v') | \right ]  | \leaves_\node  |  \,.
% \label{eqn:total_estimation_bound}
\end{eqnarray*}
\end{restatable}
The result states that the weighted expected $L_1$ estimation error averaged over all labels can be bounded by 
a weighted sum of expected $L_1$ errors of node classifiers divided by the number of labels. 
A weight associated with node $v$ is a product of the probability mass of a parent node 
and the number of leave nodes in a subtree rooted in $v$. 
This means that a node closer to the root has a higher impact on the overall performance.
This agrees with the intuition as such nodes impact estimates of more labels. 
We omit the proof of this theorem here as it is quite technical. 
It is presented with additional auxiliary results in Appendix~\ref{app:analysis-marginal}.

\subsection{Strongly proper composite losses}
\label{subsec:strongly_proper_composite_losses}

So far the results concern the probability estimates without any direct link to a learning algorithm.
In this subsection, we relate the quality of probability estimates 
to the error measured in terms of a loss function 
which can be minimized during the training of a node classifier.
We first recall the concept of strongly proper composite losses~\citep{Agarwal_2014} 
for binary classification. 
Examples of such losses are commonly used functions such as
logistic loss, squared loss, squared hinge loss, and exponential loss.
Notice that the standard hinge loss does not belong to this class of losses.
Finally, we show an extension of Theorem~\ref{thm:total_estimation_bound} 
in which the right-hand side is expressed in terms of a strongly proper composite loss.

The strongly proper composite losses are of a special interest in 
the problem of class probability estimation with two outcomes, $y \in \{-1, 1\}$. 
Let pairs $(\bx, y)$ be generated i.i.d. according to $\prob(\bx, y)$.
We denote $\prob(y = 1\given \bx)$ by $\eta(\bx)$ and its estimate by $\heta(\bx) \in [0,1]$.
Let us first define a class probability estimation (CPE) loss as a function $\ell: \{-1,1\} \times [0,1] \mapsto \mathbb{R}_+$. 
Its conditional risk is given by 
$$
R_{\ell}(\heta \given \bx) =  \eta(\bx) \ell (1, \heta(\bx)) + (1 - \eta(\bx)) \ell (-1, \heta(\bx))  \,.
$$
A CPE loss is proper if for any $\eta(\bx) \in [0,1]$, $\eta(\bx) \in \argmin_{\heta} R_{\ell}(\heta \given \bx)$. 
Since it is often more convenient for prediction algorithms 
to work with a real-valued scoring function, $f: \calX \mapsto \mathbb{R}$,
than with an estimate bounded to interval $[0,1]$, 
we transform $\heta(\bx)$ using a strictly increasing (and therefore invertible) link function 
$\psi: [0,1] \rightarrow \mathbb{R}$,
that is, $f(\bx) = \psi(\heta(\bx))$.
%
%Let a scoring function to have a form of a real-valued function $f: \calX \mapsto \mathbb{R}$.
%Moreover, we assume that there exists we use a strictly increasing (and therefore invertible) link function
%$\psi: [0,1] \rightarrow \mathbb{R}$ such that $f(\bx) = \psi(\heta(\bx))$. 
%
We then consider a composite loss function
$\ell_c: \{-1, 1\} \times \mathbb{R} \mapsto \mathbb{R}_+$ 
defined via a CPE loss as  
$$
\ell_c(y, f(\bx)) = \ell(y, \psi^{-1}(f(\bx))) \,.
$$
The regret of $f$ in terms of a loss function $\ell_c$ at point $\bx$ is defined as:
$$
\reg_{\ell_c}(f \given \bx) = R_{\ell}(\psi^{-1}(f) \given \bx) - R_{\ell}^*(\bx) \,,
$$
%where $R_{\ell}(f \given \bx)$ is the expected loss at point $\bx$:
%$$
%R_{\ell}(f \given \bx) =  \eta(\bx) \ell (1, f(\bx)) + (1 - \eta(\bx)) \ell (-1, f(\bx))  \,,
%$$
%and  $R_{\ell}^*(\bx)$ is the minimum expected loss at point $\bx$.
where  $R_{\ell}^*(\bx)$ is the minimum expected loss at point $\bx$,
achievable by $f^*(\bx) = \psi(\eta(\bx))$.
 
We say a loss function $\ell_c$ is $\lambda$-strongly proper composite loss, if for any $\eta(\bx), \psi^{-1}(f(\bx)) \in [0,1]$: 
\begin{equation}
\left | \eta(\bx)  - \psi^{-1}(f(\bx))  \right | \le \sqrt{ \frac{2}{\lambda}} \sqrt{\reg_{\ell_c}(f \given \bx)} \,.  
\label{eqn:regret_bound_for_spc_losses}
\end{equation}
It can be shown under mild regularity assumptions that $\ell_c$ is $\lambda$-strongly proper composite 
if and only if its corresponding CPE loss is proper and function $H_{\ell}(\eta) = R_{\ell}(\eta \given \bx)$ 
is $\lambda$-strongly concave,
that is, $\left | \frac{d^2 H_{\ell}(\eta)}{d^2 \eta} \right | \ge \lambda$. 

We apply the above results to node classifiers in a \Algo{PLT} tree. 
In each node $\node \in \nodes_T$ we consider a binary problem with $y = 2z_v - 1$ and 
pairs $(\bx, z_v)$ generated i.i.d. according to $\prob(\bx, z_v \given z_{\pa{\node}} = 1)$. 
Moreover, let $f_v$ be a scoring function in node $\node \in \nodes_T$ 
minimized by a strongly proper composite loss function $\ell_c$. 
We obtain then the following result. 
\begin{restatable}{theorem}{thmtotalregretbound}
\label{thm:total_regret_bound}
For any tree $T$, distribution $\prob(\bx,\by)$, weights $W_j \in R$, $j \in \{1, \ldots, m\}$, 
a strongly proper composite loss function $\ell_c$, and a set of scoring functions $f_\node$, $\node \in \nodes_T$, 
the following holds:
\begin{equation}
\frac{1}{m}\sum_{j=1}^m W_j \mathbb{E}_{\bx \sim \prob(\bx)} \left [ \left | \eta_j(\bx) - \heta_j(\bx) \right | \right ] \leq  
\frac{\sqrt{2}}{m\sqrt{\lambda}}  \sum_{\node \in \nodes} \sqrt{ \prob(z_{\pa{v}} = 1) \reg_{\lossfunc_c}(f_v)} \sum_{j \in \leaves_\node} W_j \,,
\label{eqn:generalized_total_regret_bound}
\end{equation}
where for the root node $\prob(z_{\pa{r_T}} = 1) = 1$, and $\reg_{\lossfunc_c}(f_\node)$ is the expected $\ell_c$-regret of $f_\node$ taken over $\prob(\bx, z_\node \given z_{\pa{\node}} = 1)$. 
For $W_j = 1$,  $j \in \{1, \ldots, m\}$, we have:
\begin{equation}
\frac{1}{m}\sum_{j=1}^m \mathbb{E}_{\bx \sim \prob(\bx)} \left [ \left | \eta_j(\bx) - \heta_j(\bx) \right | \right ] \leq  
\frac{\sqrt{2}}{m\sqrt{\lambda}}  \sum_{\node \in \nodes} |\leaves_\node| \sqrt{ \prob(z_{\pa{v}} = 1) \reg_{\lossfunc_c}(f_v)} \,.
\label{eqn:total_regret_bound}
\end{equation}
\end{restatable}

The theorem justifies the use of strongly proper composite losses during the training of node classifiers. 
The technical details of the proof are presented in 
Appendix~\ref{app:strongly_proper_composite_losses}. 
Here, we only notice that the weights of node errors follow
from Theorem~\ref{thm:total_estimation_bound},
while the squared root dependency from (\ref{eqn:regret_bound_for_spc_losses}).

\subsection{Generalized classification performance metrics}
\label{subsec:analyis-generalized}

The results above show guarantees of \Algo{PLT}s for estimates of marginal probabilities of labels.
In this subsection, we discuss a wide family of metrics often used to report performance of multi-label classification,
such as (weighted) Hamming loss, AM metric, or macro- and micro-averaged $F_{\beta}$-measure.
This family of metrics can be defined as a linear-fractional function of 
label-wise false positives $\FP_j(h_j) = \prob(h_j(\bx) = 1 \land y_j = 0)$ 
and false negatives  $\FN_j(h_j) = \prob(h_j(\bx) = 0 \land y_j = 1)$.
As already proven~\citep{Koyejo_et_al_2015,Kotlowski_Dembczynski_2017} the optimal strategy for these metrics 
is to find a threshold on the marginal probability $\eta_j(\bx)$ for each label $j \in \labels$.
From the practical point of view this boils down to setting thresholds to either predefined values, 
if they are known from theory (for example, this is 0.5 for Hamming loss), 
or to tune them on a validation set, 
if their optimal value depends on the optimum of the metric~\citep{Koyejo_et_al_2015,Kotlowski_Dembczynski_2017}.
Both approaches can be applied to \Algo{PLT}s. 
The prediction procedure from Algorithm~\ref{alg:tbs_prediction} can work with any set of thresholds. 
However, for small values of thresholds a problem of exploring a large part of a tree may appear.
%There also exist efficient methods for tuning thresholds in the XMLC setting,
%see~\citep{Jasinska_et_al_2016}.
The results below show the theoretical aspects of these approaches, namely,
we tailor the regret bounds for the general 1-vs-All approach, 
proven in~\citep{Kotlowski_Dembczynski_2017}, to \Algo{PLT}s.

Let us define the problem in a formal way. 
To this end we use a linear-factorial function $\Psi$ of the following generic form:
\begin{equation}
    \Psi(\FP, \FN) = \frac{a_0 + a_1 \FP + a_2 \FN}{b_0 + b_1 \FP + b_2 \FN} \,,
    \label{eqn:psi}
\end{equation}
being non-increasing in its arguments. 
Moreover, we assume that there exists $\gamma > 0$, such that 
\begin{equation}
b_0 + b_1 \FP + b_2\FN \ge \gamma,
\label{eqn:gamma}
\end{equation}
that is, the denominator of $\Psi$ is positive and bounded away from 0.
A macro-averaged generalized classification performance metric \psimacro is defined then as:
\begin{equation}
\Psi_{\mathrm{macro}}(\bh) = \frac{1}{m} \sum_{j = 1}^{m} \Psi(h_j) = \frac{1}{m} \sum_{j = 1}^{m} \Psi(\FP_j(h_j), \FN_j(h_j)).
\label{eqn:psimacro}
\end{equation}
It computes an average performance over single labels. 
Micro-averaged performance metrics, in turn, compute first the average false positives and false negatives:
\begin{equation*}
\bar \FP (\bh) = \frac{1}{m} \sum_{i = 1}^{m} \FP_j(h_j)\,, \quad \bar \FN (\bh) = \frac{1}{m} \sum_{j = 1}^{m} \FN_j(h_j)\,.
\end{equation*}
Then, a micro-averaged metric \psimicro is defined as:
\begin{equation}
\label{eqn:psimicro}
\Psi_{\mathrm{micro}} (\bh) = \Psi( \bar \FP(\bh), \bar \FN(\bh))\,.
\end{equation}

The optimal classifier, being a member of class $\calH^m_{\textrm{bin}}: \calX \rightarrow  \{0,1\}^m$, for the above generalized performance measures has the generic form:
\begin{equation}
\bh^*_{\Psi}(\bx) = \bh^*_{\balpha^*_{\Psi}}(\bx) 
= \left ( h^*_{1,\alpha^*_{\Psi,1}}(\bx), h^*_{2,\alpha^*_{\Psi,2}}(\bx),\ldots, h^*_{m,\alpha^*_{\Psi,m}}(\bx) \right )\,, 
\label{eqn:gpm-optimal_threshold}
\end{equation}
where 
$$
h^*_{j,\alpha^*_{\Psi,j}}(\bx)  = \assert{\eta_j(\bx) > \alpha^*_{\Psi,j}} \,,
\textrm{with $\Psi$-optimal~} \balpha^*_\Psi = (\alpha^*_{\Psi,1}, \alpha^*_{\Psi,2}, \ldots, \alpha^*_{\Psi,m}) \in [0,1]^m \,.
$$
In other words, for each metric there is an optimal vector $\balpha^*_\Psi$ of thresholds 
defined over the marginal probabilities of labels, $\eta_j(\bx)$, for all $j \in \labels$.
The values of its elements are given by the following expression~\citep{Koyejo_et_al_2015,Kotlowski_Dembczynski_2017}:
$$
\alpha^*_\Psi = \frac{\Psi(\FP^*, \FN^*)  b_1 - a_1}{\Psi(\FP^*, \FN^*) (b_1 + b_2) - (a_1 + a_2)} \,,
$$
where $\FP^*$, $\FN^*$ are arguments maximizing either $\Psi(\FP_j(h_j), \FN_j(h_j))$, 
for each label $j \in \labels$ separately, in case of a macro-averaged metric,
or $\Psi( \bar \FP(\bh), \bar \FN(\bh))$ in case of a  micro-averaged metric. 
This result shows that for macro-averaged metrics the threshold can be different for each label, 
while for micro-average metrics there is one common threshold shared by all labels.
Therefore, we denote the optimal classifier for a macro-average metric by $\bh^*_{\balpha^*_{\Psi}}$,
while for a micro-average metric by $\bh^*_{\alpha^*_{\Psi}}$.

The thresholds in general depend on the optimal value of $\Psi$, 
which makes their value to be unknown beforehand. 
Only for metrics for which $b_1 = b_2 = 0$,
the thresholds can be computed a priori.
This is the case of Hamming loss, its cost-sensitive variant 
(in which there are different costs of false positive and false negative predictions), 
or the AM metric. 
In the other cases, thresholds have to be found on a validation set.
For some metrics, such as the micro- and macro-F measure,  
this can be performed efficiently even in the XMLC setting,
as only positive and positively predicted labels are needed 
to tune thresholds~\citep{Jasinska_et_al_2016}. 
This can be even obtained using an online procedure~\citep{Busa-Fekete_et_al_2015, Jasinska_et_al_2016}.

We present the form of $\Psi(\FP, \FN)$ and $\alpha^*_\Psi$
for some popular generalized performance metrics in Table~\ref{tab:generalized-metrics-1}. 
We use there $P$ to denote $\prob(y_j = 1)$, for macro-averaging, 
and $\frac{1}{m}\sum_{j=1}^m\prob(y_j = 1)$, for micro-averaging. 
Remark that this is a constant not depending on $\bh$.
All these metrics can be used with macro- and micro-averaging. 
Remark, however, that for Hamming loss both variants lead to the same form. 
A similar table can be found in~\citep{Kotlowski_Dembczynski_2017}.
\begin{table}[h]
    \centering
    \begin{tabular}{l r r}
        \toprule
         Metric &   $\Psi(\FP, \FN)$ & $\alpha^*_\Psi$ \\[3pt]
         \midrule
         Hamming loss & $1 - \FP - \FN$ & $0.5$ \\
         F-measure    &  $\frac{(1+\beta^2)(P-\FN)}{(1+\beta^2)P - \FN + \FP}$ & $\Psi(\FP^*, \FN^*)/2$\\
         Jaccard similarity & $\frac{P-\FN}{P + \FP}$ & $\frac{\Psi(\FP^*, \FN^*)}{\Psi(\FP^*, \FN^*)+1}$ \\
         AM & $\frac{2P(1-P)-P\FP-(1-P)\FN}{2P(1-P)}$ & $P$\\
         \bottomrule
    \end{tabular}
    \caption{Examples of popular generalized performance metrics, with their form of $\Psi(\FP, \FN)$ and $\alpha^*_\Psi$.
             $P$ denotes $\prob(y_j = 1)$, for macro-averaging, 
             or $\frac{1}{m}\sum_{j=1}^m\prob(y_j = 1)$, for micro-averaging.}
    \label{tab:generalized-metrics-1}
\end{table}

The regret of the $\Psi_{\mathrm{macro}}$ metric decomposes into a weighted sum:
\begin{equation}
\reg_{\Psi_{\mathrm{macro}}}(\bh) = \Psi_{\mathrm{macro}} (\bh^*_{\balpha^*_{\Psi}}) - \Psi_{\mathrm{macro}}(\bh) = \frac{1}{m} \sum_{j = 1}^{m} ( \Psi(h^*_{j, \alpha^*_{\Psi,j}}) - \Psi(h_j) )
\label{eqn:macro-regret}
\end{equation}
In turn, the regret of the $\Psi_{\mathrm{micro}}$ metric is given by: 
\begin{equation}
\reg_{\Psi_{\mathrm{micro}}}(\bh) = \Psi_{\mathrm{micro}}(\bh^*_{\alpha^*_{\Psi}}) - \Psi_{\mathrm{micro}}(\bh) \,.
\label{eqn:micro-regret}
\end{equation}

We are interested in bounding these regrets with a performance of node classifiers of a \Algo{PLT}. 
We assume, similarly as in the previous subsection,
that a score function $f_v(\bx)$ in a node $\node \in \nodes$ is trained
via minimization of a strongly proper composite loss function $\ell_c$.  
The estimates $\heta(\bx, v)$, for all $\node \in \nodes$, are then computed as:
$$
\heta(\bx, v) = \psi^{-1}(f_v(\bx))\,.
$$
The final prediction is computed by Algorithm~\ref{alg:tbs_prediction} 
and has a form similar to the optimal classifier (\ref{eqn:gpm-optimal_threshold}):
$$
\bh_{\btau}(\bx) = 
\left ( h_{1,\tau_1}(\bx), h_{2,\tau_2}(\bx), \ldots, h_{m, \tau_m}(\bx) \right ), 
\mathrm{~where~} h_{j, \tau_j}(\bx) = \assert{ \heta_j(\bx) > \tau_j} \,,
$$
for some vector $\btau = (\tau_1, \tau_2, \ldots, \tau_m) \in [0,1]^m$ of thresholds. 
Estimates $\heta_j(\bx)$ are computed as in (\ref{eqn:plt-factorization-prediction}), 
that is, $\heta_j(\bx) = \prod_{v \in \Path{l_j}} \heta(\bx, v)$, 
where $l_j \in L_T$ is a node corresponding to label~$j$. 
The theorems below present the main result of this section.

\begin{restatable}{theorem}{psimacroregret}
\label{thm:psi_macro_regret}
Let $\tau^*_j = \argmax_{\tau} \Psi(h_{j, \tau})$, 
for each $j \in \labels$, and $\btau^* = (\tau^*_1, \tau^*_2, \ldots, \tau^*_m)$. 
For any tree $T$ and distribution $\prob(\bx, \by)$, 
the classifier $\bh_{\btau^*}$ achieves the following upper bound on its $\Psi_{\mathrm{macro}}$-regret:
$$
\reg_{\Psi_{\mathrm{macro}}}(\bh_{\btau^*}) \leq  
\frac{\sqrt{2}}{m\sqrt{\lambda}}  \sum_{\node \in \nodes} \sqrt{ \prob(z_{\pa{v}} = 1) \reg_{\lossfunc_c}(f_v)} \sum_{j \in L_v} C_j\,,
$$
where  $C_j = \frac{1}{\gamma} ( \Psi(h^*_\Psi,j) (b_1 + b_2) - (a_1 + a_2))$, for each $j \in \labels$,
with $\gamma$ defined in (\ref{eqn:gamma}),
$\prob(z_{\pa{r_T}} = 1) = 1$ for the root node, 
and $\reg_{\lossfunc_c}(f_\node)$ is the expected $\ell_c$-regret of $f_\node$ 
taken over $\prob(\bx, z_\node \given z_{\pa{\node}} = 1)$.
\end{restatable}

\begin{restatable}{theorem}{psimicroregret}
\label{thm:psi_micro_regret}
Let $\bh_{\tau} = (h_{1,\tau}, h_{2,\tau},\ldots, h_{m, \tau})$ 
be a classifier which shares the same threshold $\tau$ over all labels $j \in \labels$. 
For any tree $T$, distribution $\prob(\bx, \by)$, and $\tau^* = \argmax_{\tau} \Psi_{\mathrm{micro}}(\bh_{\tau})$,
classifier $\bh_{\tau^*}$ achieves the following upper bound on its $\Psi_{\mathrm{micro}}$-regret:
$$
\reg_{\Psi_{micro}}(\bh_{\tau^*}) \le \frac{C}{m} \sqrt{\frac{2}{\lambda}} \sum_{v \in \vertices } |L_v| \sqrt{ \prob(z_{\pa{v}} = 1) \reg_{\lossfunc_c}(f_v)}\,,
$$
where $C = \frac{1}{\gamma}(\Psi_{\mathrm{micro}}(\bh^*_{\mathrm{\Psi}})(b_1 + b_2) - (a_1 + a_2) )$ 
with $\gamma$ defined in (\ref{eqn:gamma}),
$\prob(z_{\pa{r_T}} = 1) = 1$ for the root node, 
and $\reg_{\lossfunc_c}(f_\node)$ is the expected $\ell_c$-regret of $f_\node$ 
taken over $\prob(\bx, z_\node \given z_{\pa{\node}} = 1)$.
\end{restatable}

The above theorems can be interpreted in the following way. 
For marginal probability estimates $\heta_j(\bx)$, $j \in \labels$, 
obtained as described just before the theorem, 
there exists a vector $\btau$ of thresholds, 
for which the regret of a generalized performance metric
is upperbounded solely by regrets of node classifiers, 
expressed in terms of a strongly proper composite loss function.
Therefore, from the perspective of learning one needs to focus on the node classifiers
to get as accurate as possible estimates of marginal probabilities of labels,
by minimizing a strongly proper composite loss function in each node.
The next step, being independent of the previous one, 
is to obtain the right values of thresholds $\btau^*$,
following one of the approaches mentioned above.

Let us analyze the regret bounds more carefully. 
In the case of macro-averaged metrics, 
the regret of each node classifier is weighted by the sum of $C_j$-values of all labels in the corresponding subtree. 
In the case of micro-averaged metrics, 
there is only one global $C$-value, 
and each node classifier is weighted by the number of labels in the corresponding subtree.
The values of $C$ and $\gamma$ for different metrics are given in Table~\ref{tab:generalized-metrics-2}
(a similar table can be found in~\citep{Kotlowski_Dembczynski_2017}). 
It is easy to verify with these values that 
for the Hamming loss  %$\reg_{\Psi_{micro}}$ and $\reg_{\Psi_{macro}}$ bounds give the same results. 
the regret bounds for macro- and micro-averaging are the same. 
This agrees with the fact that both averaging schemes boils down to the same metric in case of the Hamming loss. 
In general, the macro- and micro-averaging bounds coincide for all metrics with constant $C$.
Interestingly, the bounds are different for the $F_1$-measure and the Jaccard similarity,
while they both share the same optimal solution 
(since the Jaccard similarity is a strictly monotone transformation of the $F_1$-measure).  
As $\gamma$ is the same for both metrics, 
this observation suggests that $C$ could be defined more tightly.
%\todo{Czy mozna interpretowac regret dla macro F tak, ze wazniejsze jest to, aby lepiej nauczyc sie etykiet z malym P?}
One can also observe that $C$ grows with decreasing $P$. 
Therefore, for sparse problems and labels from the long-tail the value of $C$ can be large,
potentially leading to poor guarantees. 
\begin{table}[h]
    \centering
    \begin{tabular}{c l l}
        \toprule
         Metric & $\gamma$ & $C$ \\[3pt]
         \midrule
         Hamming loss & 1 & 2 \\
         $F_\beta$-measure & $\beta^2P$   & $\frac{1+\beta}{\beta^2P}$ \\
         Jaccard similarity & $P$ & $\frac{\Psi(\FP^*, \FN^*)+1}{P}$\\
         AM & $2P(1-P)$ & $\frac{1}{2P(1-P)}$ \\ 
         \bottomrule
    \end{tabular}
    \caption{The values of $\gamma$ and $C$ values for some generalized classification performance metrics. As before,
    $P$ denotes $\prob(y_j = 1)$, for macro-averaging, or $\frac{1}{m}\sum_{j=1}^m\prob(y_j = 1)$, for micro-averaging.}
    \label{tab:generalized-metrics-2}
\end{table}

The proofs of both theorems are given in Appendix~\ref{app:analysis-generalized}. 
They are based on results previously obtained for the \Algo{1-vs-All} approach in~\citep{Kotlowski_Dembczynski_2017},
combined together with Theorem~\ref{thm:total_regret_bound}. 
The result for the \Algo{1-vs-All} approach relies on two observations.
The first one states that the regret for a cost-sensitive binary classification 
can be upperbounded by the $L_1$ estimation error of the conditional probabilities,
if a classification procedure uses a threshold which directly corresponds to the misclassification cost.
The second shows that the regret of the generic function $\Psi(\FP, \FN)$ can be upperbounded 
by the regret of the cost-sensitive binary classification with costs related to $\alpha^*_\Psi$.
The actual value of the optimal thresholds is a direct consequence of the proof.
Putting these two observations together along with Theorem~\ref{thm:total_regret_bound} gives the final results.

\subsection{\Precat{k}}
\label{subsec:analyis-precision}

%\todo{roznia notacja precision@k, raz itali, raz mathrm, raz p at k}

In this section, we analyze \precatk{} which is of a different nature than the metrics discussed above.
Let us consider a class of functions 
$\calH^{m}_{@k} = \{ \bh \in \calH^{m}_{\textrm{bin}} : \sum_{j=1}^m h_j(\bx) = k\,,\, \forall \bx \in \calX \}$,
that is, functions that predict exactly $k$ labels with $k \le m$.
We then define \precatk{} for $\bh_{@k} \in \calH^{m}_{@k}$ as:
$$
p@k(\by, \bh_{@k}(\bx)) = \frac{1}{k} \sum_{j \in \hat \calL_{\bx}} \assert{y_j = 1} \,,
$$
where  $\hat \calL_{\bx} = \{j \in \calL: h_j(\bx) = 1 \}$ 
is a set of $k$ labels predicted by classifier $\bh_{@k}$ for $\bx$. 
In order to define conditional risk it is more convenient to consider the \precatk{} loss, 
$\ell_{p@k} = 1 - p@k(\by, \bh_{@k}(\bx))$. 
The conditional risk is then:
\begin{eqnarray*}
\loss_{p@k}(\bh_{@k} \given \bx) & = & \mathbb{E}_{\by} \ell_{p@k}(\by,\bh_{@k}(\bx)) \\
& = & 1 - \sum_{\by \in \calY} \prob(\by \given \bx) \frac{1}{k} \sum_{j \in \hat \calL_{\bx}} \assert{y_j = 1} \\
& = & 1 - \frac{1}{k} \sum_{j \in \hat \calL_{\bx}} \sum_{\by \in \calY} \prob(\by \given \bx) \assert{y_j = 1} \\
& = & 1 - \frac{1}{k} \sum_{j \in \hat \calL_{\bx}} \eta_j(\bx) \,.
\end{eqnarray*}
\noindent
From the above it is easy to notice that the optimal strategy for \precatk{}, 
$$
\bh^*_{p@k}(\bx) = \left (h^*_{1,p@k}, h^*_{2,p@k},\ldots, h^*_{m,p@k} \right ) \,,
$$
is to predict $k$ labels with the highest marginal probabilities $\eta_j(\bx)$, 
\begin{equation}
h^*_{j,p@k} = \left \{ \begin{array}{ll} 1\,, & j \in \hat \calL^*_{\bx} \\ 0\,, & \mathrm{otherwise} \end{array} \right .  \,,
\label{eqn:bayes_for_patk}
\end{equation}
with $\hat \calL^*_{\bx} = \{j \in \calL: \pi(j) \le k\}$ 
and $\pi$ being a permutation of the labels ordered with respect to descending $\eta_j(\bx)$ with ties solved in any way.
The conditional regret for \precatk{} is then:
$$
\reg_{p@k} (\bh \given \bx) = \frac{1}{k}\sum_{i \in \hat \calL^*_{\bx}} \eta_i(\bx) - \frac{1}{k}\sum_{j \in \hat \calL_{\bx}} \eta_j(\bx)\,.
$$

The conditional regret with respect to \precatk{} can be upperbounded by the $L_1$-estimation errors %\lonemarginalserror{} 
as stated by the following theorem, originally published in~\citep{Wydmuch_et_al_2018}.

\begin{restatable}{theorem}{thmregretprecisionconditional}
\label{thm:precatk}
For any tree $T$, distribution $\prob(\by \given \bx)$ and classifier $\bh_{@k} \in \calH^{m}_{@k}$ the following holds: 
\begin{align*}
\reg_{p@k} (\bh_{@k} \given \bx) = \frac{1}{k}\sum_{i \in \hat \calL^*_{\bx}} \eta_i(\bx) - \frac{1}{k}\sum_{j \in \hat 
\calL_{\bx}} \eta_j(\bx) \le 2 \max_{j} \left | \eta_j(\bx) - \heta_j(\bx) \right | \,.
\end{align*}
\end{restatable}
\begin{proof}
Let us add and subtract the following two terms, $\frac{1}{k}\sum_{i \in  \hat \calL^*_{\bx}} \heta_i(\bx)$ and $\frac{1}{k}\sum_{j \in  \hat \calL_{\bx}} \heta_j(\bx)$, to the regret and reorganize the expression in the following way:
\begin{align*}
\reg_{p@k} (\bh_{@k} \given \bx) & = 
\underbrace{\frac{1}{k}\sum_{i \in  \hat \calL^*_{\bx}} \eta_i(\bx) - 
\frac{1}{k}\sum_{i \in  \hat \calL^*_{\bx}} \heta_i(\bx )}_{\le 
\frac{1}{k}\sum_{i \in  \hat \calL^*_{\bx}} \left |\eta_i(\bx) - \heta_i(\bx) \right | } \\
&  + \underbrace{\frac{1}{k}\sum_{j \in \hat \calL_{\bx}} \heta_j(\bx) 
-  \frac{1}{k}\sum_{j \in  \hat \calL_{\bx}} \eta_j(\bx)}_{\le 
\frac{1}{k}\sum_{j \in \hat \calL_{\bx}} \left |\heta_j(\bx) - \eta_j(\bx) \right |  }\\
 &  + \underbrace{\frac{1}{k}\sum_{i \in \hat \calL^*_{\bx}} \heta_i(\bx)  - \frac{1}{k}\sum_{j \in \hat \calL_{\bx}} \heta_j(\bx)}_{\le 0} \\
 & \le \frac{1}{k}\sum_{i \in  \hat \calL^*_{\bx}} \left |\eta_i(\bx) - \heta_i(\bx) \right | + \frac{1}{k}\sum_{j \in \hat \calL_{\bx}} \left |\eta_j(\bx) - \heta_j(\bx) \right | 
\end{align*}
Next we bound each $L_1$ error, $\left |\eta_j(\bx) - \heta_j(\bx) \right |$  
by  $\max_{j} \left | \eta_j(\bx) - \heta_j(\bx) \right |$. 
There are at most $|\calY_k| + |\hat\calY_k| = 2k$ such terms. 
Therefore 
$$
\reg_{p@k} (\bh \given \bx)  \le 2 \max_{j} \left | \eta_j(\bx) - \heta_j(\bx) \right | \,.
$$
\end{proof}

Interestingly, the bound does not depend neither on $k$ nor $m$. 
%Because of the $\frac{1}{k}$ factor in the definition of precision$@k$, this result does not depend on $k$. Notice that neither it depends on $m$. 
However, if $k = m$ then $\reg_{p@m} = 0$ for any distribution, since $\hat \calL^*_{\bx} = \hat \calL_{\bx}$ in this case. 
In general, if $m < 2k$, then $\hat \calL^*_{\bx} \cap \hat \calL_{\bx} \neq \emptyset$. 
In other words, some of the labels from $\hat\calL_{\bx}$ are also in $\hat \calL^*_{\bx}$, 
so the bound can be tighter. 
For example, one can multiply the bound by $\frac{\min(k, m-k)}{k}$, assuming that $k \leq m$.
However,  in extreme classification usually $k \ll m$, so we do not use the more complex bound.

The above result suggests that \Algo{PLT}s are well-suited to optimization of \precatk{}. 
The next theorem shows this directly by providing an upper bound of the unconditional regret for a \Algo{PLT}.
%To upper bound the unconditional \precatk{} regret by a \Algo{PLT}, 
We use the same setting as in the above subsection with the difference 
that instead of thresholding probability estimates
we use Algorithm~\ref{alg:ucs_prediction} to compute 
predictions consisting of $k$ labels with the highest $\heta_j(\bx)$. 
The form of the final \Algo{PLT} classifier $\bh_{@k}(\bx)$ is then similar to~(\ref{eqn:bayes_for_patk}),
but with permutation $\pi$ defined over $\heta_j(\bx)$, $j \in \labels$.
Unfortunately, the max operator from Theorem~\ref{thm:precatk} needs to be replaced by sum in the derivations, 
therefore the theorem has the following form.

\begin{restatable}{theorem}{precisionatkregretbound}
\label{thm:precisionatk_regret_bound}
For any tree $T$ and distribution $\prob(\bx, \by)$, 
classifier $\bh_{@k}(\bx)$ achieves the following upper bound on its \precatk{} regret:  
$$
\reg_{p@k} (\bh_{@k}) \le \frac{2\sqrt{2}}{\sqrt{\lambda}}  \sum_{\node \in \nodes} |\leaves_\node| \sqrt{ \prob(z_{\pa{v}} = 1) \reg_{\lossfunc_c}(f_v)} \,,
$$
where $\prob(z_{\pa{r_T}} = 1) = 1$ for the root node, 
and $\reg_{\lossfunc_c}(f_\node)$ is the expected $\ell_c$-regret of $f_\node$ 
taken over $\prob(\bx, z_\node \given z_{\pa{\node}} = 1)$.
\end{restatable}
\begin{proof}
By taking expectation over $\prob(\bx)$ of the bound from Theorem~\ref{thm:precatk} 
and replacing the max operator by sum, that is, $\max(a,b) \le a + b$, for $a, b \ge 0$, 
we obtain: 
$$
\reg_{p@k}(\bh_{@k}) \le 2 \sum_{j=1}^m \mathbb{E}_{\bx \sim \prob(\bx)} \left [ \left | \eta_j(\bx) - \heta_j(\bx) \right | \right ]  
$$
Next, by applying (\ref{eqn:total_regret_bound}) from Theorem~\ref{thm:total_regret_bound}, we get the statement:
$$
\reg_{p@k}(\bh_{@k}) \le \frac{2\sqrt{2}}{\sqrt{\lambda}}  \sum_{\node \in \nodes} |\leaves_\node| \sqrt{ \prob(z_{\pa{v}} = 1) \reg_{\ell_c}(f_v)} \,.
$$
\end{proof}
%
%The proof presented in Appendix~\ref{app:analysis-precision} is based on Corollary~\ref{cor:marginal-conditional-rec} 
%and Theorem~\ref{thm:total_regret_bound}.
It is worth to compare the above bound with the one for the Hamming loss, 
taken either from Theorem~\ref{thm:psi_macro_regret} or Theorem~\ref{thm:psi_micro_regret},
with $C = 2$ and $\gamma = 1$ (see Table~\ref{tab:generalized-metrics-2}).
It turns out that the bound for \precatk{} is $m$ times larger. 
The reason is that if there were $k$ labels with the $L_1$-estimation error approaching 1,
but with the actual probability close to 0, 
then the \precatk{} regret would get its maximum. 
On the other hand, the Hamming loss regret is an average of label-wise regrets.
Therefore, it does not suffer much, 
as there were only $k$ labels out of $m$ with the highest regret.

\subsection{Relation to hierarchical softmax}
\label{subsec:hsm}

In this section, we show that \Algo{PLT}s are strictly related to hierarchical softmax 
designed for multi-class classification. 
Using our notation, 
we have $\sum_{i=1}^m y_i = 1$ for \multiclass{} problems, 
that is, there is one and only one label assigned to an instance $(\bx, \by)$. 
The marginal probabilities $\eta_j(\bx)$ in this case sum up to 1. 
Since in \multiclass{} classification always one label is assigned to an instance, 
there is no need to learn a root classifier 
which verifies whether there exists a positive label for an instance. 
Nevertheless, the factorization of the conditional probability of label $j$ is given by the same equation~(\ref{eqn:plt_factorization}) as for \multilabel{} case:
$$
\eta_j(\bx) =  \prod_{v' \in \Path{\leafnode_j}} \eta(\bx, v')  \,.
$$
However, in this case $\eta(\bx, v') = 1$, for $v'$ being the root, and 
$$
\sum_{\node' \in \childs{\node}} \eta(\bx, \node') = 1 \,,
$$ 
since $\sum_{i=1}^m y_i = 1$. 
One can easily verify that the model above is the same as the one presented in~\citep{Morin_Bengio_2005},
where the parent nodes are identified by a code indicating a path from the root to this node.
When used with a sigmoid function to model the conditional probabilities, 
we obtain the popular formulation of hierarchical softmax.

%Popular solutions are beam search~\citep{Kumar_et_al_2013,Prabhu_et_al_2018}, uniform-cost search~\citep{Joulin_et_al_2016}, and its approximate variant~\citep{Dembczynski_et_al_2012c,Dembczynski_et_al_2016}.

%\subsubsection{Suboptimality of HSM for multi-label classification}

To deal with multi-label problems, some popular tools, 
such as \Algo{fastText}~\citep{Joulin_et_al_2016} and its extension \Algo{Learned Tree}~\citep{Jernite_et_al_2017}, 
apply hierarchical softmax with a simple heurustic, we call \emph{pick-one-label}, 
which randomly picks one of the positive labels from a given training instance. 
The resulting instance is then treated as a multi-class instance. 
During prediction, the heuristic returns a multi-class distribution and the $k$ most probable labels. 
We show below that this specific reduction of the multi-label problem to multi-class classification is not consistent in general.

Since the probability of picking a label $j$ from $\by$ is equal to $y_j/\sum_{j'=1}^m y_{j'}$, 
the pick-one-label heuristic maps the multi-label distribution to a multi-class distribution in the following way:
\begin{equation}
\eta_j'(\bx) = \prob'(y_j = 1 \given \bx) = \sum_{\by \in \calY}  \frac{y_j}{\sum_{j'=1}^m y_{j'}}\prob(\by \given \bx) \,, j \in \labels \,.
\label{eq:heuristic}
\end{equation}
It can be easily checked that the resulting $\eta_j'(\bx)$ form a multi-class distribution 
as the probabilities sum up to 1. 
It is obvious that the heuristic changes the marginal probabilities of labels, 
unless the initial distribution is multi-class. 
Therefore this method cannot lead to a consistent classifier in terms of estimating $\eta_j(\bx)$. 
As we show below, it is also not consistent for precision@$k$ in general. 
\begin{proposition}
A classifier $\bh_{@k} \in \calH^{m}_{@k}$ predicting $k$ labels with highest $\eta_j'(\bx)$, $j \in \labels$, defined in (\ref{eq:heuristic}), has in general a non-zero regret in terms of precision@$k$.
\end{proposition}
\begin{proof}
We prove the proposition by giving a simple counterexample. 
Consider the following conditional distribution for some $\bx$: 
$$
\prob(\by = (1,0,0) \given \bx) = 0.1 \,,\quad \prob(\by = (1,1,0) \given \bx) = 0.5\,, \quad \prob(\by = (0,0,1) \given \bx) = 0.4\,.
$$
The optimal top 1 prediction for this example is obviously label $1$, 
since the marginal probabilities are $\eta_1(\bx) = 0.6, \eta_2(\bx) = 0.5,  \eta_3(\bx) = 0.4$. 
However, the pick-one-label heuristic will transform the original distribution to the following one: $\eta_1'(\bx) = 0.35, \eta_2'(\bx) = 0.25,  \eta_3'(\bx) = 0.4$. 
The predicted top label will be then label $3$, giving the regret of 0.2 for precision@$1$.  
\end{proof}

The proposition shows that the heuristic is in general inconsistent for precision@$k$.
Interestingly, the situation changes when the labels are conditionally independent, that is, 
$
\prob(\by \given \bx) = \prod_{j=1}^m \prob(y_i \given \bx)\,. 
$
\begin{restatable}{proposition}{hsmindependent}
\label{prop:hsm-independent}
Given conditionally independent labels, $\bh_{@k} \in \calH^{m}_{@k}$ predicting $k$ labels with highest $\eta_j'(\bx)$, $j \in \labels$, defined in (\ref{eq:heuristic}), has zero regret in terms of the precision@$k$ loss.
\end{restatable}
We show here only a sketch of the proof. 
The full proof is given in Appendix~\ref{app:hsm}. 
It is enough to show that in the case of conditionally independent labels 
the pick-one-label heuristic does not change the order of marginal probabilities.
Let $y_i$ and $y_j$ be so that $\eta_i(\bx) \ge \eta_j(\bx)$. 
Then in the summation over all $\by$s in (\ref{eq:heuristic}), 
we are interested in four different subsets of $\calY$,
$S_{i,j}^{u,w}  =  \{\by\in \calY: y_i = u \land y_j = w\}$, where $u,w \in \{0,1\}$.
Remark that during mapping none of $\by \in S^{0,0}_{i,j}$ plays any role, 
and for each $\by \in S^{1,1}_{i,j}$, the value of 
$$
y_t/(\sum_{t'=1}^m y_{t'}) \times \prob(\by \given \bx) \,,
$$ 
for $t \in \{i,j\}$, is the same for both $y_i$ and $y_j$. 
Now, let $\by' \in S^{1,0}_{i,j}$ and $\by'' \in S^{0,1}_{i,j}$ be the same 
on all elements except the $i$-th and the $j$-th one. 
Then, because of the label independence and the assumption that 
$\eta_i(\bx) \ge \eta_j(\bx) $, 
we have $\prob(\by' \given \bx) \ge \prob(\by'' \given \bx)$. 
Therefore, after mapping~(\ref{eq:heuristic}) we obtain $\eta_i'(\bx) \ge \eta_j'(\bx)$. 
Thus, for independent labels, the pick-one-label heuristic is consistent for precision@$k$.

%%%%%%%%%%%%%%%%%%%%%%%%%%%%%%%%%%%%%%%%%%%%%%%%%%%%%%%%%%%%%%%%%%%%%%%%%%%%%%%%%%%%%

%% file: 06-oplt.tex
\section{Online PLT}
\label{sec:oplt}

A \Algo{PLT} model can be trained either in batch mode or incrementally. 
The batch algorithm has been presented in Algorithm~\ref{alg:plt-learning} in Section~\ref{sec:plt}. 
It can be easily transformed into an incremental algorithm
operating sequentially on observations from $\calD = \{(\bx_i, \by_i) \}_{i=1}^n$.
%on a sequence $\calS = \{(\bx_t, \by_t)\}_t$ of observations, 
%where we use $t$ to index a possibly infinite number of training examples. 
To this end, we need to use an incremental learning algorithm $A_{\textrm{online}}$ in the tree nodes.
Such \emph{incremental PLT} (\Algo{IPLT}) is given in Algorithm~\ref{alg:plt-incremental-learning}.

\begin{algorithm}[H]
\caption{\Algo{IPLT.Train}$(T, A_{\textrm{online}}, \calD)$}
\label{alg:plt-incremental-learning}
\begin{small}
\begin{algorithmic}[1]
\State ${H_T} = \emptyset$  \Comment{Initialize a set of node probabilistic classifiers}
\For{each node $\vertex \in \vertices_{T}$} \Comment{For each node in the tree}
    \State {$\heta(v) = \textsc{NewClassifier}()$, $H_T = H_T \cup \{\heta(v)\}$} \Comment{Initialize its binary classifier.}
    %    \State $H = H \cup \{\heta(v)\}$ \Comment{Initialize its binary classifier.}
\EndFor
\For{$i = 1 \to n$}  \Comment{For each observation in the training sequence}
    \State $(P, N) = \mathrm{\textsc{AssignToNodes}}(T, \bx_i, \calL_{\bx_i})$ \Comment{Compute its positive and negative nodes}
    \For{$v \in P$} \Comment{For all positive nodes}
        \State $A_{\textrm{online}}\textsc{.Update}(\heta(v), (\bx_i, 1))$ \Comment{Update classifiers with a positive update with $\bx_i$.}
    \EndFor 
    \For{$v \in N$} \Comment{For each negative node}
        \State $A_{\textrm{online}}\textsc{.Update}(\heta(v), (\bx_i, 0))$ \Comment{Update classifiers with a negative update with $\bx_i$.}
    \EndFor
  \State \textbf{return} $H_T$ \Comment{Return the set of node probabilistic classifiers}
\EndFor
\end{algorithmic}
\end{small}
\end{algorithm} 

%%%%%%%%%%%%%%%%%%%%%%%%%%%%%%%%%%%%%%%%%%%%%%%%%%%%%%%%%%%%%%%%%%%%%%%%%%%%%%%%

The above algorithm, similarly as its batch counterpart, 
works on a finite training set and requires a tree structure $T$ to be given in advance. 
To construct $T$ at least the number $m$ of labels needs to be known.
More advance tree construction procedures, as discussed in Section~\ref{subsec:tree-structure}, 
exploit additional information like feature values or label co-occurrence~\citep{Prabhu_et_al_2018}. 
In all such algorithms, the tree is built in a batch mode prior to the learning of node classifiers.
Here, we analyze a different scenario in which 
an algorithm operates on a possibly infinite sequence of training instances 
and the tree is constructed online, 
simultaneously with incremental training of node classifiers,
without any prior knowledge of the set of labels or training data.
We refer to such approach as online probabilistic label trees. % (\Algo{OPLT}s). 

%We consider a setting in which at each time step $t$, a new observation $(\bx_t, \by_t)$ is observed.
Let us denote a sequence of observations by $\calS = \{(\bx_i, \calL_{\bx_i})\}_{i=1}^{\infty}$
and a subsequence consisting of the first $t$ instances by $\calS_t$.
%\todo{As the number of observed labels increases over time, instead of $m$, we use $|\calL_t|$ to denote the number of labels observed up to time $t$,  with $\calL_0 = \emptyset$. This number of also the length of vector $\by_t$, which however increases in this online scenario, does not impact the algorithms which operate on the positive labels from $\by_t$}.%
% previous version:
We refer here to labels of $\bx_i$ only by $\calL_{\bx_i}$, not using the vector notation $\by_i$. 
This is because the number of labels $m$ increases over time, 
which would also change the length of $\by_i$.% 
\footnote{The same applies to $\bx_t$ as the number of features also increases. 
We keep however the vector notation in this case, 
as it does not impact the description of the algorithm.}
Furthermore, let the set of labels observed in $\calS_t$ be denoted by $\calL_t$,
with $\calL_0 = \emptyset$. 
An online algorithm returns at step $t$ a tree structure $T_t$ 
constructed over labels in $\calL_t$ and a set of node classifiers $H_t$. 
Notice that the tree structure and the set of classifiers change in each iteration 
in which one or more new labels are observed. 
Below we discuss two properties that are desired for such online algorithm,
defined in relation to the \Algo{IPLT} algorithm given above.
%To let us consider the new online algorithm jointly, up to some point, with \Algo{PLT} and IPLT, we find the following property desired:
\begin{definition}[A proper online \Algo{PLT} algorithm]
\label{oplt-proper}
Let $T_t$ and $H_t$ be respectively a tree structure and a set of node classifiers 
trained on a sequence $\calS_t$ using an online algorithm $A$. 
We say that $A$ is a \emph{proper online \Algo{PLT} algorithm}, 
when for any $\calS$ and $t$ we have that
\begin{itemize}
\item $l_j \in L_{T_t}$ iff $j \in \labels_t$, that is, 
leaves of $T_t$ correspond to all labels observed in $S_t$, 
\item and $H_t$ is exactly the same as $H = \textsc{IPLT.Train}(T_t, A_{\textrm{online}}, \mathcal{S}_t)$, 
that is, node classifiers from $H_t$ are the same 
as the ones trained incrementally by Algorithm~\ref{alg:plt-incremental-learning} 
on $\calD = \calS_t$ and tree $T_t$ given as input parameter.
\end{itemize}
\end{definition}
\noindent
In other words, we require that whatever tree an online algorithm produces, 
the node classifiers should be trained the same way
as the tree would be know from the very beginning of training.
Thanks to that we can control the quality of each node classifier, 
as we are not missing any update. 
Moreover, since the result of a proper online \Algo{PLT} is the same as of \Algo{IPLT}, 
the same statistical guarantees apply to both of them.

%Having this property allows us to reason of $H_t$ as of a \Algo{PLT} classifier trained incrementally with IPLT algorithm on the same $\calS_t$ data.
%Also by using the same prediction procedure as \Algo{PLT} uses, the same predictions can be obtained.

The above definition can be satisfied by a naive algorithm 
that stores all observations seen so far,
use them in each iteration to build a tree, 
and train node classifiers with the \Algo{IPLT} algorithm. 
This approach is costly in terms of both memory, used for storing $S_t$, 
and time, as all computations are run from scratch in each iteration.
Therefore, we also demand an online algorithm to be space and time-efficient in the following sense. 
\begin{definition}[An efficient online \Algo{PLT} algorithm]
\label{oplt-efficient}
Let $T_t$ and $H_t$ be respectively a tree structure and a set of node classifiers 
trained on a sequence $\calS_t$ using an online algorithm $A$. 
Let $C_s$ and $C_t$ be the space and time training cost of \Algo{IPLT} 
trained on sequence $\calS_t$ and tree $T_t$.
An online algorithm is an \emph{efficient online \Algo{PLT} algorithm} 
when for any $S$ and $t$ we have its space and time complexity to be in a constant factor of $C_s$ and $C_t$, respectively.
%the space complexity of \Algo{IPLT} trained on the same sequence $\calS_t$ and $T_t$.
%the corresponding \ for a single observation $(\bx_t, \by_t)$ the time 
%We say that $A$ is a proper \Algo{OPLT} algorithm, 
\end{definition}
In this definition, we abstract from the actual implementation of \Algo{IPLT}. 
In other words, the complexity of an efficient online \Algo{PLT} algorithm 
depends directly on design choices for an~\Algo{IPLT}. 
Let us recall that the training cost for a single training example can be expressed by (\ref{eqn:learning_cost}),
as discussed in Section~\ref{subsec:complexity_analysis}.
By summing it over all examples in $\calS_t$, we obtain the cost $C_t$ of an~\Algo{IPLT}. 
The space complexity is upperbounded by $2m-1$ (the maximum number of node models), 
but it also depends on the chosen type of node models and the way of storing them
(see Section~\ref{sec:implementation} for a detailed discussion on implementation choices).
Let us also notice that the definition implies 
that the update of a tree structure has to be in a constant factor of the training cost of a single instance, 
given by (\ref{eqn:learning_cost}). 

\subsection{Online tree building and training of node classifiers}

Below we describe an online algorithm that, 
as we show in the next subsection, 
satisfies both properties defined above. 
%We call this algorithm \Algo{OPLT}. %, since it trains an \Algo{OPLT} in a proper and  efficient way. 
It is similar to the conditional probability tree (\Algo{CPT})~\citep{Beygelzimer_et_al_2009b},
introduced for \multiclass{} problems and binary trees, 
but extends it to \multilabel{} problems 
and trees of any arity. %adds additional elements to satisfy the two properties. 
We refer to this algorithm as \Algo{OPLT}.

The pseudocode is presented in Algorithms~\ref{alg:oplt}-\ref{alg:oplt-update_classifiers}.
In a nutshell, \Algo{OPLT} processes observations from $\calS$ sequentially, 
updating node classifiers. 
For new incoming labels it creates new nodes according to a chosen tree building policy
which is responsible for the main logic of the algorithm.
Each new node $v$ is associated with two classifiers, 
a regular one $\heta(v) \in H_T$, 
and an \emph{auxiliary} one $\hat\theta(v) \in \Theta_T$, 
where $H_T$ and $\Theta_T$ denote the corresponding sets of node classifiers. 
The task of the auxiliary classifiers is to accumulate positives updates. 
The algorithm uses them later to initialize classifiers in new nodes added to a tree.
%In general, they reside in all nodes of the tree, 
They can be removed if a given node will not be used anymore to extend the tree.
A particular criterion for removing an auxiliary classifier depends, however, on a tree building policy. 
%This could be, for example, whether a node has reached its maximum number of children.

\Algo{OPLT.Train}, outlined in Algorithm~\ref{alg:oplt}, administrates the entire process. 
It first initializes a tree with a root node $r_T$ only
and creates two corresponding classifiers, $\heta(v_{r_T})$ and $\hat\theta(v_{r_T})$.
Notice that the root has both classifiers initialized from the very beginning without a label assigned to it. 
Thanks to this, the algorithm can properly estimate the probability of $\prob(\calL_{\bx} = \emptyset \given \bx)$.
Observations from $\calS$ are processed sequentially in the main loop of \Algo{OPLT.Train}. 
If a new observation contains one or more new labels 
then the tree structure is appropriately extended by calling \textsc{UpdateTree}.
The node classifiers are updated in \textsc{UpdateClassifiers}. 
After each iteration $t$, the algorithm sends $H_T$ along with the tree structure $T$,
respectively as $H_t$ and $T_t$, to be used outside the algorithm for prediction tasks.
We assume that tree $T$ along with sets of its all nodes $V_T$ and leaves $L_T$, 
as well as sets of classifiers $H_T$ and $\Theta_T$,
are accessible to all subroutines discussed below.

\begin{algorithm}[H]
\caption{{\Algo{OPLT.Train}}$(\mathcal{S}, A_{\textrm{online}}, A_{\textrm{policy}})$} 
\label{alg:oplt}
\begin{small}
\begin{algorithmic}[1]
\State $r_T = \textsc{NewNode()}$, $V_T = \{r_T\}$ \Comment{Create the root of the tree}
\State {$\heta(r_T) = \textsc{NewClassifier}()$, $H_T = \{\heta_T(r_T)\}$} \Comment {Initialize a new classifier in the root}
\State {$\hat{\theta}(r_T) = \textsc{NewClassifier}()$, $\Theta_T = \{\theta(r_T)\}$} \Comment {Initialize an auxiliary classifier in the root}

\For{$(\bx_t, \calL_{\bx_t}) \in \calS$}  \Comment{For each observation in $\mathcal{S}$}
	\If{ $\calL_{\bx_t} \setminus \calL_{t-1} \neq \emptyset$} \Comment {If the observation contains new labels}
        \State \textsc{UpdateTree}$(\bx_t, \calL_{\bx_t}, A_{\textrm{policy}})$ \Comment{Add them to the tree}
    \EndIf
    \State \textsc{UpdateClassifiers}$(\bx_t, \calL_{\bx_t}, A_{\textrm{online}})$ \Comment {Update the classifiers}
    \State \textbf{send} $H_t, T_t = H_T, V_T$  \Comment{Send the node classifiers and the tree structure.}
\EndFor 
\end{algorithmic}
\end{small}
\end{algorithm}

\input{pics/oplt-tree_building}
Algorithm~\ref{alg:oplt-update_tree}, \Algo{UpdateTree}, builds the tree structure. 
%by adding new nodes and assigning new labels to leaves. 
It iterates over all new labels from $\calL_{\bx}$. 
If there were no labels in the sequence $\calS$ before, 
the first new label taken from $\calL_{\bx}$ is assigned to the root note. 
Otherwise, the tree needs to be extended by one or two nodes according to a selected tree building policy.
One of these nodes is a leaf to which the new label will be assigned. 
There are in general three variants of performing this step 
illustrated in Figure~\ref{fig:oplt-tree_building}.
The first one relies on selecting an internal node $v$
whose number of children is lower than the accepted maximum, 
and adding to it a child node $v''$ with the new label assigned to it.
In the second one, two new child nodes, $v'$ and $v''$, are added to a selected internal node $v$. 
Node $v'$ becomes a new parent of child nodes of the selected node $v$, 
that is, the subtree of $v$ is moved down by one level. 
Node $v''$ is a leaf with the new label assigned to it. 
The third variant is a modification of the second one.
The difference is that the selected node $v$ is a leaf node. 
Therefore there are no children nodes to be moved to $v'$,
but label of $v$ is reassigned to $v'$.
The $A_{\textrm{policy}}$ method encodes the tree building policy, 
that is, it decides which of the three variants to follow and selects the node $v$. 
The additional node $v'$ is inserted by the \textsc{InsertNode} method.
Finally, a leaf node is added by the \textsc{AddLeaf} method.
We discuss the three methods in more detail below.
\begin{algorithm}[H]
\caption{\Algo{OPLT.UpdateTree}$(\bx, \calL_{\bx},  A_{\textrm{policy}})$}
\label{alg:oplt-update_tree}
\begin{small}
\begin{algorithmic}[1]
\For{ $j \in \labels_{\bx} \setminus \calL_{t-1}$}\Comment{For each new label in the observation}
	\If{$\calL_T$ is $\emptyset$} \Comment{If no labels have been seen so far}
	    \State {$\textsc{label}(r_T) = j$} \Comment{Assign label $j$ to the root node}    
	\Else \Comment{If there are already labels in the tree.}
		\State {$v,\ {insert}  = A_{\textrm{policy}}(\bx, j, \calL_{\bx})$} \Comment{Select a variant of extending the tree}
		\IfThen{${insert}$}{$\textsc{InsertNode}(v)$} \Comment{Insert an additional node if needed.}
		\State{$\textsc{AddLeaf}(j, v)$}  \Comment{Add a new leaf for label $j$.}
	\EndIf
\EndFor
\end{algorithmic} 
\end{small}
\end{algorithm}

$A_{\textrm{policy}}$ returns the selected node $v$ and 
a Boolean variable $insert$ which indicates whether an additional node $v'$ has to be added to the tree.
For the first variant, $v$ is an internal node and $insert$ is set to false.
For the second variant, $v$ is an internal node and $insert$ is set to true.
For the third variant, $v$ is a leaf node and $insert$ is set to true.
In general, the policy can be guided by $\bx$, current label $j$, and set $\calL_{\bx}$ of all labels of $\bx$.
As an instance of the tree building policy, 
we consider, however, a much simpler method presented in Algorithm~\ref{alg:oplt-apply_policy}. 
It creates a $b$-ary complete tree.  
In this case, the selected node is either the leftmost internal node 
with the number of children less than $b$ 
or the leftmost leaf of the lowest depth. 
The $insert$ variable is then $false$ or $true$, respectively.
So, only the first and the third variants occur here.
Notice, however, that this policy can be efficiently performed in amortized constant time per label
if the complete tree is implemented using a dynamic array with doubling. 
Nevertheless, more advanced and computationally complex policies can be applied.
As mentioned before, the complexity of this step should be at most proportional 
to the complexity of updating the node classifiers for one label, 
that is, it should be proportional to the depth of the tree.
%
%For example, one can adapt the policy used in \Algo{CPET},
%originally introduced for multi-class problems and binary trees~\citep{Beygelzimer_et_al_2009b}.
%
\begin{algorithm}[H]
\caption{\Algo{BuildCompleteTree}$(b)$}
\label{alg:oplt-apply_policy}
\begin{small}
\begin{algorithmic}[1]
\State $\mathrm{array} = T.\mathrm{array}$ 
\Comment{Let nodes of complete tree $T$ be stored in a dynamic array $T.\mathrm{array}$}
% \State \todo{Linie wyzej mozna usunac tylko komentarz zostawic}
\State $s = \mathrm{array.length}$ \Comment{Read the number of nodes in $T$} 
\State $pa = \lceil \frac{s}{b} \rceil - 1$ \Comment{Get the index of a parent of a next added node; the array is indexed from 0}
\State $v = \mathrm{array}(pa)$  \Comment{Get the parent node}
\State \textbf{return} $v$, $\textsc{IsLeaf}(v)$ \Comment{Return the node and whether it is a leaf.}
\end{algorithmic} 
\end{small}
\end{algorithm}

The \textsc{InsertNode} and \textsc{AddLeaf} procedures involve specific operations 
concerning initialization of classifiers in the new nodes.
\textsc{InsertNode} is given in Algorithm~\ref{alg:oplt-insert_node}. 
It inserts a new node $v'$ as a child of the selected node $v$.
If $v$ is a leaf then its label is reassigned to the new node.
Otherwise, all children of $v$ become the children of $v'$.
In both cases, $v'$ becomes the only child of $v$.
Figure~\ref{fig:oplt-tree_building} illustrates inserting $v'$ as 
either a child of an internal node (c) or a leaf node (d). 
Since, the node classifier of $\node'$ aims at estimating $\eta(\bx, \node')$,
defined as $\prob(z_{\node'} = 1 \given z_{\pa{\node'}} = 1, \bx)$,
its both classifiers, $\heta(v')$ and $\hat\theta(v')$, 
are initialized as copies (by calling the \textsc{Copy} function) 
of the auxiliary classifier $\hat\theta(v)$ of the parent node $v$.
Recall that the task of auxiliary classifiers is to accumulate all positive updates in nodes,
so the conditioning $z_{\pa{\node'}} = 1$ is satisfied in that way.
\begin{algorithm}[H]
\caption{\Algo{OPLT.InsertNode}$(v)$}
\label{alg:oplt-insert_node}
\begin{small}
\begin{algorithmic}[1]
    \State{$v' = \textsc{NewNode}$(), $V_T = V_T \cup \{v'\}$} \Comment{Create a new node and add it to the tree nodes}
    \If{$\textsc{IsLeaf}(v)$} \Comment{If node $v$ is a leaf}
        \State{$\textsc{Label}(v') = \textsc{Label}(v)$, $\textsc{Label}(v) = \textsc{Null}$} \Comment{Reassign label of $v$ to $v'$}
    \Else \Comment{Otherwise}
        \State{$\childs{v'} = \childs{v}$}  \Comment{All children of $v$ become children  of $v'$}
        \ForDo{$v_{\textrm{ch}} \in \childs{v'}$}{$\pa{v_{\textrm{ch}}} = v'$} \Comment{And $v'$ becomes their parent} 
    \EndIf
    \State{$\childs{v} = \{v'\}$, $\pa{v'} = v$} \Comment{The new node $v'$ becomes the only child of $v$}
    \State{$\heta(v') = \textsc{Copy}(\hat\theta(v))$, $H_T = H_T \cup \{\heta(v')\}$} \Comment{Create a classifier.}
    \State{$\hat{\theta}(v') = \textsc{Copy}(\hat\theta(v))$, $\Theta_T = \Theta_T \cup \{\hat\theta(v')\}$} \Comment{And an auxiliary classifier.}
\end{algorithmic} 
\end{small}
\end{algorithm}

Algorithm~\ref{alg:oplt-add_leaf} outlines the \textsc{AddLeaf} procedure. 
It adds a new leaf node $v''$ for label $j$ as a child of node $v$ . 
The classifier $\heta(v'')$ is created as an ``inverse''
of the auxiliary classifier $\hat\theta(v)$ from node $v$.
More precisely, the $\textsc{InverseClassifier}$ procedure creates a wrapper 
inverting the behavior of the base classifier.
It predicts $1 - \heta$, where $\heta$ is the prediction of the base classifier,
and flips the updates, that is, positive updates become negative and negative updates become positive. 
Finally, the auxiliary classifier $\hat{\theta}(v'')$ of the new leaf node is initialized.  
%The three different variants of adding a leaf node to a parent node are illustrate in Figures~\ref{fig:oplt-tree_building}b-d.
%
\begin{algorithm}[H]
\caption{\Algo{OPLT.AddLeaf}$(j, v)$}
\label{alg:oplt-add_leaf}
\begin{small}
\begin{algorithmic}[1]
    \State $v'' = \textsc{NewNode()}$, $V_T = V_T \cup \{v''\}$ \Comment{Create a new node and add it to the tree nodes}
    \State $\childs{v} = \childs{v} \cup \{v''\}$,  $\pa{v''} = v$   \Comment{Add this node to children of $v$.}
    \State {$\textsc{label}(v'') = j$} \Comment{Assign label $j$ to the node $v''$}    
    \State {$\heta(v'') = \textsc{InverseClassifier}(\hat\theta(v))$, $H_T = H_T \cup \{ \heta(v'')\} $} \Comment{Initialize a classifier for $v''$}
    \State{$\hat{\theta}(v'') = \textsc{NewClassifier}()$, $\Theta_T = \Theta_T \cup \{ \hat{\theta}(v'')\}$} \Comment{Initialize an auxiliary classifier for $v''$}
\end{algorithmic} 
\end{small}
\end{algorithm}

The final step in the main loop of \Algo{OPLT.Train} updates the node classifiers. 
The regular classifiers, $\heta(v) \in H_T$, 
are updated exactly as in \Algo{IPLT.TRAIN} given in Algorithm~\ref{alg:plt-incremental-learning}. 
The auxiliary classifiers, $\theta(v) \in \Theta_T$,  
are updated only in positive nodes according to their definition and purpose.
\begin{algorithm}[H]
\caption{\Algo{OPLT.UpdateClassifiers}$(\bx, \calL_{\bx}, A_{\textrm{online}})$}
\label{alg:oplt-update_classifiers}
\begin{small}
\begin{algorithmic}[1]
    \State $(P, N) = \textsc{AssignToNodes}(T, \bx, \calL_{\bx})$ \Comment{Compute its positive and negative nodes}
    \For{$v \in P$} \Comment{For all positive nodes}
         \State $A_{\textrm{online}}\textsc{.Update}(\heta(v), (\bx, 1))$ \Comment{Update classifiers with a positive update with $\bx$.}
    	\If{$\hat{\theta}(v) \in \Theta$} \Comment{Update auxiliary classifier if it exists.}
            \State $A_{\textrm{online}}\textsc{.Update}(\hat \theta(v), (\bx, 1))$ \Comment{With a positive online update with $\bx_i$.}
        \EndIf
    \EndFor
    \For{$v \in N$} \Comment{For each negative node}
        \State $A_{\textrm{online}}\textsc{.Update}(\heta(v), (\bx, 0))$ \Comment{Update classifiers with a negative update with $\bx$.}
    \EndFor
% 	\State \textbf{return} ${H}$.
\end{algorithmic} 
\end{small}
\end{algorithm}

\subsection{Theoretical analysis of OPLT}

The \Algo{OPLT} algorithm has been designed to satisfy the properness and efficiency property of 
online probabilistic label trees. 
The theorem below states this fact formally.
\begin{restatable}{theorem}{thmoplt}
\label{thm:oplt}
\Algo{OPLT} is an proper and efficient \Algo{OPLT} algorithm. 
\end{restatable}
The proof is quite technical and we present it in Appendix~\ref{app:oplt}. 
{To show the properness, 
it uses induction for both the outer and inner loop of the algorithm,
where the outer loop iterates over observations $(\bx_t, \calL_{\bx_t})$, 
while the inner loop over new labels in $\calL_{\bx_{t}}$. }
The key elements used to prove this property are the use of the auxiliary classifiers 
and the analysis of the three variants of the tree structure extension.
{The efficiency is proved by noticing that each node has two classifiers, and the algorithm creates 
and updates no more than one additional classifier per node comparing to \Algo{IPLT}. }
Moreover, any node selection policy which cost is proportional to the cost 
of updating \Algo{IPLT} classifiers for a single label
meets the efficiency requirement. 
Particularly, the policy building a complete tree presented above satisfies this constraint. 

The \Algo{OPLT} algorithm aims at constructing the node classifiers in such a way 
that its properness can be met by a wide range of tree building policies. 
The naive complete tree policy was introduced mainly for ease of presentation.  
One can, for example, easily adapt and further extend the policy originally used in  \Algo{CPT}~\citep{Beygelzimer_et_al_2009b}.
In short, the \Algo{CPT} policy selects a node $v$ 
which trade-offs balancedness of the tree and a fit of $\bx$, that is, the value of $\heta_v(\bx)$.
Since it works with binary trees only, 
the policy uses solely the third variant of the tree extension. 
Moreover, it was designed for \multiclass{} problems. 
In~\citep{Jasinska_et_al_2020} we have considered such extension.
From this point of view, the presented framework significantly extends \Algo{CPT}.
It solves both types of problems, \multiclass{} and \multilabel, 
and can be used with more advanced policies 
that exploit all three variants of the tree extension. 

%% file: pics/oplt-tree_building.tex
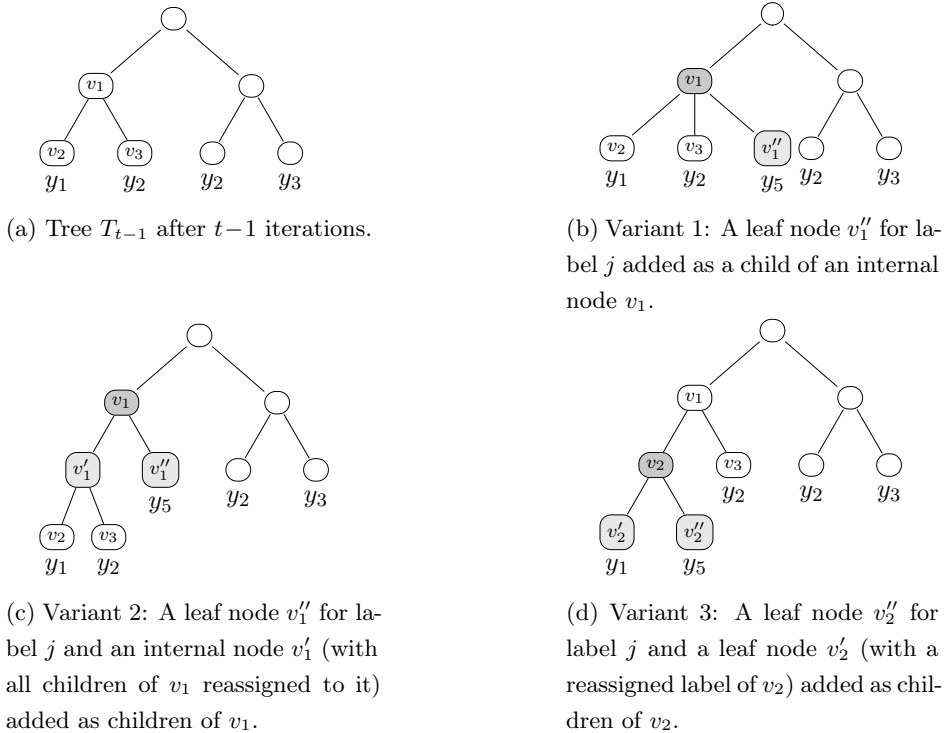
\begin{figure}[ht]
\begin{tabular}{p{0.10\textwidth}p{0.3\textwidth}p{0.10\textwidth}p{0.3\textwidth}p{0.10\textwidth}}
%%%%%%%%%%%%%%%%%%%%%%%%%%%%%%%%%%%%%%%%%%%%%%%%%%%%%%%%%%%%%%%%%%%%%%%%%%%%%%%%%%%% 
&
\scalebox{0.9}{
\begin{tikzpicture}[scale = 1,every node/.style={scale=1},
		regnode/.style={circle,draw,minimum width=1.5ex,inner sep=0pt},
		leaf/.style={circle,fill=black,draw,minimum width=1.5ex,inner sep=0pt},
		pleaf/.style={rectangle,rounded corners=1ex,draw,font=\scriptsize,inner sep=3pt},
		pnode/.style={rectangle,rounded corners=1ex,draw,font=\scriptsize,inner sep=3pt},
		rootnode/.style={rectangle,rounded corners=1ex,draw,font=\scriptsize,inner sep=3pt},
		level/.style={sibling distance=6em/#1, level distance=6ex}
	]
	\node (z) [rootnode,] {\color{white}{$v$}}
    child {node [fill=white!10!white!90] (a) [pnode] {$v_1$} 
    		child {node [fill=white!10!white!90, label=below:{\makebox[1cm][c]{$y_1$}}] (a1) [pnode] {$v_2$} 
    		edge from parent node[above left]{}
    }
    child {node [label=below:{\makebox[1cm][c]{$y_2$}},fill=white!10!white!90] (g) [pleaf] {$v_3$} edge from parent node[above]{}}
    		edge from parent node[above left]{}
    }	
	child {node (j) [pnode] {\color{white}{$v$}}
		child {node [label=below:{\makebox[1cm][c]{$y_2$}}] (k) [pleaf] {\color{white}{$v$}} edge from parent node[above left]{}}
		child {node [label=below:{\makebox[1cm][c]{$y_3$}}] (l) [pleaf] {\color{white}{$v$}}
			{
				child [grow=right] {node (s) {} edge from parent[draw=none]
					child [grow=up] {node (t) {} edge from parent[draw=none]
						child [grow=up] {node (u) {} edge from parent[draw=none]}
					}
				}
			}
			edge from parent node[above right]{}
		}
		edge from parent node[above right]{}
	};
\end{tikzpicture}}
& 
&
%%%%%%%%%%%%%%%%%%%%%%%%%%%%%%%%%%%%%%%%%%%%%%%%%%%%%%%%%%%%%%%%%%%%%%%%%%%%%%%%%%%%
\scalebox{0.9}{
\begin{tikzpicture}[scale = 1,every node/.style={scale=1},
    		regnode/.style={circle,draw,minimum width=1.5ex,inner sep=0pt},
    		leaf/.style={circle,fill=black,draw,minimum width=1.5ex,inner sep=0pt},
    		pleaf/.style={rectangle,rounded corners=1ex,draw,font=\scriptsize,inner sep=3pt},
    		pnode/.style={rectangle,rounded corners=1ex,draw,font=\scriptsize,inner sep=3pt},
    		rootnode/.style={rectangle,rounded corners=1ex,draw,font=\scriptsize,inner sep=3pt},
    		level/.style={sibling distance=6em/#1, level distance=6ex},
    		level 3/.style={sibling distance=3em, level distance=6ex}
    	]
    	\node (z) [rootnode] {\color{white}{r}}
    	child {node [fill=black!30!white!70] (a) [pnode] {$v_1$} 
    		child {node [fill=white!10!white!90, label=below:{\makebox[1cm][c]{$y_1$}}] (a1) [pleaf] {$v_2$} 
    		edge from parent node[above left]{}
    	    }
    	    child {node [fill=white!10!white!90,
    	    label=below:{\makebox[1cm][c]{$y_2$}}] (g) [pleaf] {$v_3$} edge from parent node[above]{}
    	    }
    		child {node [fill=black!10!white!90, label=below:{\makebox[1cm][c]{$y_5$}}] (a2) [pleaf] {$v''_1$} 
    		edge from parent node[above right]{}
    	    }
    	    edge from parent node[above left]{}   
        }
    	child {node (j) [pnode] {\color{white}{$v$}}
    		child {node [label=below:{\makebox[1cm][c]{$y_2$}}] (k) [pleaf]  {\color{white}{$v$}} edge from parent node[above left]{}}
    		child {node [label=below:{\makebox[1cm][c]{$y_3$}}] (l) [pleaf]  {\color{white}{$v$}} edge from parent node[above right]{}
    		}
    		edge from parent node[above right]{}
    	};
\end{tikzpicture}}
&
\\
%%%%%%%%%%%%%%%%%%%%%%%%%%%%%%%%%%%%%%%%%%%%%%%%%%%%%%%%%%%%%%%%%%%%%%%%%%%%%%%%%%%%
&
{\footnotesize (a) Tree $T_{t-1}$ after $t\!-\!1$ iterations.}
&
& 
{\footnotesize (b) Variant 1: A leaf node $v_1''$ for label $j$ added as a child of an internal node $v_1$.}
& \\[18pt]
%%%%%%%%%%%%%%%%%%%%%%%%%%%%%%%%%%%%%%%%%%%%%%%%%%%%%%%%%%%%%%%%%%%%%%%%%%%%%%%%%%%% 
&
\scalebox{0.9}{
\begin{tikzpicture}[scale = 1,every node/.style={scale=1},
		regnode/.style={circle,draw,minimum width=1.5ex,inner sep=0pt},
		leaf/.style={circle,fill=black,draw,minimum width=1.5ex,inner sep=0pt},
		pleaf/.style={rectangle,rounded corners=1ex,draw,font=\scriptsize,inner sep=3pt},
		pnode/.style={rectangle,rounded corners=1ex,draw,font=\scriptsize,inner sep=3pt},
		rootnode/.style={rectangle,rounded corners=1ex,draw,font=\scriptsize,inner sep=3pt},
		level/.style={sibling distance=6em/#1, level distance=6ex}
	]
	\node (z) [rootnode] {\color{white}{$v$}}
	child {node [fill=black!30!white!70] (a) [pnode] {$v_1$}
	    child {node [fill=black!10!white!90] (a1) [pnode] {$v'_1$} 
	        child{node [fill=white!10!white!90, label=below:{\makebox[1cm][c]{$y_1$}}] (a11) [pleaf] {$v_2$}
	        edge from parent node[above left]{}
	        }
	        child{node [fill=white!10!white!90, label=below:{\makebox[1cm][c]{$y_2$}}] (a12) [pleaf] {$v_3$}
	        edge from parent node[above right]{}
	        }
    	edge from parent node[above left]{}
    	}
    	child {node [fill=black!10!white!90,
    	label=below:{\makebox[1cm][c]{$y_5$}}] (g) [pleaf] {$v''_1$} 
    	edge from parent node[above]{}
        }
		edge from parent node[above left]{}
	}
	child {node (j) [pnode] {\color{white}{$v$}}
		child {node [label=below:{\makebox[1cm][c]{$y_2$}}] (k) [pleaf] {\color{white}{$v$}} edge from parent node[above left]{}}
		child {node [label=below:{\makebox[1cm][c]{$y_3$}}] (l) [pleaf] {\color{white}{$v$}}
			{
				child [grow=right] {node (s) {} edge from parent[draw=none]
					child [grow=up] {node (t) {} edge from parent[draw=none]
						child [grow=up] {node (u) {} edge from parent[draw=none]}
					}
				}
			}
			edge from parent node[above right]{}
		}
		edge from parent node[above right]{}
	};
\end{tikzpicture}}
& 
&
%%%%%%%%%%%%%%%%%%%%%%%%%%%%%%%%%%%%%%%%%%%%%%%%%%%%%%%%%%%%%%%%%%%%%%%%%%%%%%%%%%%%
\scalebox{0.9}{
\begin{tikzpicture}[scale = 1,every node/.style={scale=1},
    		regnode/.style={circle,draw,minimum width=1.5ex,inner sep=0pt},
    		leaf/.style={circle,fill=black,draw,minimum width=1.5ex,inner sep=0pt},
    		pleaf/.style={rectangle,rounded corners=1ex,draw,font=\scriptsize,inner sep=3pt},
    		pnode/.style={rectangle,rounded corners=1ex,draw,font=\scriptsize,inner sep=3pt},
    		rootnode/.style={rectangle,rounded corners=1ex,draw,font=\scriptsize,inner sep=3pt},
    		level/.style={sibling distance=6em/#1, level distance=6ex},
    		level 3/.style={sibling distance=3em, level distance=6ex}
    	]
	\node (z) [rootnode] {\color{white}{$v$}}
	child {node [fill=white!10!white!90] (a) [pnode] {$v_1$}
	    child {node [fill=black!30!white!70] (a1) [pnode] {$v_2$} 
	        child{node [fill=black!10!white!90, label=below:{\makebox[1cm][c]{$y_1$}}] (a11) [pleaf] {$v'_2$}
	        edge from parent node[above left]{}
	        }
	        child{node [fill=black!10!white!90, label=below:{\makebox[1cm][c]{$y_5$}}] (a12) [pleaf] {$v''_2$}
	        edge from parent node[above right]{}
	        }
    	edge from parent node[above left]{}
    	}
    	child {node [fill=white!10!white!90,
    	label=below:{\makebox[1cm][c]{$y_2$}}] (g) [pleaf] {$v_3$} 
    	edge from parent node[above]{}
        }
		edge from parent node[above left]{}
	}
	child {node (j) [pnode] {\color{white}{$v$}}
		child {node [label=below:{\makebox[1cm][c]{$y_2$}}] (k) [pleaf] {\color{white}{$v$}} edge from parent node[above left]{}}
		child {node [label=below:{\makebox[1cm][c]{$y_3$}}] (l) [pleaf] {\color{white}{$v$}}
			{
				child [grow=right] {node (s) {} edge from parent[draw=none]
					child [grow=up] {node (t) {} edge from parent[draw=none]
						child [grow=up] {node (u) {} edge from parent[draw=none]}
					}
				}
			}
			edge from parent node[above right]{}
		}
		edge from parent node[above right]{}
    	};
\end{tikzpicture}}
&
\\
%%%%%%%%%%%%%%%%%%%%%%%%%%%%%%%%%%%%%%%%%%%%%%%%%%%%%%%%%%%%%%%%%%%%%%%%%%%%%%%%%%%%
&
{\footnotesize (c) Variant 2: A leaf node $v_1''$ for label $j$ and an internal node $v_1'$ (with all children of $v_1$ reassigned to it) added as children of $v_1$.}
&
& 
{\footnotesize (d) Variant 3: A leaf node $v_2''$ for label $j$ and a leaf node $v_2'$ (with a reassigned label of $v_2$) added as children of $v_2$.} 
& \\
\end{tabular}
\caption{%
Three variants of tree extension for a new label $j$. % 
%Panel (a) presents an exemplary tree $T_{t-1}$ after $t$-th iteration with 3 nodes tagged by $v_1$, $v_2$, and $v_3$. %   
%In panel (b) presents the first variant in which a new leaf $v_1^{''}$ for label $j$ is added %
%directly to the selected node $v_1$ without adding any new nodes. % 
%In this case $v_1$ accepts at least 3 child nodes. %
%Panel (c) presents the second variant in which the new leaf $v_1^{''}$ is also added directly to $v_1$, %
%but this time $v_1$ can have only 2 child nodes. %
%Therefore, a new intermediate node $v_1^'$ is added as a new child of $v_1$ and %
%all previous child nodes of $v_1$ are moved down as children of  $v_1^'i$. %
%Finally, panel (d) illustrates the third variant in which the left node $v_2$ is selected. %
%This variant is similar to the second one with the difference that $v_2$ does not have any children to be moved down %
%as new children of the new node $v_2^'$. %
%Label $j$ is associated with the new leaf node $v_2^''$. %
}
\label{fig:oplt-tree_building}
\end{figure}

%% file: 07-implementation.tex
\section{Implementation}
\label{sec:implementation}

\begin{sloppypar}
There are several {popular} packages that implement the \Algo{PLT} model, for example, 
\Algo{XMLC-PLT}~\citep{Jasinska_et_al_2016},\footnote{\url{https://github.com/busarobi/XMLC}}
\Algo{PLT-vw},\footnote{\url{https://github.com/VowpalWabbit/vowpal_wabbit}}
\Algo{Parabel}~\citep{Prabhu_et_al_2018},\footnote{\url{http://manikvarma.org/code/Parabel/download.html}}
\Algo{extremeText}~\citep{Wydmuch_et_al_2018},\footnote{\url{https://github.com/mwydmuch/extremeText}} 
\Algo{AttentionXML}~\citep{You_et_al_2019},\footnote{\url{{https://github.com/yourh/AttentionXML}}}
\Algo{Bonsai}~\citep{Khandagale_et_al_2019},\footnote{\url{https://github.com/xmc-aalto/bonsai}}
or \Algo{napkinXC}\footnote{\url{https://github.com/mwydmuch/napkinXC}} that we introduce in this paper.
In this section, we discuss the differences between them in terms of 
training node classifiers, 
dealing with sparse and dense features, 
%online and batch learning of node classifiers, 
efficient prediction, % via different tree search techniques,
tree structure learning,
and ensembling.
%and parallelization of computations. %, algorithms for learning tree structures over the labels.
Table~\ref{tab:implementations-compared} summarizes the differences between the discussed implementations. 
At the end of this section we also shortly discuss a different approach to obtain $\heta_{\node}(\bx)$, 
which uses multi-class probability estimation instead of binary probability estimation.
\end{sloppypar}

\begin{table}[!h]
    \centering
    \footnotesize
    \begin{tabular}{l |p{.1\textwidth} p{.1\textwidth} p{.2\textwidth} p{.15\textwidth} p{.08\textwidth}}
    \toprule
    Implementation & node        & represen- & prediction & tree            & ensem- \\
              & classifiers & tation    &            & structure       & bling  \\
    \midrule
    \Algo{XMLC-PLT}/     & online   & sparse    & online/             & complete tree          & no         \\
    \Algo{PLT-vw}        &          &           & unif.-cost search   & based on freq.         &            \\  
    \hline
    \Algo{Parabel}       & batch    & sparse    & batch/              & balanced               & yes \\
                         &          &           & beam search         & h.~$2$-means           & \\
    \hline
    \Algo{Bonsai}        & batch    & sparse    & batch/              & unbalanced             & yes \\
                         &          &           & beam search         & h.~$k$-means           & \\
    \hline
    \Algo{extremeText}  & online   & dense     & online/             & h.~$k$-means            & yes        \\
                        &          &           & unif.-cost search.  &                         &            \\  
    \hline
    \Algo{AttentionXML} & online \&   & dense  & batch \& leveled/  & shallow                  & no         \\
                        & leveled  &           & beam search         & h. $k$-means            &            \\  
    \hline
    \Algo{napkinXC}     & both     & both      & online/            & any             & yes        \\
                        &          &           & {unif.-cost search, } &                 &            \\  
                        &          &           & {thresholds-based } &                 &            \\  
    
    \bottomrule
    \end{tabular}
    \caption{Comparison of different implementations of the \Algo{PLT} model in terms of node classifiers (online, batch, or both), representation of features (sparse, dense, both), prediction algorithm (online, batch, leveled/beam search, uniform cost search, thresholds-based), tree structure learning (complete tree based on frequencies, hierarchical $k$ means, their shallow variant, or any), ensembling (yes, no). All the options are described in text.}
    \label{tab:implementations-compared}
\end{table}

\subsection{Training of node classifiers}

Given the tree structure, the node classifiers of \Algo{PLT}s can be trained either in online or batch mode. 
Both training modes have their pros and cons.
The batch variant, 
implemented for example in \Algo{Parabel}, 
can benefit from using well-known batch solvers, such as \Algo{LIBLINEAR}~\citep{LIBLINEAR}. 
These variants are relatively easy to train and achieve high predictive performance.
Moreover, each model can be trained independently which enables a simple parallelization of training.

In turn, the online variant, such as~\Algo{XMLC-PLT}, \Algo{PLT-vw}, or~\Algo{extremeText}, 
can be applied to stream data. % and demand less operational memory. 
However, all those implementations
% both implementations mentioned above
demand a tree structure to be known prior to the training of node classifiers. 
Nevertheless, the online node learners give the possibility of combining them with the online tree construction, 
as discussed in the previous section.
Moreover, they can benefit from using deep networks to learn complex representation of input instances. 
Parallelization can be performed similarly as in the case of batch models. 
If accepting additional conflicts during updates of the node models, 
we can apply parallelization on the level of single examples as in~\Algo{extremeText}. 
Each thread consumes a part of the training examples and updates model allocated in shared memory.

\Algo{napkinXC} follows a modular design, therefore it can be easily used with either batch or online node learners.
In the latter case, it has an implemented functionality of collaborating with external learners 
to exchange the ``forward'' and ``backward'' signals. 
It also supports the online tree construction. 

\subsection{Sparse features}
\label{sec:sparse-features}

For problems with sparse features, such as text classification, 
\Algo{PLT}s can be efficiently implemented using different approaches. 
The simplest and naive one relies on transforming sparse representation to dense one,
training a given classifier using this representation, 
and then storing the final model in the sparse representation again. 
For example, if someone wants to use the popular \Algo{LIBLINEAR} package, 
this is the approach to go as this package uses dense representation internally. 
The resulting model can be stored as sparse, mainly if the model has been trained with $L_1$ regularization.
However, in the case of $L_2$ regularization one can remove all weights being close to zero, 
similarly as in \Algo{Dismec}~\citep{Babbar_Scholkopf_2017}. 
Since the node classifiers can be trained independently, 
the runtime memory of this approach can also be optimized. 

Alternatively, one can use sparse learning algorithms. 
Such algorithms follow usually the online/incremental learning paradigm, 
in which training instances are processed sequentially one-by-one. 
Instances of such algorithms are Fobos~\citep{fobos} or AdaGrad~\citep{adagrad}.
To store and update weights they use either hash maps or feature hashing~\citep{Weinberger_et_al_2009}.
The latter, implemented for example in the popular \Algo{Vowpal Wabbit} package~\citep{vw},  
relies on allocating a constant memory space for features weights.
Since the allocated space can be too small, conflicts can exist between different weights.
They are not resolved, that is, a weight is shared by all conflicting features. 
If used with \Algo{PLT}s, the allocated memory can be shared by all node models. 
In case of hash maps, one needs to reallocate the memory if the map is close to be full. 
In our implementation we follow the \Algo{Robin Hood Hashing}~\citep{Celis_et_al_1985} 
which allows for very efficient insert and find operations, 
having at the same time minimal memory overhead. 
It uses open addressing, but compared to hash maps with linear and quadratic probing,
it significantly reduces the expected average and maximum probe lengths.
This is achieved by shifting the keys around in such a way 
that all keys stay reasonably close to the slot they hash to. 
When inserting a new element, 
if the probe length for the existing element is less than the current probe length for the element being inserted, 
\Algo{Robin Hood} swaps the two elements and continues the procedure. 
This results in much lower average probe lengths as well as its variance.
This also allows for a straightforward lookup algorithm 
that ignores empty slots and 
keeps looking for a key until it reaches the known maximum probe length for the whole table. 
This property also allows usage of high load factors (higher than 0.9). 
Since it uses open addressing, 
the memory usage for the whole map is very close to memory needed to store its content as a sparse vector.

It is worth noticing that for sparse data the weights sparsity increases with the depth of a tree.
This implies a significant reduction of space of the final \Algo{PLT} model. 
Paradoxically, this reduction can be the largest in case of binary trees, 
although the number of nodes is the highest in this case, 
(equal to $2m-1$, being as much as twice the number of models in the \Algo{1-vs-all} approach). 
This is because the models use only non-zero features of the sibling nodes, 
and there are only two such nodes in binary trees. 
No other features are needed to build the corresponding classifiers.

\subsection{Dense features}

\Algo{PLT}s can also work with dense features, however, the dimensionality of the feature space cannot be too high. 
Otherwise, the memory used for models would be too large. 
The dense representation is usually connected with deep or shallow neural networks. % This idea is presented on Figure~\ref{pic:model-embedding}.

One possibility is to use pretrained embeddings. 
For text classification, one can use word representations trained by \Algo{word2vec}~\citep{Mikolov_et_al_2013} 
or \Algo{GloVe}~\citep{Pennigton_et_al_2014} on large text corpuses. 
The document representation can be then created from these word embeddings in many ways, for example, 
as an average or as the maximum or minimum value of each element of the embeddings~\citep{De_Boom_et_al_2016}.
Alternatively, one can train the word embeddings simultaneously with the node classifiers, similarly as in \Algo{FastText}~\citep{Joulin_et_al_2016}. 
This approach is taken in \Algo{extremeText}.
Another option is to initiate the network with the pretrained embeddings 
and then update all the parameters of both the node classifiers and text representations.

To improve the document representation, instead of a simple aggregation over words, one can use word embeddings in a more advanced deep architecture, 
such as LSTM ~\citep{Hochreiter_Schmidhuber_1997} or the text-based convolution neural network~\citep{Liu_et_al_2017}. 
\Algo{PLT}s can be used with such architecture as the output layer. 
However, the speed advantage of this architecture might not be so visible in the case of GPU-based training, 
as matrix multiplication can be efficiently performed on GPUs. 
Nevertheless, in the case of complex and memory-intensive approaches, 
the \Algo{PLT} approach can be used to decompose the problem in such a way
that computations are performed level-by-level in a tree (with additional decomposition possible on a given level).
All the layers except the \Algo{PLT} one are initialized using the trained values from the preceding level. 
This idea is followed in \Algo{AttentionXML}. 

\subsection{Prediction}
\label{subsec:implementation-prediction}

The top-$k$ prediction, 
discussed in Section~\ref{subsec:plt-prediction} as Algorithm~\ref{alg:ucs_prediction},  
is a variant of the uniform-cost search. 
It is used in \Algo{XMLC-PLT}, \Algo{extremeText}, and \Algo{napkinXC}. 
It has the advantage of being very efficient for online prediction. 
The algorithm also allows for extending $k$ at any time, 
without restarting the whole procedure from the beginning. 
If one does not need to change $k$, 
the algorithm can be improved by adding to the priority queue only those nodes 
whose probability is greater than the probability of the $k$-th top leaf already added to the queue.
In case of sparse models, they should be stored either in hash maps or use the feature hashing 
to allow random access to model weights. 

Unpacking a sparse model to dense representation, as implemented in \Algo{Parabel}, 
can be very costly in case of online prediction. 
However, for sufficiently large batches, 
the approach taken in \Algo{Parabel} can benefit from beam search.
In this case, a node classifier is unpacked once for a batch of test examples 
allowing a very efficient computation of the dot products. 
If models are dense and they can be all load to the main memory, 
then both search methods perform similarly in terms of computational times.
However, beam search is an approximate method 
and it may not find the actual top $k$ labels with the highest estimates of label probabilities. 
Therefore it may suffer regret for precision$@k$~\citep{Zhuo_et_al_2020}. 
In the case of memory-intensive deep models, 
such as the one used in \Algo{AttentionXML}, 
the prediction is usually performed level-by-level, similarly as training, 
and, therefore, it uses beam search.
Nevertheless, there also exists a variant of uniform-cost search 
efficiently operating in a batch mode~\citep{Jasinska_2018}.

\subsection{Tree structure}
\label{subsec:tree-structure}

The tree structure of a \Algo{PLT} is a crucial modeling decision. 
The theoretical results from Section~\ref{sec:analysis} concerning the vanishing regret of \Algo{PLT} 
hold regardless of the tree structure, 
however, this theory requires the regret of the node classifiers also to vanish. 
In practice, we can only estimate the conditional probabilities in the nodes, 
therefore the tree structure does indeed matter as it affects the difficulty of the node learning problems. 
In Section~\ref{sec:oplt}, we already discussed the problem of building a tree in the online setting. 
Here, we focus on batch approaches which assume that labels are known. 
%They also rely on the label distribution and feature values. 
The original \Algo{PLT} paper~\citep{Jasinska_et_al_2016} uses simple complete trees with labels assigned to leaves according to their frequencies. 
Another option, routinely used in \Algo{HSM}~\citep{Joulin_et_al_2016}, is the Huffman tree built over the label frequencies. 
Such tree takes into account the computational complexity by putting the most frequent labels close to the root. 
This approach has been further extended to optimize GPU operations in~\citep{Grave_et_al_2017}. 
Unfortunately, for multi-label classification the Huffman tree is no longer optimal in terms of computational cost. 
As already mentioned in Section~\ref{subsec:complexity_analysis}, 
an exhaustive analysis of computational complexity of \Algo{PLT}s has been performed by~\citet{Busa-Fekete_et_al_2019}. 
Furthermore, Huffman trees ignore the statistical properties of the tree structure.
There exist, however, other methods that focus on building a tree with high overall accuracy~\citep{Tagami_2017,Prabhu_et_al_2018}. 

The method of~\citep{Prabhu_et_al_2018}, implemented in \Algo{Parabel}, performs a simple top-down hierarchical clustering. 
Each label in this approach is represented by a profile vector being an average of the training vectors tagged by this label. 
Then the profile vectors are clustered using balanced $k$-means which divides the labels into two or more clusters with approximately the same size. 
This procedure is then repeated recursively until the clusters are smaller than a given value (for example, 100). 
The nodes of the resulting tree are then of different arities. 
The internal nodes up to the pre-leaf nodes have $k$ children, but the pre-leaf nodes are usually of higher arity. 
Thanks to this clustering, similar labels are close to each other in the tree. 
Moreover, the tree is balanced, so its depth is logarithmic in terms of the number of labels.
Variants of this method have been used in \Algo{extremeText}~\citep{Wydmuch_et_al_2018}, 
\Algo{Bonsai Trees}~\citep{Khandagale_et_al_2019} and \Algo{AttentionXML}~\citep{You_et_al_2019}.
The two latter algorithms promote shallow trees, that is, trees of a much higher arity.

\subsection{Ensemble of PLTs}
\label{sec:plt-ensemble}

Various ensemble techniques, such as bagging, are routinely applied with tree-based learners. 
A  simple ensemble approach can also be 
%easily 
implemented for \Algo{PLT}s. 
One can use several \Algo{PLT} instances of a different structure, which share the same feature space. 
This can be obtained by running the $k$-means-based top-down hierarchical clustering. 
Since $k$-means can be initialized randomly, each time a different tree can be produced.
Depending on the tree structure, the accuracy of a single \Algo{PLT} for specific labels may vary. 
Thus the aggregation of predictions of this diverse pool of \Algo{PLT}s should lead to improvement of the overall predictive performance. 
Such ensemble technique has been used in \Algo{Parabel}.

Classification of test examples with multiple trees can be still performed very efficiently. 
Each tree can be queried for its top $k$ predictions. 
Then, the labels are pooled and their average scores are computed.
Based on them, the final prediction is made. 
If a label does not have a probability estimate from a given tree (it is not included in the top-$k$ predictions),
then the estimate is either set to 0 or computed by traversing a single path in the tree corresponding to the label.
The good trade-off between the improvement of the results and the required computational resources is usually obtained for around 3 or 5 trees. 

\subsection{Node probabilities via multi-class classification}

So far we assumed that each node $\node \in \nodes$ is associated with a binary probabilistic classifier
that seeks for estimating $\eta_v(\bx)$.
As already mentioned in Section~\ref{subsec:plt-prediction}, 
this may require the additional normalization step (\ref{eqn:plt-normalization}),
as all models of siblings nodes are trained independently.
To avoid this problem, one can train a joint multi-class classifier over the sibling nodes. 
Let $\node \in \nodes$ be a parent of the sibling nodes $\childs{v}$.
Then, the class labels of the multi-class problem correspond to binary codes of vector $\vec{c}$ 
whose elements correspond to $z_{v'}$, $v' \in \childs{v}$. 
The classifier estimates $\prob(\vec{c} \given \bx, z_v = 1)$, for all $\vec{c} \in \{0,1\}^{|\childs{v}|}$. 
Probability $\eta_{v'}(\bx)$, for $v' \in \childs{v}$, 
is obtained by proper marginalization of the multi-class distribution over vectors~$\vec{c}$:
$$
\eta_{v'}(\bx) = \sum_{z_{v'} = 1} \prob(\vec{c} \given \bx, z_v = 1)\,.
$$
This approach has been investigated in \Algo{Parabel}~\citep{Prabhu_et_al_2018}. 
% Obviously, 
It can be applied to trees of small arity only
as the number of class labels grows exponentially with the number of sibling nodes.

%% file: 08-experiments.tex
\section{Empirical validation of \Algo{PLT}s}
\label{sec:experiments}

In this section, we show results of a wide empirical study 
we performed to comprehensively evaluate the described algorithms and theoretical findings.
We mainly report the results of the predictive performance in terms of \precatk{},
as this is the most used metric in XMLC experiments. 
Moreover, as shown in Section~\ref{sec:analysis}, \Algo{PLT}s are well-suited for this metric.
We also present training and test times, as well as memory consumption.
Whenever it was necessary, we repeated an experiment 5 times to eliminate the impact of the randomness of algorithms.
In such cases, we report the mean performance along with standard errors. 
All computations were conducted on an Intel Xeon E5-2697 v3 2.60GHz (14 cores) machine with 128GB RAM.
In most experiments, we use TF-IDF versions of the real-word benchmark data sets 
from the Extreme Classification Repository~\citep{xml-repo},%
\footnote{\url{http://manikvarma.org/downloads/XC/XMLRepository.html}}
for which we use the original train and test splits.
Table~\ref{tab:datasets} gives basic statistics of the data sets. 
\input{tables/table-datasets.tex}

%In the experiments, 
In the first part of the study, we analyze different design choices for \Algo{PLT}s.
To this end, we mainly use \Algo{napkinXC}, because of its modular architecture.  
However, whenever a tested configuration agrees with another \Algo{PLT} implementation, 
we use this one in the experiment. 
In the next part, we evaluate \Algo{PLT}s on Hamming loss and micro-F measure, 
to verify our theoretical results concerning generalized performance metrics.
Later, we empirically confirm the suboptimality of hierarchical softmax with pick-one-label heuristic.
The next experiment studies the performance of the fully online variant of \Algo{PLT}s, 
in which both node classifiers and tree structure are built incrementally on a sequence of training examples. 
Finally, we compare \Algo{PLT}s to relevant state-of-the-art algorithms.%
\footnote{The experiments with \Algo{napkinXC} are reproducible
by running scripts available from \url{https://github.com/mwydmuch/napkinXC/experiments}}

\subsection{\Algo{PLT}s with different design choices}
\label{subsec:pltvariants}

We analyze different design choices for \Algo{PLT}s. 
To this end, we use \Algo{napkinXC}, as thanks to its modular design,
we can easily experiment with different settings. 
However, whenever a given configuration agrees with an existing \Algo{PLT} implementation, 
we use this one in the experiment. 
This is the case of \Algo{Parabel} and \Algo{extremeText}. 
The former uses a dual coordinate descent method from \Algo{LIBLINEAR} with squared hinge loss to train node classifiers. 
It uses weight pruning at threshold 0.1, that is, it sets model weights less than 0.1 to zero.
The prediction algorithm is based on beam search. 
\Algo{extremeText} is built over \Algo{fastText}~\citep{Grave_et_al_2017}. 
It uses dense representation, shared by all nodes, 
which is a result of a 1-layer network implementing the \Algo{CBOW} architecture~\citep{Mikolov_et_al_2013}. 
This representation is trained along with the node models using
stochastic gradient descent with $L_2$ regularization and logistic loss.
Both implementations use hierarchical $k$-means clustering.
\Algo{Parabel} uses $k=2$, while \Algo{extremeText} allows for different values of $k$. 
Both use pre-leaves of high degree equal to 100.
In the experiment, we do not use \Algo{AttentionXML},
as it uses a complex deep architecture requiring powerful GPUs 
and runs over raw textual data. 
By comparing the results from the original paper~\citep{You_et_al_2019},
we admit that it achieves the best results among \Algo{PLT}-based approaches.
Nevertheless, in this study, we focus on efficient CPU implementations and 
the TF-IDF versions of the benchmark data sets.

We start with a comparison of batch and incremental learning of  node classifiers.
For both, we use logistic and squared hinge loss. 
Next, we verify two different methods of prediction. 
The first one is based on uniform-cost search, 
while the second on the beam search.
We then compare training and prediction with sparse and dense representation. 
In the next experiment, we analyze different tree-building strategies.
Finally, we check the impact of ensembling. 

\subsubsection{Batch and incremental learning}
\label{sec:optim-loss}

\afterpage{%
\input{tables/table-optim-loss.tex}
\clearpage
}

For batch learning, we use \Algo{LIBLINEAR}, the dual coordinate descent method, 
with either logistic loss or squared hinge loss. 
We use $L_2$ regularization for both and tune its $C$ parameter for each data set.
For incremental learning, we use \Algo{AdaGrad}~\citep{adagrad} with 3 epochs 
and tune the base learning rate $\epsilon$ for each data set. 
As above, we use either logistic loss or squared hinge loss.
In all algorithms, we prune the weights at 0.1 to obtain smaller models
and use uniform-cost search to obtain top-$k$ predictions.
Let us point out that 
the configuration based on \Algo{LIBLINEAR} with squared hinge loss is similar to \Algo{Parabel}.
The difference is that \Algo{Parabel} uses beam search,  
thus we run the implementation from \Algo{napkinXC} here.

The results are given in Table~\ref{tab:optim-loss}. 
None of the configurations strictly dominates the others.
It seems, however, that \Algo{AdaGrad} with squared hinge loss usually performs the worst.
This agrees with the fact that 
stochastic gradient approaches perform usually better with logistic loss.
This configuration also leads to models with substantially longer testing times.
In turn, significantly larger models for some data sets are built by \Algo{AdaGrad} with logistic loss,
while the training time can be doubled by \Algo{LIBLINEAR} with the same loss.
It seems from this analysis that the batch training with squared hinge loss is the most reliable,
without outlying results.
Moreover, it performs the best on both \Algo{Wikipedia} data sets.
Nevertheless, incremental learning is a valid competitor that 
can be easily combined with training of dense representation 
or online tree structure building.

\subsubsection{Prediction methods}

\afterpage{%
\input{tables/table-parabel-vs-napkinxc.tex}

\input{pics/plots-batches-times.tex}
\clearpage
}

We compare two prediction algorithms, the uniform-cost search and the beam search. 
The former is well-suited for online predictions
and is implemented as a default method in \Algo{napkinXC}.
It exploits efficient \Algo{Robin Hood} hash maps to perform fast predictions. 
The latter method, used in \Algo{Parabel}, 
benefits from using larger batches of test examples.
In each node it decompresses its sparse model to a dense form before evaluation of test examples
(see discussion on both algorithms in Section~\ref{subsec:implementation-prediction}).
We use the default size of the beam equal to 10. 
Besides the prediction method, 
\Algo{napkinXC} is set in the experiment to have the same setting as \Algo{Parabel}.
As shown in Table~\ref{tab:parabel-vs-napkinxc}, 
both methods perform very similarly in terms of \precatk{}.
This means that the beam of size 10 is indeed sufficient to approximate well the exact solution.
The small deviations in the results between both prediction methods are likely caused by 
small differences in their implementations and randomness present in the experiment.
Moreover, selecting top $k$ labels with respect to the estimates does not necessarily lead to the best result.
In the remainder, we focus on computational costs. 
Figure~\ref{fig:batches-times} shows the average prediction time per a single test example,
${T}/{N_{\textrm{test}}}$, 
as a function of the batch size $N_{\textrm{test}}$.
For each batch, we create 50 samples of observations by selecting them uniformly from the test set.
We measure the average prediction time over 5 different \Algo{napkinXC} and \Algo{Parabel} models, 
giving 250 measurements in total per each batch size.
The prediction time of the uniform-cost search, processing example by example, 
is independent of the batch size 
and it is lower than 10ms for most of the data sets. 
Beam search, working on batches of examples, 
is more than 100 times slower than the uniform-cost search for small batches.
To reach the prediction time of uniform-cost search it requires often batch sizes greater than 1000.

\subsubsection{Sparse and dense representation}

In the next experiment, we test the performance of \Algo{PLT}s with sparse and dense representation.
To this end, we use \Algo{napkinXC} and  \Algo{extremeText}, respectively. 
Besides representation, we use a similar setting for both algorithms.
Trees are built with hierarchical $k$-means clustering. 
Node classifiers are trained incrementally by minimizing logistic loss, however,
\Algo{napkinXC} uses \Algo{AdaGrad}, 
while \Algo{extremeText} stochastic gradient descent with $L_2$ regularization.
Prediction in both is based on uniform-cost search.

\input{tables/table-representations.tex}
Table~\ref{tab:representations} shows the results. 
\Algo{napkinXC} with sparse representation achieves higher precision$@k$ than \Algo{extremeText} 
on almost all data sets.  
This agrees with a common observation that 
learning a powerful dense representation for XMLC problems is difficult.
Nevertheless, the results of \Algo{extremeText} are approaching those of \Algo{napkinXC}, 
despite its very simple architecture and gradient updates.
Because of the additional layer, \Algo{extremeText} needs more time for training,  
but the dense representation allows for faster predictions 
and smaller models.
Let us comment, however, on the last observation.
The size of a model in \Algo{extremeText} 
%depends directly on
is determined by 
the dimension of dense representation, 
the number of labels, features, and tree nodes. 
In \Algo{napkinXC} with sparse representation,
the model size is influenced by the distribution of features among the labels, the tree structure,
% as similar labels usually share the same features. 
the chosen loss function and regularization.
Moreover, aggressive weight pruning can be applied
without a significant drop in predictive performance,
as we show in  Appendix~\ref{app:pruning}.

\subsubsection{Tree structure}
\label{sec:exp-tree-structure}

The choice of the tree structure is crucial 
as it affects all aspects of the performance: 
the accuracy of predictions, execution times, and  model sizes. 
In this experiment, we investigate different tree building strategies 
described in Section~\ref{subsec:tree-structure}.  
We first compare two types of balanced binary trees 
in which labels are split either randomly 
or using $k$-means clustering in a top-down procedure.
In trees of both types, pre-leaf nodes are set to have a high degree equal to 100,
that is, the splitting procedure stops when a cluster contains less than 100 labels
which are then transformed into leaves. 
% As a result of both algorithms, 
% we obtain trees of the same depth, 
Both algorithms create trees of the same depth, 
but with a different arrangement of labels. 
The results 
%for \precatk{}
are given in Table~\ref{tab:random-vs-kmeans-pred-comp-perf}.
In all cases, the $k$-means tree outperforms the random one in terms of \precat{k}.
On some data sets this difference is not substantial, 
but in several cases, $k$-means clustering leads to a huge boost of almost 20 percent.
Also training time and model size benefit from the clustering.
% , as reported in Table~\ref{tab:random-vs-kmeans-comp-perf}. 
The power of $k$-means trees can be explained by the fact 
that co-occurring and similar labels are grouped.
Thanks to this a training example is used in fewer nodes on average, the training tasks in tree nodes become simpler, and less features are necessary to solve them.

\input{tables/table-trees.tex}
\input{tables/table-trees-arity-maxleaves-logloss}
% \clearpage
% }

\ifjmlr
\vspace{-7mm}
\fi

In the next two experiments, 
we evaluate the impact of tree depth on predictive and computational performance of \Algo{PLT}s.
In the first experiment, we increase the degree of tree nodes to 16 and 64, 
but keep the degree of pre-leaves equal to 100.
% This way the tree depth is reduced 2 and 4 times.
Such an approach is similar to the one used in \Algo{Bonsai Tree}~\citep{Khandagale_et_al_2019}.
The results for $k$-means trees and logistic loss are given in~Table~\ref{tab:k-means-tree-arity-logistic}. 
In the second experiment, we use binary trees, 
but change the degree of pre-leaves from 100 to 25 and 400. 
% This leads to an increase of tree depth by two in the first case 
% and in a decrease by two in the second case. 
The results are given in~Table~\ref{tab:k-means-tree-maxleaves-logistic}.%
\footnote{For completeness, 
we present the results for the squared hinge loss in Appendix~\ref{app:treehinge}.}
In both experiments, \precatk{} slightly increases with a decrease in the tree depth.
This behavior is expected as suggested by the theoretical results from Section~\ref{sec:analysis}. 
The shorter paths should result in tighter upper bounds. 
On the other hand, a shallower tree leads to longer training times, 
% due to the higher number of nodes in which an example is used for training. 
as an example is used for training in more nodes.
Notice that the \Algo{1-vs-All} approach can be treated as an extremely shallow tree, with each training example used in all nodes.
Similarly to the training time, we should expect prediction time to increase with the decreasing tree depth. This is, however, clearly visible only in the second experiment,  
in which we change the degree of pre-leaf nodes.
Interestingly, the size of the resulting models does not significantly change 
over the different tree structures.
The larger number of nodes is likely compensated by sparser models.

\subsubsection{Ensemble of \Algo{PLT}s}

In the last experiment focused on design choices, 
we analyze the predictive performance of small ensembles of \Algo{PLT}s. 
In Figure~\ref{fig:ensembles}, we compare ensembles of size 3 and 5 to a single tree.
Each tree is trained using $2$-means clustering with the degree of pre-leaf nodes set to 100. 
The tree nodes are trained using \Algo{LIBLINEAR} with either logistic loss or squared hinge loss.
The results show that the gain of using 3 trees instead of one is much greater than 
the gain of using 5 trees instead of 3. 
This makes the ensemble of size 3 a reasonable trade-off between predictive and computational cost,
as the size of the models and computational cost grow linearly with the number of trees. 
It seems also that ensembles with squared hinge loss gain slightly more 
than the ensembles trained with logistic loss.

\afterpage{%
\clearpage

\input{pics/plots-ensembles.tex}
\clearpage
}

\subsection{Generalized performance metrics}

The previous experiments concern \Algo{PLT}s with top-$k$ predictions suited for \precatk{}.
In this section, we focus on threshold-based predictions and 
generalized performance metrics, discussed in Section~\ref{subsec:analyis-generalized}.
We constrain our analysis to Hamming loss and micro $F_1$-measure. 
A similar experiment for the macro $F_1$-measure has been conducted in~\citep{Jasinska_et_al_2016}. 
For each metric, we report the results of two approaches. 
The first one uses a fixed threshold of 0.5, 
which is theoretically optimal for Hamming loss.
Remark, however, that for estimated probabilities the optimal threshold can be different.
The second one optimizes the micro $F_1$-measure
by tuning one global threshold, as suggested by theory.
To this end, it uses $70\%$ of original training data to train a \Algo{PLT} model,
and the rest to tune the threshold by running 
online F-measure optimization (\Algo{OFO})~\citep{Busa-Fekete_et_al_2015}.
The same method has been used in~\citep{Jasinska_et_al_2016} to optimize the macro $F$-measure.
We repeat computations 5 times and report the average results along with standard errors.
Table~\ref{tab:plt-fmeasure} presents the results.
Notice that for Hamming loss the lower the value the better is the performance,
while for the micro $F$-measures it is the opposite, the higher the value the better.
As expected from the theoretical analysis, 
a procedure suited for a given metric leads to significantly better results,
with only one exception.

\input{tables/table-fmeasure.tex}

\subsection{Comparison to hierarchical softmax}

In this experiment, we verify our theoretical findings from~Section~\ref{subsec:hsm}.
As we have shown, 
hierarchical softmax with the pick-one-label heuristic (\Algo{HSM-POL}) leads 
to a suboptimal solution with respect to \precatk{}. 
To demonstrate this empirically, we run two experiments. 
In the first one, 
we compare the performance of \Algo{PLT}s and \Algo{HSM-POL} on synthetic data.
In the second experiment, we evaluate both algorithms on benchmark data sets. 
To conduct the experiments, 
we implemented \Algo{HSM-POL} in \Algo{napkinXC}.
We made it as similar as possible to the implementation of \Algo{PLT}s,
with the only differences coming from the model definition.
For both algorithms, we use \Algo{LIBLINEAR} with L2-regularized logistic loss to train node classifiers.
In the experiment on benchmark data, we additionally use weight pruning to reduce model sizes. 
This is not necessary for synthetic data. %to train node classifiers \todo{(used only for the benchmark data sets)}. 
To simulate the pick-one-label heuristic in batch learning,
we transform a \multilabel{} example with $||\by||_1$ positive labels 
to $||\by||_1$ weighted \multiclass{} examples, 
each with a different positive label assigned and weight equal $\frac{1}{||\by||_1}$.
%Using such implementation instead of an incrementally trained one allows us to compare those two models more fair, without differences caused by different number of classifier updates per epoch in case of the online implementation.

\input{tables/table-hsm-vs-plt.tex}
The results for \precat{1} are given in Table~\ref{tab:synthetic}. 
We use three types of synthetic data generated from different distributions: 
\multilabel{} with conditionally independent labels, 
\multilabel{} with conditionally dependent labels, 
and \multiclass.  
The detailed description of the data generation process is given in Appendix~\ref{app:synthetic}. 
The presented values are averages over 50 runs along with standard errors. 
Notice, however, that the data generation processes may lead to very diverse problems,
with a different level of noise. 
Therefore, standard errors indicate rather the diversity of the generated problems.
To overcome this issue, we report the number of wins, ties, and losses,
as well as $p$-values of the very conservative sign test. 
On data with conditionally dependent labels, 
\Algo{PLT}s clearly outperform \Algo{HSM-POL} as indicated by the p-value.
This agrees with our theoretical results.
On data with conditionally independent labels, 
both algorithms perform similarly without statistically significant differences.
This also agrees with the theory, 
as we have proven that under label independence \Algo{HSM-POL} performs optimally for \precatk{}.
The results on \multiclass{} data completely match, 
as for this distribution the \Algo{PLT} model boils down to \Algo{HSM}.

\input{tables/table-hsm-vs-plt-benchmark.tex}
Table~\ref{tab:hsm-benchmark} gives the results on benchmark data sets.
The difference in performance between \Algo{PLT}s and \Algo{HSM-POL} is clearly visible.
It is even more substantial than in the previous experiment.
Besides precision$@1$, the table contains also results for recall$@k$ (with $k = 1,5$),
$$
r@k(\by, \bh_{@k}(\bx)) = \frac{1}{||\by||_1} \sum_{j \in \hat \calL_{\bx}} \assert{y_j = 1} \,,
$$
where  $\hat \calL_{\bx} = \{j \in \calL: h_j(\bx) = 1 \}$ 
is a set of $k$ labels predicted by classifier $\bh_{@k}$ for $\bx$. 
The pick-one-label heuristic should lead to optimal results for this metric, 
as shown by~\citet{Menon_et_al_2019}. 
Nevertheless, \Algo{PLT}s obtain better results also for this metric, 
but the difference is much smaller. 
This suggests that indeed \Algo{HSM-POL} can be well-suited for recall$@k$, 
but the pick-one-label heuristic may lead to corrupted learning problems in tree nodes. 
As discussed in~\citep{Menon_et_al_2019}, there exist other strategies for optimizing recall$@k$,
which may perform better than \Algo{PLT}s.

\subsection{Online \Algo{PLT}s}

We empirically verify online probabilistic label trees 
in which both node classifiers and tree structure are built incrementally. 
We implemented the \Algo{OPLT} algorithm, introduced in Section~\ref{sec:oplt}, in \Algo{napkinXC}.
The tree is constructed using the simple complete tree policy from Algorithm~\ref{alg:oplt-apply_policy}.
To train node classifiers, we use \Algo{AdaGrad} with logistic loss.
The incremental learning in the online setting requires quick access to model weights, 
preferably storing all of them at once in memory. 
Unfortunately, maintaining an array for all possible weights in a dense format
would require, for many data sets, thousands of GB of memory.
As described in Section~\ref{sec:sparse-features}, 
either hash maps, such as \Algo{Robin Hood}, or feature hashing should be applied to overcome this problem. 
In the experiment, we compare both approaches.

Feature hashing allows us to directly control the amount of memory used, 
but it may result with many unresolved collisions 
if the allocated space is too small. 
We consider setups with 64GB, 128GB, and 256GB of RAM.
The \Algo{Robin Hood} hash map avoids collisions, 
but does not allow for restraining memory consumption.
The number of hashed features
which can be allocated without collisions
is given in~Table~\ref{tab:oplt-random-features}.
We report this number for each data set and memory setup.
It takes into account memory needed for $2m-1$ nodes,
each containing model weights and cumulative gradients required by \Algo{AdaGrad},
and auxiliary classifiers which number can be limited to $m$ for the chosen tree building policy.
Additionally, we present in the same table 
the amount of memory required by \Algo{OPLT} with \Algo{Robin Hood} 
and \Algo{OPLT} with dense vectors.

\input{tables/table-online-vs-batch-features.tex}

To simulate the online/streaming setting, 
the \Algo{OPLT} algorithms run three times over training examples, each time permuted randomly.
We evaluate the performance on the original test sets in terms of \precat{1}.
The results are given in Table~\ref{tab:oplt-random}.
For reference, we also present the results of a batch \Algo{PLT} trained with logistic loss on a complete binary tree.
\Algo{OPLT} with \Algo{Robin Hood} performs similarly to \Algo{PLT}. 
This agrees with the results from Section~\ref{sec:optim-loss}
showing that incremental learning under logistic loss is competitive to its batch counterpart. 
Interestingly, \Algo{Robin Hood} allows us to train \Algo{OPLT} in 256GB of RAM for all data sets, 
with the only exception of Amazon-3M for which 280GB is required.
The performance of \Algo{OPLT} with feature hashing drops significantly for large data sets, 
even when using the same amount of memory as \Algo{OPLT} with \Algo{Robin Hood}. 
One may observe that the smaller is the hashing space compared to the original feature space, 
the larger is the drop.

\input{tables/table-online-vs-batch.tex}

A better tree building policy may improve the predictive performance of \Algo{OPLT}.
By comparing the results presented here to the ones of $k$-means trees, 
we observe a large gap. 
The online tree building algorithms are not able to fully eliminate it, 
but we believe that the regret can be much smaller.
Also, memory usage could be improved by better utilization of auxiliary classifiers. 

\subsection{PLT vs. state-of-the-art}
\label{sec:plt-vs-sota}

In the final part of the empirical study, 
we compare \Algo{PLT}s with state-of-the-art algorithms.
In the comparison, we use two decision tree methods.
Their main difference to label trees is that they split the feature space, 
not the set of labels. 
\Algo{FastXML}, introduced in~\citep{Prabhu_Varma_2014},
uses sparse linear classifiers in internal tree nodes, 
also trained using~\Algo{LIBLINEAR}. 
Each linear classifier decides between two classes, the  left or the right child.
These two classes are initiated by a random assignment of training examples to the children nodes.
In the next steps, the assignment is reshaped by optimizing the normalized discounted cumulative gain (nDCG) over both children. 
Once the assignment stabilizes, a sparse linear classifier is trained using logistic loss.
To improve the overall accuracy \Algo{FastXML} uses an ensemble of trees.
\Algo{PfastreXML}~\citep{Jain_et_al_2016} is a modification of \Algo{FastXML}
that optimizes propensity scored nDCG at each tree node and re-ranks the predicted labels.
Besides decision trees, we also use two \Algo{1-vs-All} algorithms
which are known to be the best \emph{no-deep} (or CPU-based) XMLC methods.
\Algo{DiSMEC} trains a single classifier per label under $L_2$ regularized squared hinge loss, also using \Algo{LIBLINEAR}.
It prunes weights of final models at threshold equal 0.01 to reduce the memory needed to store a \Algo{1-vs-All} classifier. 
It uses distributed training over multiple cores and processors to speed up computations. 
\Algo{PPD-Sparse}~\citep{Yen_et_al_2017}, in turn, 
parallelizes \Algo{PD-Sparse}~\citep{Yen_et_al_2016} 
which optimizes a max-margin loss by exploiting the primal-dual sparsity, 
resulting from the use of the max-margin loss under $L_1$ regularization, 
given that for each training example the set of highly scored incorrect labels is small.
We exclude from the comparison all XMLC algorithms based on complex deep networks,
requiring the use of GPUs and raw versions of text data. 
We remark, however, that some of such methods, for example \Algo{XML-CNN}~\citep{Liu_et_al_2017},
perform worse than the best methods used in this study.
As instances of \Algo{PLT}s, 
we use \Algo{Parabel} and \Algo{napkinXC}. 
For the former, we use an ensemble of three trees trained with squared hinge loss.
This is the first label tree algorithm being competitive to state-of-the-art,
as reported in~\cite{Prabhu_et_al_2018}.
For \Algo{napkinXC}, we use a configuration, suggested by the results of the previous experiments,
which uses also an ensemble of three trees, but with arity of 16, 
providing a significant predictive performance boost over binary trees, 
at the same keeping training and prediction times reasonably low.
For training node classifiers, we use \Algo{LIBLINEAR} with logistic loss for 3 datasets (\amazoncatsmall, \wikiten{} and \deliciouslarge) and squared hinge loss for the rest of the datasets.
%smaller data sets and squared hinge loss for the larger data sets.
%\todo{\textbf{Czy powyższe jest poprawne?}}
%, trained with logistic or squared hinge loss.}

The results are given in Table~\ref{tab:plt-vs-sota}.
We report \precat{k}, training and prediction times, and model sizes.
For each not deterministic algorithm, we repeat the experiment 5 times and report means with standard errors. 
We use original implementations of all competitors. 
The hyperparameters used to tune the final models are given in Appendix~\ref{app:hiperparams}. 
For \Algo{DiSMEC} and \Algo{PPDSparse} we report the best results found in the literature,
namely from \citep{Babbar_Scholkopf_2017, Yen_et_al_2017, Prabhu_et_al_2018, xml-repo}.
We use the provided implementations to approximate training and prediction times on our hardware.
%, since we could not reproduce PDDSparse results. 
% Wydaje się, że DiSMEC i PDDSparse są deterministyczne, aczkolwiek ludzie uzyskują dla nich różne wyniki, mi totalnie się nie udało powtórzyć wyników dla PDDSparse.
%
From the results, we see that \Algo{PLT}s are indeed competitive to the \Algo{1-vs-All} approaches,
achieving the best \precat{1} on 5 from 8 data sets 
and is only slightly worse on the rest of the data sets. 
They 
%clearly 
outperform the decision tree-based methods.
\Algo{PLT}s are almost always the fastest in training and prediction 
and achieve the smallest model sizes.
They can be even a thousand times faster in training and prediction than \Algo{1-vs-All}. 
The variant of \Algo{napkinXC} used in this experiment outperforms \Algo{Parabel} in terms of \precat{k}
by sacrificing the computational performance of training and prediction. 
However, it can predict in an online setting at the same time often consuming less memory.
%\todo{Czy powyższe jest poprawne?}
%
% \afterpage{%
    \input{tables/table-best-plt-vs-sota.tex}
    % \clearpage
% }

%% file: tables/table-datasets.tex
\begin{table}[h]
    \centering
    \footnotesize
    \begin{tabular}{l|rrrrr}
        %Dataset         & \multicolumn{1}{c}{\#features} & \multicolumn{1}{c}{\#labels} & \multicolumn{1}{c}{$N_{\textrm{train}}$} & \multicolumn{1}{c}{$N_{\textrm{test}}$} & \\
        \toprule
        Dataset & \multicolumn{1}{c}{$\dim{\calX}$} & \multicolumn{1}{c}{$\dim{\calY}$ ($m$)} &  \multicolumn{1}{c}{$N_{\textrm{train}}$} & \multicolumn{1}{c}{$N_{\textrm{test}}$} & \multicolumn{1}{c}{avg. $|\labels_{\bx}|$} \\
        \midrule
        \eurlex         & 5000      & 3993      & 15539      & 3809      & 5.31  \\
        \amazoncatsmall & 203882    & 13330     & 1186239    & 306782    & 5.04  \\
        % \amazoncatlarge & 597540    & 14588     & 4398050    & 1099725  & 3.53  \\
        \wikiten        & 101938    & 30938     & 14146      & 6616      & 18.64 \\
        \deliciouslarge & 782585    & 205443    & 196606     & 100095    & 75.54 \\
        \wikilshtc      & 1617899   & 325056    & 1778351    & 587084    & 3.19  \\
        \wikipedia      & 2381304   & 501070    & 1813391    & 783743    & 4.77  \\
        \amazon         & 135909    & 670091    & 490449     & 153025    & 5.45  \\
        \amazonlarge    & 337067    & 2812281   & 1717899    & 742507    & 36.17 \\
        \bottomrule
    \end{tabular}
    \caption{The number of unique features, labels, examples in train and test splits, and the average number of true labels per example in the benchmark data sets.}
    \label{tab:datasets}
\end{table}

%% file: tables/table-optim-loss.tex
\begin{table}[H]
\footnotesize{
\resizebox{\textwidth}{!}{
\tabcolsep=5pt
\begin{tabular}{l|r@{}lr@{}lr@{}l|r@{}lr@{}lr@{}l}
\specialrule{0.94pt}{0.4ex}{0.65ex}
Optimizer
& \multicolumn{2}{c}{$p@1$ [\%]}
& \multicolumn{2}{c}{$p@3$ [\%]}
& \multicolumn{2}{c|}{$p@5$ [\%]}
& \multicolumn{2}{c}{$T_{\textrm{train}}$ [h]} 
& \multicolumn{2}{c}{$T/N_{\textrm{test}}$ [ms]} 
& \multicolumn{2}{c}{$M_{\textrm{size}}$ [GB]}
\\

\specialrule{0.94pt}{0.4ex}{0.65ex}
& \multicolumn{12}{c}{\eurlex}
\\
\midrule
\Algo{nXC}-B,log & \boldmath$80.51$ & \boldmath$\pm 0.16$ & $65.65$ & $\pm 0.40$ & $53.33$ & $\pm 0.68$ & $0.02$ & $\pm 0.00$ & $0.39$ & $\pm 0.03$ & $0.02$ & $\pm 0.00$ \\
\Algo{nXC}-B,s.h. & $80.17$ & $\pm 0.27$ & $65.33$ & $\pm 0.53$ & $53.01$ & $\pm 0.87$ & \boldmath$0.01$ & \boldmath$\pm 0.00$ & \boldmath$0.24$ & \boldmath$\pm 0.02$ & \boldmath$0.00$ & \boldmath$\pm 0.00$ \\
\Algo{nXC}-I,log & $80.43$ & $\pm 0.09$ & \boldmath$66.08$ & \boldmath$\pm 0.26$ & \boldmath$53.87$ & \boldmath$\pm 0.58$ & $0.01$ & $\pm 0.00$ & $0.25$ & $\pm 0.02$ & $0.05$ & $\pm 0.00$ \\
\Algo{nXC}-I,s.h. & $78.72$ & $\pm 0.18$ & $61.54$ & $\pm 0.34$ & $48.09$ & $\pm 0.51$ & $0.01$ & $\pm 0.00$ & $0.33$ & $\pm 0.03$ & $0.03$ & $\pm 0.00$ \\

\specialrule{0.94pt}{0.4ex}{0.65ex}
& \multicolumn{12}{c}{\amazoncatsmall}
\\
\midrule
\Algo{nXC}-B,log & $93.04$ & $\pm 0.02$ & $78.44$ & $\pm 0.02$ & $63.70$ & $\pm 0.02$ & $0.72$ & $\pm 0.02$ & $0.32$ & $\pm 0.03$ & $0.35$ & $\pm 0.00$ \\
\Algo{nXC}-B,s.h. & $92.40$ & $\pm 0.04$ & $78.49$ & $\pm 0.02$ & $63.88$ & $\pm 0.02$ & $0.29$ & $\pm 0.00$ & \boldmath$0.19$ & \boldmath$\pm 0.00$ & \boldmath$0.19$ & \boldmath$\pm 0.00$ \\
\Algo{nXC}-I,log & \boldmath$93.23$ & \boldmath$\pm 0.02$ & \boldmath$78.76$ & \boldmath$\pm 0.03$ & \boldmath$64.05$ & \boldmath$\pm 0.02$ & $0.17$ & $\pm 0.00$ & $0.32$ & $\pm 0.02$ & $0.72$ & $\pm 0.00$ \\
\Algo{nXC}-I,s.h. & $92.62$ & $\pm 0.05$ & $76.39$ & $\pm 0.06$ & $60.67$ & $\pm 0.06$ & \boldmath$0.16$ & \boldmath$\pm 0.00$ & $0.38$ & $\pm 0.01$ & $0.55$ & $\pm 0.00$ \\

\specialrule{0.94pt}{0.4ex}{0.65ex}
& \multicolumn{12}{c}{\wikiten}
\\
\midrule
\Algo{nXC}-B,log & $85.36$ & $\pm 0.09$ & $73.90$ & $\pm 0.07$ & $63.84$ & $\pm 0.07$ & $0.21$ & $\pm 0.00$ & $5.35$ & $\pm 0.32$ & $0.58$ & $\pm 0.00$ \\
\Algo{nXC}-B,s.h. & $84.17$ & $\pm 0.10$ & $72.43$ & $\pm 0.10$ & $63.12$ & $\pm 0.04$ & $0.11$ & $\pm 0.00$ & \boldmath$2.87$ & \boldmath$\pm 0.08$ & \boldmath$0.06$ & \boldmath$\pm 0.00$ \\
\Algo{nXC}-I,log & $84.92$ & $\pm 0.10$ & \boldmath$74.52$ & \boldmath$\pm 0.09$ & \boldmath$65.29$ & \boldmath$\pm 0.04$ & $0.11$ & $\pm 0.00$ & $5.24$ & $\pm 0.42$ & $0.91$ & $\pm 0.00$ \\
\Algo{nXC}-I,s.h. & \boldmath$85.64$ & \boldmath$\pm 0.10$ & $70.37$ & $\pm 0.09$ & $59.43$ & $\pm 0.13$ & \boldmath$0.10$ & \boldmath$\pm 0.00$ & $7.19$ & $\pm 0.06$ & $0.35$ & $\pm 0.00$ \\

\specialrule{0.94pt}{0.4ex}{0.65ex}
& \multicolumn{12}{c}{\deliciouslarge}
\\
\midrule
\Algo{nXC}-B,log & \boldmath$49.55$ & \boldmath$\pm 0.05$ & \boldmath$43.08$ & \boldmath$\pm 0.03$ & \boldmath$39.90$ & \boldmath$\pm 0.02$ & \boldmath$2.58$ & \boldmath$\pm 0.15$ & \boldmath$9.89$ & \boldmath$\pm 0.89$ & \boldmath$0.95$ & \boldmath$\pm 0.00$ \\
\Algo{nXC}-B,s.h. & $46.30$ & $\pm 0.07$ & $39.76$ & $\pm 0.08$ & $36.54$ & $\pm 0.07$ & $5.51$ & $\pm 0.29$ & $10.07$ & $\pm 0.23$ & $1.82$ & $\pm 0.00$ \\
\Algo{nXC}-I,log & $45.27$ & $\pm 0.06$ & $38.26$ & $\pm 0.03$ & $34.88$ & $\pm 0.02$ & $2.97$ & $\pm 0.03$ & $11.51$ & $\pm 0.57$ & $15.05$ & $\pm 0.00$ \\
\Algo{nXC}-I,s.h. & $45.29$ & $\pm 0.34$ & $38.15$ & $\pm 0.50$ & $34.44$ & $\pm 0.58$ & $3.13$ & $\pm 0.10$ & $27.11$ & $\pm 1.55$ & $9.59$ & $\pm 0.00$ \\

\specialrule{0.94pt}{0.4ex}{0.65ex}
& \multicolumn{12}{c}{\wikilshtc}
\\
\midrule
\Algo{nXC}-B,log & $61.96$ & $\pm 0.03$ & $40.77$ & $\pm 0.02$ & $30.19$ & $\pm 0.02$ & $2.95$ & $\pm 0.15$ & $1.77$ & $\pm 0.11$ & $2.73$ & $\pm 0.00$ \\
\Algo{nXC}-B,s.h. & \boldmath$62.78$ & \boldmath$\pm 0.03$ & \boldmath$41.17$ & \boldmath$\pm 0.02$ & \boldmath$30.25$ & \boldmath$\pm 0.02$ & $1.60$ & $\pm 0.06$ & \boldmath$0.86$ & \boldmath$\pm 0.06$ & \boldmath$0.97$ & \boldmath$\pm 0.00$ \\
\Algo{nXC}-I,log & $60.99$ & $\pm 0.04$ & $39.85$ & $\pm 0.02$ & $29.50$ & $\pm 0.01$ & $1.52$ & $\pm 0.00$ & $2.44$ & $\pm 0.02$ & $4.93$ & $\pm 0.00$ \\
\Algo{nXC}-I,s.h. & $59.55$ & $\pm 0.05$ & $37.33$ & $\pm 0.04$ & $27.02$ & $\pm 0.03$ & \boldmath$1.41$ & \boldmath$\pm 0.05$ & $1.70$ & $\pm 0.17$ & $3.64$ & $\pm 0.00$ \\

\specialrule{0.94pt}{0.4ex}{0.65ex}
& \multicolumn{12}{c}{\wikipedia}
\\
\midrule
\Algo{nXC}-B,log & $66.20$ & $\pm 0.05$ & $47.14$ & $\pm 0.02$ & $36.83$ & $\pm 0.01$ & $16.10$ & $\pm 0.44$ & $6.67$ & $\pm 0.23$ & $8.89$ & $\pm 0.00$ \\
\Algo{nXC}-B,s.h. & \boldmath$66.77$ & \boldmath$\pm 0.08$ & \boldmath$47.63$ & \boldmath$\pm 0.04$ & \boldmath$36.94$ & \boldmath$\pm 0.02$ & $9.48$ & $\pm 0.33$ & \boldmath$2.86$ & \boldmath$\pm 0.07$ & \boldmath$1.78$ & \boldmath$\pm 0.00$ \\
\Algo{nXC}-I,log & $65.68$ & $\pm 0.15$ & $46.62$ & $\pm 0.09$ & $36.52$ & $\pm 0.06$ & \boldmath$8.11$ & \boldmath$\pm 0.18$ & $7.86$ & $\pm 0.19$ & $19.11$ & $\pm 0.00$ \\
\Algo{nXC}-I,s.h. & $65.05$ & $\pm 0.08$ & $44.35$ & $\pm 0.03$ & $33.74$ & $\pm 0.05$ & $8.22$ & $\pm 0.06$ & $6.66$ & $\pm 0.06$ & $24.71$ & $\pm 0.00$ \\

\specialrule{0.94pt}{0.4ex}{0.65ex}
& \multicolumn{12}{c}{\amazon}
\\
\midrule
\Algo{nXC}-B,log & $43.54$ & $\pm 0.01$ & $38.71$ & $\pm 0.02$ & $35.15$ & $\pm 0.03$ & $0.56$ & $\pm 0.00$ & $4.13$ & $\pm 0.28$ & $2.26$ & $\pm 0.00$ \\
\Algo{nXC}-B,s.h. & $43.31$ & $\pm 0.03$ & $38.19$ & $\pm 0.03$ & $34.31$ & $\pm 0.03$ & \boldmath$0.40$ & \boldmath$\pm 0.01$ & \boldmath$1.32$ & \boldmath$\pm 0.08$ & \boldmath$0.63$ & \boldmath$\pm 0.00$ \\
\Algo{nXC}-I,log & \boldmath$43.82$ & \boldmath$\pm 0.01$ & \boldmath$38.88$ & \boldmath$\pm 0.03$ & \boldmath$35.31$ & \boldmath$\pm 0.03$ & $0.42$ & $\pm 0.00$ & $5.93$ & $\pm 0.11$ & $6.22$ & $\pm 0.00$ \\
\Algo{nXC}-I,s.h. & $41.46$ & $\pm 0.02$ & $36.16$ & $\pm 0.04$ & $32.34$ & $\pm 0.03$ & $0.41$ & $\pm 0.00$ & $2.08$ & $\pm 0.17$ & $5.26$ & $\pm 0.00$ \\

\specialrule{0.94pt}{0.4ex}{0.65ex}
& \multicolumn{12}{c}{\amazonlarge}
\\
\midrule
\Algo{nXC}-B,log & $46.09$ & $\pm 0.02$ & $43.11$ & $\pm 0.01$ & $40.98$ & $\pm 0.01$ & $7.07$ & $\pm 0.56$ & $3.26$ & $\pm 0.08$ & $20.84$ & $\pm 0.00$ \\
\Algo{nXC}-B,s.h. & \boldmath$46.23$ & \boldmath$\pm 0.01$ & \boldmath$43.48$ & \boldmath$\pm 0.01$ & \boldmath$41.41$ & \boldmath$\pm 0.01$ & $5.44$ & $\pm 0.13$ & \boldmath$1.96$ & \boldmath$\pm 0.05$ & \boldmath$9.86$ & \boldmath$\pm 0.00$ \\
\Algo{nXC}-I,log & $43.61$ & $\pm 0.12$ & $40.44$ & $\pm 0.09$ & $38.32$ & $\pm 0.06$ & \boldmath$4.44$ & \boldmath$\pm 0.10$ & $4.42$ & $\pm 0.00$ & $52.16$ & $\pm 0.00$ \\
\Algo{nXC}-I,s.h. & $43.73$ & $\pm 0.07$ & $40.19$ & $\pm 0.06$ & $37.75$ & $\pm 0.05$ & $4.45$ & $\pm 0.04$ & $11.03$ & $\pm 0.02$ & $24.77$ & $\pm 0.01$ \\

\specialrule{0.94pt}{0.4ex}{0.65ex}
\end{tabular}}}

\caption{Comparison of \Algo{napkinXC} (\Algo{nXC}) with different modes of node classifiers training. We perform batch \Algo{LibLinear} (B) or incremental \Algo{AdaGrad} (I) training with logistic loss (log) or squared hinge loss (s.h.).}
\label{tab:optim-loss}
\end{table}

%% file: tables/table-parabel-vs-napkinxc.tex
\begin{table}[H]
\footnotesize{
\resizebox{\textwidth}{!}{
\tabcolsep=5pt
\begin{tabular}{l|r@{}lr@{}l|r@{}lr@{}l|r@{}lr@{}l}
\toprule
& \multicolumn{4}{c|}{$p@1$ [\%]}
& \multicolumn{4}{c|}{$p@3$ [\%]}
& \multicolumn{4}{c}{$p@5$ [\%]} \\
Prediction method & \multicolumn{2}{c}{u-c search} & \multicolumn{2}{c|}{beam search}
& \multicolumn{2}{c}{u-c search} & \multicolumn{2}{c|}{beam search}
& \multicolumn{2}{c}{u-c search} & \multicolumn{2}{c}{beam search} \\
\midrule
\eurlex & $80.17$ & $\pm 0.27$ & \boldmath$80.66$ & \boldmath$\pm 0.16$ & $65.33$ & $\pm 0.53$ & \boldmath$67.76$ & \boldmath$\pm 0.05$ & $53.01$ & $\pm 0.87$ & \boldmath$56.57$ & \boldmath$\pm 0.07$ \\
\amazoncatsmall & $92.40$ & $\pm 0.04$ & \boldmath$92.58$ & \boldmath$\pm 0.02$ & $78.49$ & $\pm 0.02$ & \boldmath$78.53$ & \boldmath$\pm 0.00$ & $63.88$ & $\pm 0.02$ & \boldmath$63.90$ & \boldmath$\pm 0.01$ \\
\wikiten & \boldmath$84.17$ & \boldmath$\pm 0.10$ & \boldmath$84.17$ & \boldmath$\pm 0.03$ & \boldmath$72.43$ & \boldmath$\pm 0.10$ & $72.12$ & $\pm 0.04$ & $63.12$ & $\pm 0.04$ & \boldmath$63.30$ & \boldmath$\pm 0.05$ \\
\deliciouslarge & $46.30$ & $\pm 0.07$ & \boldmath$46.44$ & \boldmath$\pm 0.07$ & \boldmath$39.76$ & \boldmath$\pm 0.08$ & $39.66$ & $\pm 0.06$ & \boldmath$36.54$ & \boldmath$\pm 0.07$ & $36.19$ & $\pm 0.05$ \\
\wikilshtc & \boldmath$62.78$ & \boldmath$\pm 0.03$ & \boldmath$62.78$ & \boldmath$\pm 0.02$ & $41.17$ & $\pm 0.02$ & \boldmath$41.22$ & \boldmath$\pm 0.02$ & $30.25$ & $\pm 0.02$ & \boldmath$30.27$ & \boldmath$\pm 0.01$ \\
\wikipedia & $66.77$ & $\pm 0.08$ & \boldmath$67.05$ & \boldmath$\pm 0.14$ & $47.63$ & $\pm 0.04$ & \boldmath$47.75$ & \boldmath$\pm 0.10$ & $36.94$ & $\pm 0.02$ & \boldmath$36.99$ & \boldmath$\pm 0.08$ \\
\amazon & \boldmath$43.31$ & \boldmath$\pm 0.03$ & $43.13$ & $\pm 0.02$ & \boldmath$38.19$ & \boldmath$\pm 0.03$ & $37.94$ & $\pm 0.03$ & \boldmath$34.31$ & \boldmath$\pm 0.03$ & $34.00$ & $\pm 0.00$ \\
\amazonlarge & \boldmath$46.23$ & \boldmath$\pm 0.01$ & $46.14$ & $\pm 0.01$ & \boldmath$43.48$ & \boldmath$\pm 0.01$ & $43.32$ & $\pm 0.01$ & \boldmath$41.41$ & \boldmath$\pm 0.01$ & $41.20$ & $\pm 0.01$ \\

\bottomrule
\end{tabular}}}

\caption{\Precatk{} of uniform-cost (u-c) search (\Algo{napkinXC}) and beam search (\Algo{Parabel}).}
\label{tab:parabel-vs-napkinxc}
\end{table}

%% file: pics/plots-batches-times.tex
\vspace{-0.5cm}

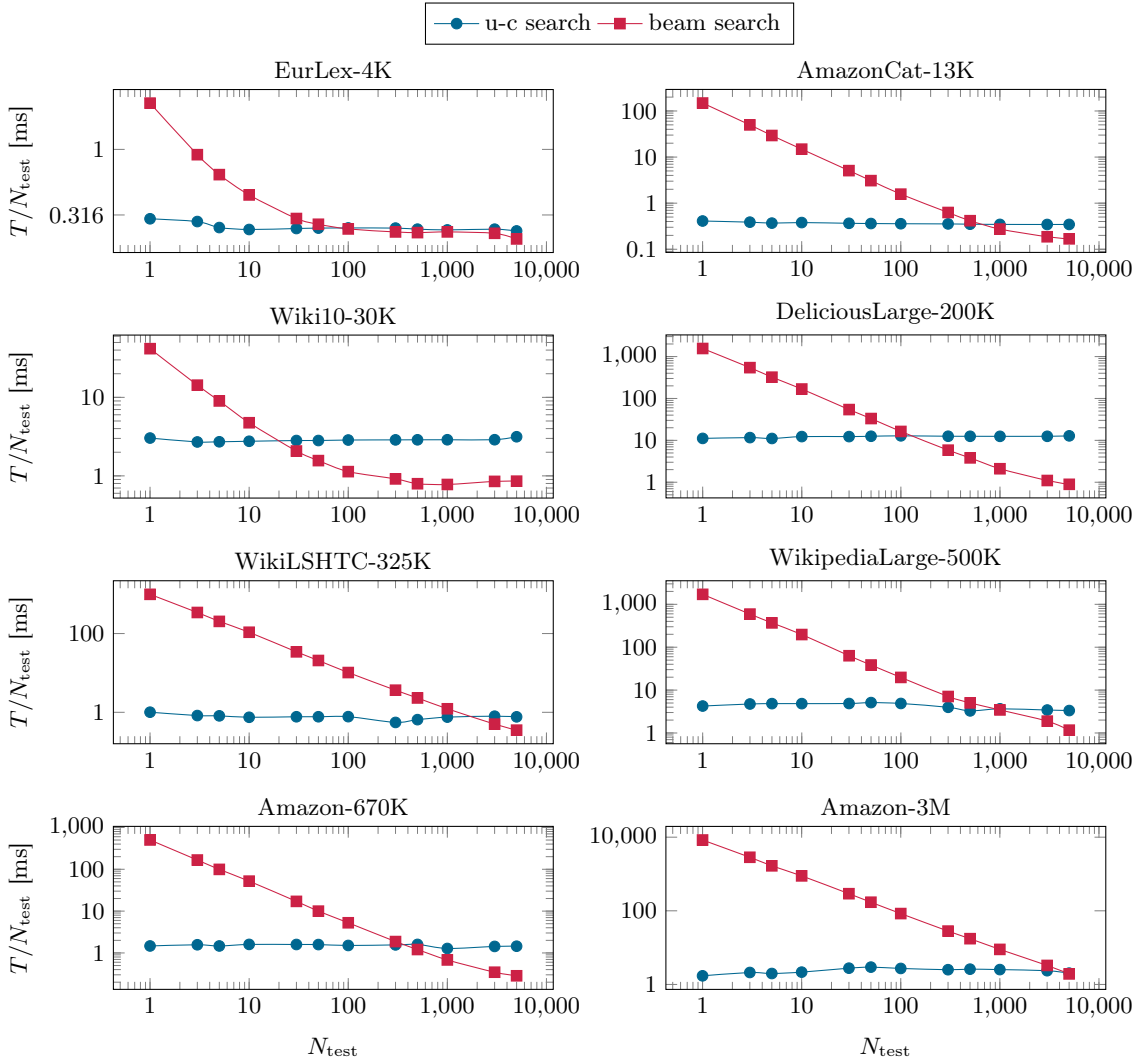
\begin{figure}[H]
%\usetikzlibrary{pgfplots.groupplots}
%\usepgfplotslibrary{fillbetween}
\tikzset{fontscale/.style = {font=\footnotesize}}
\centering
\footnotesize{
\centering
\begin{tikzpicture}

\begin{groupplot}[%
    ,group style={group name=plt_vs_parabel_times, columns=2, rows=4, horizontal sep=1.5cm, vertical sep=1.1cm}
    ,title style={at={(0.5,0.92)}}
    ,width=0.45\linewidth
    ,height=3.75cm
    ,xmode=log
    ,ymode=log
    ,log ticks with fixed point
    ,scaled x ticks=real:1e3
    ,scaled y ticks=real:1e3
]%
\nextgroupplot[title=\eurlex, legend to name=grouplegend, legend style={legend columns=2}, ylabel={$T/N_{\textrm{test}}$ [ms]}]%
            \addplot+[name path=avg_plt, color=putblue, mark options={fill=putblue}, smooth] coordinates {
                (1,0.2957)
                (3,0.2817)
                (5,0.2535)
                (10,0.2450)
                (30,0.2493)
                (50,0.2510)
                (100,0.2518)
                (300,0.2511)
                (500,0.2465)
                (1000,0.2432)
                (3000,0.2464)
                (5000,0.2389)
            };\addlegendentry{u-c search}%
            \addplot+[name path=avg_parabel, color=putred, mark options={fill=putred}, smooth] coordinates {
                (1,2.2536)
                (3,0.9113)
                (5,0.6421)
                (10,0.4494)
                (30,0.2971)
                (50,0.2686)
                (100,0.2474)
                (300,0.2347)
                (500,0.2320)
                (1000,0.2349)
                (3000,0.2292)
                (5000,0.2076)
            };\addlegendentry{beam search}%
                \coordinate (top) at (rel axis cs:0,1);%
\nextgroupplot[title=\amazoncatsmall]%
            \addplot+[name path=avg_plt, color=putblue, mark options={fill=putblue}, smooth] coordinates {
                (1,0.4101)
                (3,0.3873)
                (5,0.3691)
                (10,0.3794)
                (30,0.3657)
                (50,0.3616)
                (100,0.3579)
                (300,0.3548)
                (500,0.3495)
                (1000,0.3461)
                (3000,0.3447)
                (5000,0.3456)
            };%
            \addplot+[name path=avg_parabel, color=putred, mark options={fill=putred}, smooth] coordinates {
                (1,148.8617)
                (3,49.8993)
                (5,29.3805)
                (10,14.8517)
                (30,5.0825)
                (50,3.0828)
                (100,1.5691)
                (300,0.6274)
                (500,0.4164)
                (1000,0.2740)
                (3000,0.1867)
                (5000,0.1683)
            };%
\nextgroupplot[title=\wikiten, ylabel={$T/N_{\textrm{test}}$ [ms]}]%
            \addplot+[name path=avg_plt, color=putblue, mark options={fill=putblue}, smooth] coordinates {
                (1,3.0323)
                (3,2.7032)
                (5,2.7173)
                (10,2.7596)
                (30,2.8192)
                (50,2.8203)
                (100,2.8614)
                (300,2.8739)
                (500,2.8803)
                (1000,2.8846)
                (3000,2.8873)
                (5000,3.1499)
            };%
            \addplot+[name path=avg_parabel, color=putred, mark options={fill=putred}, smooth] coordinates {
                (1,41.7057)
                (3,14.3029)
                (5,8.9992)
                (10,4.7505)
                (30,2.0723)
                (50,1.5648)
                (100,1.1326)
                (300,0.9149)
                (500,0.7943)
                (1000,0.7770)
                (3000,0.8507)
                (5000,0.8561)
            };%
\nextgroupplot[title=\deliciouslarge]%
            \addplot+[name path=avg_plt, color=putblue, mark options={fill=putblue}, smooth] coordinates {
                (1,11.1312)
                (3,11.6259)
                (5,11.0817)
                (10,12.2979)
                (30,12.2910)
                (50,12.5296)
                (100,12.8223)
                (300,12.5379)
                (500,12.5087)
                (1000,12.4725)
                (3000,12.4923)
                (5000,12.8100)
            };%
            \addplot+[name path=avg_parabel, color=putred, mark options={fill=putred}, smooth] coordinates {
                (1,1560.0019)
                (3,544.4581)
                (5,322.3766)
                (10,168.1329)
                (30,54.3223)
                (50,33.1740)
                (100,16.3127)
                (300,5.8310)
                (500,3.8171)
                (1000,2.0999)
                (3000,1.1001)
                (5000,0.8919)
            };%
\nextgroupplot[title=\wikilshtc, ylabel={$T/N_{\textrm{test}}$ [ms]}]%
            \addplot+[name path=avg_plt, color=putblue, mark options={fill=putblue}, smooth] coordinates {
                (1,1.0071)
                (3,0.8268)
                (5,0.8131)
                (10,0.7479)
                (30,0.7709)
                (50,0.7714)
                (100,0.7757)
                (300,0.5530)
                (500,0.6510)
                (1000,0.7551)
                (3000,0.7966)
                (5000,0.7672)
            };%
            \addplot+[name path=avg_parabel, color=putred, mark options={fill=putred}, smooth] coordinates {
                (1,991.0036)
                (3,342.7682)
                (5,203.6264)
                (10,107.9469)
                (30,34.3212)
                (50,20.7011)
                (100,10.2296)
                (300,3.6456)
                (500,2.3110)
                (1000,1.2286)
                (3000,0.5012)
                (5000,0.3523)
            };%
\nextgroupplot[title=\wikipedia]%
            \addplot+[name path=avg_plt, color=putblue, mark options={fill=putblue}, smooth] coordinates {
                (1,4.2496)
                (3,4.7417)
                (5,4.8371)
                (10,4.8371)
                (30,4.8679)
                (50,5.1129)
                (100,4.8773)
                (300,3.9726)
                (500,3.2778)
                (1000,3.6627)
                (3000,3.4296)
                (5000,3.3359)
            };%
            \addplot+[name path=avg_parabel, color=putred, mark options={fill=putred}, smooth] coordinates {
                (1,1709.1639)
                (3,591.7830)
                (5,368.7918)
                (10,197.4506)
                (30,63.2851)
                (50,38.4733)
                (100,19.7521)
                (300,7.0935)
                (500,5.0259)
                (1000,3.4476)
                (3000,1.8901)
                (5000,1.1643)
            };%
\nextgroupplot[title=\amazon, ylabel={$T/N_{\textrm{test}}$ [ms]}, xlabel={$N_{\textrm{test}}$}]%
            \addplot+[name path=avg_plt, color=putblue, mark options={fill=putblue}, smooth] coordinates {
                (1,1.4688)
                (3,1.5686)
                (5,1.4688)
                (10,1.5992)
                (30,1.5923)
                (50,1.5858)
                (100,1.5092)
                (300,1.5576)
                (500,1.5992)
                (1000,1.2781)
                (3000,1.4324)
                (5000,1.4484)
            };%
            \addplot+[name path=avg_parabel, color=putred, mark options={fill=putred}, smooth] coordinates {
                (1,501.8380)
                (3,166.4142)
                (5,99.3791)
                (10,52.0475)
                (30,17.0829)
                (50,9.9819)
                (100,5.2635)
                (300,1.8712)
                (500,1.2076)
                (1000,0.6799)
                (3000,0.3462)
                (5000,0.2845)
            };%
\nextgroupplot[title=\amazonlarge, xlabel={$N_{\textrm{test}}$}]%
            \addplot+[name path=avg_plt, color=putblue, mark options={fill=putblue}, smooth] coordinates {
                (1,1.7153)
                (3,2.1122)
                (5,1.9816)
                (10,2.1595)
                (30,2.7672)
                (50,2.9547)
                (100,2.7200)
                (300,2.5142)
                (500,2.6011)
                (1000,2.5320)
                (3000,2.3624)
                (5000,2.0297)
            };%
            \addplot+[name path=avg_parabel, color=putred, mark options={fill=putred}, smooth] coordinates {
                (1,8380.5069)
                (3,2845.9193)
                (5,1667.2510)
                (10,888.6968)
                (30,291.6177)
                (50,171.1401)
                (100,84.1458)
                (300,28.1344)
                (500,17.5065)
                (1000,8.9507)
                (3000,3.2835)
                (5000,1.9309)
            };%
                \coordinate (bot) at (rel axis cs:1,0);
\end{groupplot}
\path (top|-current bounding box.north) -- node[above]{\ref{grouplegend}} (bot|-current bounding box.north);
\end{tikzpicture}
}
\caption{Average prediction times of uniform-cost search (\Algo{napkinXC}) and beam search (\Algo{Parabel}) as a function of batch size $N_{\textrm{test}}$.}
\label{fig:batches-times}
\end{figure}

%% file: tables/table-representations.tex
\begin{table}[H]
\footnotesize{
\resizebox{\textwidth}{!}{
\tabcolsep=5pt
\begin{tabular}{l|r@{}lr@{}l|r@{}lr@{}l|r@{}lr@{}l}
\specialrule{0.94pt}{0.4ex}{0.65ex}
Representation & \multicolumn{2}{c}{dense} & \multicolumn{2}{c|}{sparse}
& \multicolumn{2}{c}{dense} & \multicolumn{2}{c|}{sparse}
& \multicolumn{2}{c}{dense} & \multicolumn{2}{c}{sparse} \\
\specialrule{0.94pt}{0.4ex}{0.65ex}
& \multicolumn{4}{c|}{$p@1$ [\%]}
& \multicolumn{4}{c|}{$p@3$ [\%]}
& \multicolumn{4}{c}{$p@5$ [\%]} \\
\midrule
\eurlex & $77.29$ & $\pm 0.21$ & \boldmath$80.43$ & \boldmath$\pm 0.09$ & $64.41$ & $\pm 0.11$ & \boldmath$66.08$ & \boldmath$\pm 0.26$ & $53.56$ & $\pm 0.11$ & \boldmath$53.87$ & \boldmath$\pm 0.58$ \\
\amazoncatsmall & $91.96$ & $\pm 0.02$ & \boldmath$93.23$ & \boldmath$\pm 0.02$ & $77.41$ & $\pm 0.01$ & \boldmath$78.76$ & \boldmath$\pm 0.03$ & $62.75$ & $\pm 0.02$ & \boldmath$64.05$ & \boldmath$\pm 0.02$ \\
\wikiten & \boldmath$85.76$ & \boldmath$\pm 0.06$ & $84.92$ & $\pm 0.10$ & $74.37$ & $\pm 0.10$ & \boldmath$74.52$ & \boldmath$\pm 0.09$ & $64.44$ & $\pm 0.05$ & \boldmath$65.29$ & \boldmath$\pm 0.04$ \\
\deliciouslarge & \boldmath$47.95$ & \boldmath$\pm 0.03$ & $45.27$ & $\pm 0.06$ & \boldmath$41.69$ & \boldmath$\pm 0.02$ & $38.26$ & $\pm 0.03$ & \boldmath$38.60$ & \boldmath$\pm 0.01$ & $34.88$ & $\pm 0.02$ \\
\wikilshtc & $57.58$ & $\pm 0.06$ & \boldmath$60.99$ & \boldmath$\pm 0.04$ & $38.01$ & $\pm 0.04$ & \boldmath$39.85$ & \boldmath$\pm 0.02$ & $28.33$ & $\pm 0.02$ & \boldmath$29.50$ & \boldmath$\pm 0.01$ \\
\wikipedia & $64.56$ & $\pm 0.06$ & \boldmath$65.68$ & \boldmath$\pm 0.15$ & $46.04$ & $\pm 0.06$ & \boldmath$46.62$ & \boldmath$\pm 0.09$ & $36.06$ & $\pm 0.04$ & \boldmath$36.52$ & \boldmath$\pm 0.06$ \\
\amazon & $40.24$ & $\pm 0.03$ & \boldmath$43.82$ & \boldmath$\pm 0.01$ & $35.84$ & $\pm 0.04$ & \boldmath$38.88$ & \boldmath$\pm 0.03$ & $32.61$ & $\pm 0.05$ & \boldmath$35.31$ & \boldmath$\pm 0.03$ \\
\amazonlarge & $39.29$ & $\pm 0.03$ & \boldmath$43.61$ & \boldmath$\pm 0.12$ & $36.53$ & $\pm 0.02$ & \boldmath$40.44$ & \boldmath$\pm 0.09$ & $34.65$ & $\pm 0.01$ & \boldmath$38.32$ & \boldmath$\pm 0.06$ \\

\specialrule{0.94pt}{0.4ex}{0.65ex}
& \multicolumn{4}{c|}{$T_{\textrm{train}}$ [h]}
& \multicolumn{4}{c|}{$T/N_{\textrm{test}}$ [ms]}
& \multicolumn{4}{c}{$M_{\textrm{size}}$ [GB]} \\
\midrule
\eurlex & $0.21$ & $\pm 0.00$ & \boldmath$0.01$ & \boldmath$\pm 0.00$ & $0.36$ & $\pm 0.00$ & \boldmath$0.25$ & \boldmath$\pm 0.02$ & \boldmath$0.02$ & \boldmath$\pm 0.00$ & $0.05$ & $\pm 0.00$ \\
\amazoncatsmall & $10.75$ & $\pm 0.71$ & \boldmath$0.17$ & \boldmath$\pm 0.00$ & \boldmath$0.17$ & \boldmath$\pm 0.01$ & $0.32$ & $\pm 0.02$ & \boldmath$0.41$ & \boldmath$\pm 0.00$ & $0.72$ & $\pm 0.00$ \\
\wikiten & $0.72$ & $\pm 0.06$ & \boldmath$0.11$ & \boldmath$\pm 0.00$ & \boldmath$0.91$ & \boldmath$\pm 0.06$ & $5.24$ & $\pm 0.42$ & \boldmath$0.25$ & \boldmath$\pm 0.00$ & $0.91$ & $\pm 0.00$ \\
\deliciouslarge & $38.32$ & $\pm 1.84$ & \boldmath$2.97$ & \boldmath$\pm 0.03$ & \boldmath$1.69$ & \boldmath$\pm 0.05$ & $11.51$ & $\pm 0.57$ & \boldmath$1.90$ & \boldmath$\pm 0.00$ & $15.05$ & $\pm 0.00$ \\
\wikilshtc & $2.34$ & $\pm 0.17$ & \boldmath$1.52$ & \boldmath$\pm 0.00$ & \boldmath$0.57$ & \boldmath$\pm 0.01$ & $2.44$ & $\pm 0.02$ & \boldmath$3.30$ & \boldmath$\pm 0.00$ & $4.93$ & $\pm 0.00$ \\
\wikipedia & $28.23$ & $\pm 1.41$ & \boldmath$8.11$ & \boldmath$\pm 0.18$ & \boldmath$0.64$ & \boldmath$\pm 0.01$ & $7.86$ & $\pm 0.19$ & \boldmath$5.50$ & \boldmath$\pm 0.00$ & $19.11$ & $\pm 0.00$ \\
\amazon & $6.27$ & $\pm 0.49$ & \boldmath$0.42$ & \boldmath$\pm 0.00$ & \boldmath$1.18$ & \boldmath$\pm 0.01$ & $5.93$ & $\pm 0.11$ & \boldmath$1.60$ & \boldmath$\pm 0.00$ & $6.22$ & $\pm 0.00$ \\
\amazonlarge & $36.07$ & $\pm 2.05$ & \boldmath$4.44$ & \boldmath$\pm 0.10$ & \boldmath$0.91$ & \boldmath$\pm 0.01$ & $4.42$ & $\pm 0.00$ & \boldmath$6.20$ & \boldmath$\pm 0.00$ & $52.16$ & $\pm 0.00$ \\

\specialrule{0.94pt}{0.4ex}{0.65ex}
\end{tabular}}}
\caption{Comparison of dense (\Algo{extremeText}) and sparse (\Algo{napkinXC}) representation.}
\label{tab:representations}
\end{table}

%% file: tables/table-trees.tex
\begin{table}[H]
\centering
\footnotesize{
\resizebox{\textwidth}{!}{
\tabcolsep=5pt
\begin{tabular}{l|r@{}lr@{}l|r@{}lr@{}l|r@{}lr@{}l}
\toprule
Tree type
& \multicolumn{2}{c}{random} & \multicolumn{2}{c|}{k-means}
& \multicolumn{2}{c}{random} & \multicolumn{2}{c|}{k-means}
& \multicolumn{2}{c}{random} & \multicolumn{2}{c}{k-means} \\
\midrule
& \multicolumn{4}{c|}{$p@1$ [\%]}
& \multicolumn{4}{c|}{$p@3$ [\%]}
& \multicolumn{4}{c}{$p@5$ [\%]} \\

\midrule
\eurlex & $76.22$ & $\pm 0.21$ & \boldmath$80.51$ & \boldmath$\pm 0.16$ & $54.54$ & $\pm 0.13$ & \boldmath$65.65$ & \boldmath$\pm 0.40$ & $39.20$ & $\pm 0.09$ & \boldmath$53.33$ & \boldmath$\pm 0.68$ \\
\amazoncatsmall & $91.39$ & $\pm 0.04$ & \boldmath$93.04$ & \boldmath$\pm 0.02$ & $75.86$ & $\pm 0.03$ & \boldmath$78.44$ & \boldmath$\pm 0.02$ & $61.00$ & $\pm 0.02$ & \boldmath$63.70$ & \boldmath$\pm 0.02$ \\
\wikiten & $84.26$ & $\pm 0.07$ & \boldmath$85.36$ & \boldmath$\pm 0.09$ & $71.68$ & $\pm 0.06$ & \boldmath$73.90$ & \boldmath$\pm 0.07$ & $60.16$ & $\pm 0.12$ & \boldmath$63.84$ & \boldmath$\pm 0.07$ \\
\deliciouslarge & $49.13$ & $\pm 0.03$ & \boldmath$49.55$ & \boldmath$\pm 0.05$ & $42.68$ & $\pm 0.01$ & \boldmath$43.08$ & \boldmath$\pm 0.03$ & $39.47$ & $\pm 0.02$ & \boldmath$39.90$ & \boldmath$\pm 0.02$ \\
\wikilshtc & $44.04$ & $\pm 0.02$ & \boldmath$61.96$ & \boldmath$\pm 0.03$ & $24.69$ & $\pm 0.45$ & \boldmath$40.77$ & \boldmath$\pm 0.02$ & $18.20$ & $\pm 0.18$ & \boldmath$30.19$ & \boldmath$\pm 0.02$ \\
\wikipedia & $48.40$ & $\pm 0.02$ & \boldmath$66.20$ & \boldmath$\pm 0.05$ & $32.05$ & $\pm 0.01$ & \boldmath$47.14$ & \boldmath$\pm 0.02$ & $24.82$ & $\pm 0.01$ & \boldmath$36.83$ & \boldmath$\pm 0.01$ \\
\amazon & $33.76$ & $\pm 0.06$ & \boldmath$43.54$ & \boldmath$\pm 0.01$ & $28.25$ & $\pm 0.02$ & \boldmath$38.71$ & \boldmath$\pm 0.02$ & $24.88$ & $\pm 0.01$ & \boldmath$35.15$ & \boldmath$\pm 0.03$ \\
\amazonlarge & $37.77$ & $\pm 0.05$ & \boldmath$46.09$ & \boldmath$\pm 0.02$ & $34.55$ & $\pm 0.01$ & \boldmath$43.11$ & \boldmath$\pm 0.01$ & $32.49$ & $\pm 0.01$ & \boldmath$40.98$ & \boldmath$\pm 0.01$ \\

\bottomrule

& \multicolumn{4}{c|}{$T_{\textrm{train}}$ [h]}
& \multicolumn{4}{c|}{$T/N_{\textrm{test}}$ [ms]}
& \multicolumn{4}{c}{$M_{\textrm{size}}$ [GB]} \\
\midrule
\eurlex & \boldmath$0.02$ & \boldmath$\pm 0.00$ & $0.02$ & $\pm 0.00$ & \boldmath$0.27$ & \boldmath$\pm 0.02$ & $0.39$ & $\pm 0.03$ & \boldmath$0.02$ & \boldmath$\pm 0.00$ & $0.02$ & $\pm 0.00$ \\
\amazoncatsmall & $0.84$ & $\pm 0.03$ & \boldmath$0.72$ & \boldmath$\pm 0.02$ & $0.37$ & $\pm 0.02$ & \boldmath$0.32$ & \boldmath$\pm 0.03$ & $0.40$ & $\pm 0.00$ & \boldmath$0.35$ & \boldmath$\pm 0.00$ \\
\wikiten & \boldmath$0.19$ & \boldmath$\pm 0.01$ & $0.21$ & $\pm 0.00$ & \boldmath$2.59$ & \boldmath$\pm 0.21$ & $5.35$ & $\pm 0.32$ & $0.62$ & $\pm 0.00$ & \boldmath$0.58$ & \boldmath$\pm 0.00$ \\
\deliciouslarge & $5.26$ & $\pm 0.26$ & \boldmath$2.58$ & \boldmath$\pm 0.15$ & \boldmath$5.24$ & \boldmath$\pm 0.43$ & $9.89$ & $\pm 0.89$ & $1.30$ & $\pm 0.00$ & \boldmath$0.95$ & \boldmath$\pm 0.00$ \\
\wikilshtc & $3.01$ & $\pm 0.16$ & \boldmath$2.95$ & \boldmath$\pm 0.15$ & \boldmath$1.36$ & \boldmath$\pm 0.13$ & $1.77$ & $\pm 0.11$ & $3.25$ & $\pm 0.00$ & \boldmath$2.73$ & \boldmath$\pm 0.00$ \\
\wikipedia & $22.07$ & $\pm 0.59$ & \boldmath$16.10$ & \boldmath$\pm 0.44$ & $10.65$ & $\pm 0.52$ & \boldmath$6.67$ & \boldmath$\pm 0.23$ & $12.00$ & $\pm 0.00$ & \boldmath$8.89$ & \boldmath$\pm 0.00$ \\
\amazon & $1.02$ & $\pm 0.02$ & \boldmath$0.56$ & \boldmath$\pm 0.00$ & $4.74$ & $\pm 0.22$ & \boldmath$4.13$ & \boldmath$\pm 0.28$ & $3.00$ & $\pm 0.00$ & \boldmath$2.26$ & \boldmath$\pm 0.00$ \\
\amazonlarge & $48.09$ & $\pm 1.72$ & \boldmath$7.07$ & \boldmath$\pm 0.56$ & $8.05$ & $\pm 0.17$ & \boldmath$3.26$ & \boldmath$\pm 0.08$ & $29.00$ & $\pm 0.00$ & \boldmath$20.84$ & \boldmath$\pm 0.00$ \\

\bottomrule
\end{tabular}}}

\caption{\Precatk{} for $k=1,3,5$, training time, average test time per example, and model size for random and k-means trees.}
\label{tab:random-vs-kmeans-pred-comp-perf}
\end{table}

%% file: tables/table-trees-arity-maxleaves-logloss.tex
\begin{table}[h]
\centering

\begin{subtable}{1\textwidth}

\footnotesize{
\resizebox{\textwidth}{!}{
\tabcolsep=5pt

\begin{tabular}{l|r@{}lr@{}lr@{}l|r@{}lr@{}lr@{}l}

\specialrule{0.94pt}{0.4ex}{0.65ex}
Arity & \multicolumn{2}{c}{2} & \multicolumn{2}{c}{16} & \multicolumn{2}{c|}{64} 
& \multicolumn{2}{c}{2} & \multicolumn{2}{c}{16} & \multicolumn{2}{c}{64} \\
\specialrule{0.94pt}{0.4ex}{0.65ex}
& \multicolumn{6}{c|}{$p@1$ [\%]}
& \multicolumn{6}{c}{$T/N_{\textrm{test}}$ [ms]} \\
\midrule
% \eurlex & $80.51$ & $\pm 0.16$ & $81.06$ & $\pm 0.03$ & \boldmath$81.20$ & \boldmath$\pm 0.13$ & $0.39$ & $\pm 0.03$ & \boldmath$0.23$ & \boldmath$\pm 0.02$ & $0.28$ & $\pm 0.03$ \\
% \amazoncatsmall & $93.04$ & $\pm 0.02$ & $93.15$ & $\pm 0.04$ & \boldmath$93.19$ & \boldmath$\pm 0.01$ & \boldmath$0.32$ & \boldmath$\pm 0.03$ & $0.34$ & $\pm 0.03$ & $0.43$ & $\pm 0.04$ \\
% \wikiten & $85.36$ & $\pm 0.09$ & $85.85$ & $\pm 0.06$ & \boldmath$86.03$ & \boldmath$\pm 0.07$ & \boldmath$5.35$ & \boldmath$\pm 0.32$ & $5.63$ & $\pm 0.33$ & $7.34$ & $\pm 0.62$ \\
% \deliciouslarge & $49.55$ & $\pm 0.05$ & \boldmath$49.56$ & \boldmath$\pm 0.04$ & $49.51$ & $\pm 0.04$ & \boldmath$9.89$ & \boldmath$\pm 0.89$ & $12.74$ & $\pm 1.69$ & $11.77$ & $\pm 0.61$ \\
\wikilshtc & $61.96$ & $\pm 0.03$ & $63.16$ & $\pm 0.03$ & \boldmath$63.62$ & \boldmath$\pm 0.03$ & $1.77$ & $\pm 0.11$ & \boldmath$1.26$ & \boldmath$\pm 0.03$ & $1.91$ & $\pm 0.19$ \\
\wikipedia & $66.20$ & $\pm 0.05$ & $67.36$ & $\pm 0.10$ & \boldmath$67.49$ & \boldmath$\pm 0.05$ & $6.67$ & $\pm 0.23$ & \boldmath$6.47$ & \boldmath$\pm 0.30$ & $9.02$ & $\pm 0.39$ \\
\amazon & \boldmath$43.54$ & \boldmath$\pm 0.01$ & $43.33$ & $\pm 0.04$ & $43.53$ & $\pm 0.05$ & $4.13$ & $\pm 0.28$ & \boldmath$2.91$ & \boldmath$\pm 0.17$ & $5.12$ & $\pm 0.24$ \\
\amazonlarge & $46.09$ & $\pm 0.02$ & $46.74$ & $\pm 0.01$ & \boldmath$46.97$ & \boldmath$\pm 0.01$ & $3.26$ & $\pm 0.08$ & \boldmath$3.06$ & \boldmath$\pm 0.04$ & $4.19$ & $\pm 0.16$ \\

\specialrule{0.94pt}{0.4ex}{0.65ex}
& \multicolumn{6}{c|}{$T_{\textrm{train}}$ [h]}
& \multicolumn{6}{c}{$M_{\textrm{size}}$ [GB]} \\
\midrule
% \eurlex & \boldmath$0.02$ & \boldmath$\pm 0.00$ & $0.02$ & $\pm 0.00$ & $0.02$ & $\pm 0.00$ & \boldmath$0.02$ & \boldmath$\pm 0.00$ & $0.02$ & $\pm 0.00$ & $0.02$ & $\pm 0.00$ \\
% \amazoncatsmall & \boldmath$0.72$ & \boldmath$\pm 0.02$ & $0.82$ & $\pm 0.05$ & $1.09$ & $\pm 0.03$ & $0.35$ & $\pm 0.00$ & \boldmath$0.34$ & \boldmath$\pm 0.00$ & $0.34$ & $\pm 0.00$ \\
% \wikiten & $0.21$ & $\pm 0.00$ & \boldmath$0.20$ & \boldmath$\pm 0.01$ & $0.43$ & $\pm 0.04$ & $0.58$ & $\pm 0.00$ & \boldmath$0.50$ & \boldmath$\pm 0.00$ & $0.51$ & $\pm 0.00$ \\
% \deliciouslarge & \boldmath$2.58$ & \boldmath$\pm 0.15$ & $5.49$ & $\pm 0.36$ & $9.21$ & $\pm 0.90$ & $0.95$ & $\pm 0.00$ & $0.93$ & $\pm 0.00$ & \boldmath$0.91$ & \boldmath$\pm 0.00$ \\
\wikilshtc & \boldmath$2.95$ & \boldmath$\pm 0.15$ & $3.71$ & $\pm 0.13$ & $7.13$ & $\pm 0.41$ & $2.73$ & $\pm 0.00$ & $2.65$ & $\pm 0.00$ & \boldmath$2.62$ & \boldmath$\pm 0.00$ \\
\wikipedia & \boldmath$16.10$ & \boldmath$\pm 0.44$ & $26.28$ & $\pm 0.40$ & $46.84$ & $\pm 2.94$ & $8.89$ & $\pm 0.00$ & $8.20$ & $\pm 0.00$ & \boldmath$8.13$ & \boldmath$\pm 0.00$ \\
\amazon & \boldmath$0.56$ & \boldmath$\pm 0.00$ & $0.78$ & $\pm 0.03$ & $2.17$ & $\pm 0.12$ & $2.26$ & $\pm 0.00$ & $1.68$ & $\pm 0.00$ & \boldmath$1.60$ & \boldmath$\pm 0.00$ \\
\amazonlarge & \boldmath$7.07$ & \boldmath$\pm 0.56$ & $9.80$ & $\pm 0.10$ & $23.69$ & $\pm 1.09$ & $20.84$ & $\pm 0.00$ & $20.32$ & $\pm 0.00$ & \boldmath$20.17$ & \boldmath$\pm 0.00$ \\

\specialrule{0.94pt}{0.4ex}{0.65ex}
\end{tabular}}}
\caption{Results for arity equal to 2, 16 or 64 and pre-leaf node degree equal to 100.}
\label{tab:k-means-tree-arity-logistic}
\end{subtable}

\vspace{12pt}

\begin{subtable}{\textwidth}

\footnotesize{
\resizebox{\textwidth}{!}{
\tabcolsep=5pt
\begin{tabular}{l|r@{}lr@{}lr@{}l|r@{}lr@{}lr@{}l}
\specialrule{0.94pt}{0.4ex}{0.65ex}
Pre-leaf degree & \multicolumn{2}{c}{25} & \multicolumn{2}{c}{100} & \multicolumn{2}{c|}{400} 
& \multicolumn{2}{c}{25} & \multicolumn{2}{c}{100} & \multicolumn{2}{c}{400} \\
\specialrule{0.94pt}{0.4ex}{0.65ex}
& \multicolumn{6}{c|}{$p@1$ [\%]}
& \multicolumn{6}{c}{$T/N_{\textrm{test}}$ [ms]} \\
\midrule
% \eurlex & \boldmath$80.62$ & \boldmath$\pm 0.12$ & $80.51$ & $\pm 0.16$ & $79.49$ & $\pm 0.39$ & \boldmath$0.15$ & \boldmath$\pm 0.00$ & $0.39$ & $\pm 0.03$ & $0.62$ & $\pm 0.02$ \\
% \amazoncatsmall & $93.05$ & $\pm 0.02$ & $93.04$ & $\pm 0.02$ & \boldmath$93.07$ & \boldmath$\pm 0.02$ & \boldmath$0.14$ & \boldmath$\pm 0.02$ & $0.32$ & $\pm 0.03$ & $0.67$ & $\pm 0.04$ \\
% \wikiten & $85.28$ & $\pm 0.08$ & $85.36$ & $\pm 0.09$ & \boldmath$85.41$ & \boldmath$\pm 0.06$ & \boldmath$2.07$ & \boldmath$\pm 0.30$ & $5.35$ & $\pm 0.32$ & $11.43$ & $\pm 0.21$ \\
% \deliciouslarge & $49.56$ & $\pm 0.03$ & $49.55$ & $\pm 0.05$ & \boldmath$49.62$ & \boldmath$\pm 0.03$ & \boldmath$6.25$ & \boldmath$\pm 0.35$ & $9.89$ & $\pm 0.89$ & $18.23$ & $\pm 0.72$ \\
\wikilshtc & $61.42$ & $\pm 0.03$ & $61.96$ & $\pm 0.03$ & \boldmath$62.16$ & \boldmath$\pm 0.05$ & \boldmath$0.66$ & \boldmath$\pm 0.04$ & $1.77$ & $\pm 0.11$ & $4.21$ & $\pm 0.10$ \\
\wikipedia & $65.72$ & $\pm 0.12$ & \boldmath$66.20$ & \boldmath$\pm 0.05$ & $65.95$ & $\pm 0.10$ & \boldmath$3.52$ & \boldmath$\pm 0.05$ & $6.67$ & $\pm 0.23$ & $20.13$ & $\pm 1.07$ \\
\amazon & $41.83$ & $\pm 0.02$ & \boldmath$43.54$ & \boldmath$\pm 0.01$ & $43.21$ & $\pm 0.03$ & \boldmath$1.61$ & \boldmath$\pm 0.04$ & $4.13$ & $\pm 0.28$ & $11.74$ & $\pm 0.46$ \\
\amazonlarge & $46.07$ & $\pm 0.01$ & $46.09$ & $\pm 0.02$ & \boldmath$46.13$ & \boldmath$\pm 0.01$ & \boldmath$1.47$ & \boldmath$\pm 0.00$ & $3.26$ & $\pm 0.08$ & $11.45$ & $\pm 0.22$ \\

\specialrule{0.94pt}{0.4ex}{0.65ex}
& \multicolumn{6}{c|}{$T_{\textrm{train}}$ [h]}
& \multicolumn{6}{c}{$M_{\textrm{size}}$ [GB]} \\
\midrule
% \eurlex & \boldmath$0.01$ & \boldmath$\pm 0.00$ & $0.02$ & $\pm 0.00$ & $0.04$ & $\pm 0.00$ & $0.03$ & $\pm 0.00$ & \boldmath$0.02$ & \boldmath$\pm 0.00$ & \boldmath$0.02$ & \boldmath$\pm 0.00$ \\
% \amazoncatsmall & \boldmath$0.36$ & \boldmath$\pm 0.03$ & $0.72$ & $\pm 0.02$ & $1.38$ & $\pm 0.03$ & $0.38$ & $\pm 0.00$ & $0.35$ & $\pm 0.00$ & \boldmath$0.34$ & \boldmath$\pm 0.00$ \\
% \wikiten & \boldmath$0.09$ & \boldmath$\pm 0.01$ & $0.21$ & $\pm 0.00$ & $0.32$ & $\pm 0.01$ & $0.73$ & $\pm 0.00$ & $0.58$ & $\pm 0.00$ & \boldmath$0.37$ & \boldmath$\pm 0.00$ \\
% \deliciouslarge & \boldmath$2.14$ & \boldmath$\pm 0.11$ & $2.58$ & $\pm 0.15$ & $4.18$ & $\pm 0.23$ & $1.14$ & $\pm 0.00$ & $0.95$ & $\pm 0.00$ & \boldmath$0.86$ & \boldmath$\pm 0.00$ \\
\wikilshtc & \boldmath$2.36$ & \boldmath$\pm 0.05$ & $2.95$ & $\pm 0.15$ & $4.83$ & $\pm 0.29$ & $3.28$ & $\pm 0.00$ & $2.73$ & $\pm 0.00$ & \boldmath$2.47$ & \boldmath$\pm 0.00$ \\
\wikipedia & \boldmath$12.40$ & \boldmath$\pm 0.08$ & $16.10$ & $\pm 0.44$ & $34.15$ & $\pm 0.46$ & $10.60$ & $\pm 0.01$ & $8.89$ & $\pm 0.00$ & \boldmath$7.94$ & \boldmath$\pm 0.00$ \\
\amazon & \boldmath$0.42$ & \boldmath$\pm 0.01$ & $0.56$ & $\pm 0.00$ & $1.01$ & $\pm 0.05$ & $2.20$ & $\pm 0.00$ & $2.26$ & $\pm 0.00$ & \boldmath$1.52$ & \boldmath$\pm 0.00$ \\
\amazonlarge & \boldmath$5.18$ & \boldmath$\pm 0.25$ & $7.07$ & $\pm 0.56$ & $16.41$ & $\pm 0.52$ & $24.00$ & $\pm 0.00$ & $20.84$ & $\pm 0.00$ & \boldmath$19.35$ & \boldmath$\pm 0.01$ \\

\specialrule{0.94pt}{0.4ex}{0.65ex}
\end{tabular}}}
\caption{Results arity equal to 2 and pre-leaf node degree equal to 25, 100, or 400.}
\label{tab:k-means-tree-maxleaves-logistic}
\end{subtable}

\caption{\Precatk{1}, average prediction time per example, training time and model size for $k$-means trees if different depths with logistic loss.}
\label{tab::k-means-tree-depths-logistic}
\end{table}

%% file: pics/plots-ensembles.tex
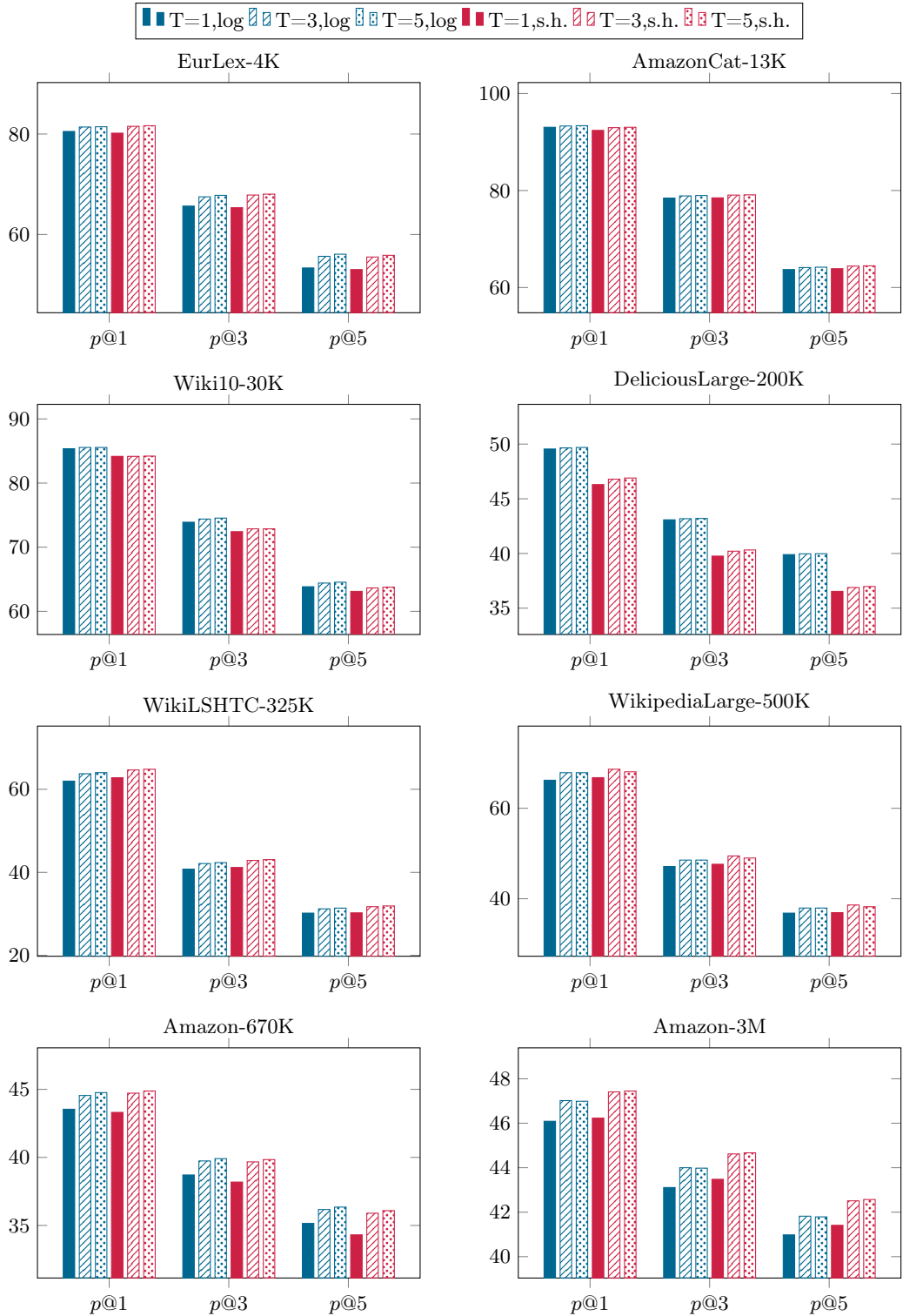
\begin{figure}[htb]
%\usetikzlibrary{pgfplots.groupplots}
\tikzset{fontscale/.style = {font=\footnotesize}}
\footnotesize

\centering
\begin{tikzpicture}

\begin{groupplot}[%
    ybar
    ,/pgf/bar width=5pt
    ,enlargelimits=0.3
    ,group style={group name=ensembles, columns=2, rows=4, horizontal sep=1.5cm, vertical sep=1.4cm}
    ,title style={at={(0.5,0.97)}}
    ,width=0.45\linewidth
    ,height=5.1cm
    ,symbolic x coords={$p@1$,$p@3$,$p@5$},
    ,xtick=data
    % ,ymin=0
    % ,ymax=100
]%
\nextgroupplot[title=\eurlex, legend to name=grouplegend, legend style={legend columns=6}]%
            \addplot+[color=putblue] coordinates {
                ($p@1$,80.51)
                ($p@3$,65.65)
                ($p@5$,53.33)
            };\addlegendentry{T=1,log}%
            \addplot+[color=putblue, pattern=north east lines, pattern color=putblue] coordinates {
                ($p@1$,81.39)
                ($p@3$,67.49)
                ($p@5$,55.65)
            };\addlegendentry{T=3,log}%
            \addplot+[color=putblue, pattern=crosshatch dots, pattern color=putblue] coordinates {
                ($p@1$,81.44)
                ($p@3$,67.79)
                ($p@5$,56.07)
            };\addlegendentry{T=5,log}%
            \addplot+[color=putred] coordinates {
                ($p@1$,80.17)
                ($p@3$,65.33)
                ($p@5$,53.01)
            };\addlegendentry{T=1,s.h.}%
            \addplot+[color=putred, pattern=north east lines, pattern color=putred] coordinates {
                ($p@1$,81.53)
                ($p@3$,67.85)
                ($p@5$,55.51)
            };\addlegendentry{T=3,s.h.}%
            \addplot+[color=putred, pattern=crosshatch dots, pattern color=putred] coordinates {
                ($p@1$,81.65)
                ($p@3$,68.03)
                ($p@5$,55.85)
            };\addlegendentry{T=5,s.h.}%
                \coordinate (top) at (rel axis cs:0,1);%
\nextgroupplot[title=\amazoncatsmall]%
            \addplot+[color=putblue] coordinates {
                ($p@1$,93.04)
                ($p@3$,78.44)
                ($p@5$,63.70)
            };%
            \addplot+[color=putblue, pattern=north east lines, pattern color=putblue] coordinates {
                ($p@1$,93.30)
                ($p@3$,78.89)
                ($p@5$,64.12)
            };%
            \addplot+[color=putblue, pattern=crosshatch dots, pattern color=putblue] coordinates {
                ($p@1$,93.36)
                ($p@3$,78.97)
                ($p@5$,64.19)
            };%
            \addplot+[color=putred] coordinates {
                ($p@1$,92.40)
                ($p@3$,78.49)
                ($p@5$,63.88)
            };%
            \addplot+[color=putred, pattern=north east lines, pattern color=putred] coordinates {
                ($p@1$,92.96)
                ($p@3$,79.06)
                ($p@5$,64.42)
            };%
            \addplot+[color=putred, pattern=crosshatch dots, pattern color=putred] coordinates {
                ($p@1$,93.04)
                ($p@3$,79.11)
                ($p@5$,64.48)
            };%
\nextgroupplot[title=\wikiten]%
            \addplot+[color=putblue] coordinates {
                ($p@1$,85.36)
                ($p@3$,73.90)
                ($p@5$,63.84)
            };%
            \addplot+[color=putblue, pattern=north east lines, pattern color=putblue] coordinates {
                ($p@1$,85.55)
                ($p@3$,74.39)
                ($p@5$,64.41)
            };%
            \addplot+[color=putblue, pattern=crosshatch dots, pattern color=putblue] coordinates {
                ($p@1$,85.55)
                ($p@3$,74.54)
                ($p@5$,64.54)
            };%
            \addplot+[color=putred] coordinates {
                ($p@1$,84.17)
                ($p@3$,72.43)
                ($p@5$,63.12)
            };%
            \addplot+[color=putred, pattern=north east lines, pattern color=putred] coordinates {
                ($p@1$,84.17)
                ($p@3$,72.86)
                ($p@5$,63.65)
            };%
            \addplot+[color=putred, pattern=crosshatch dots, pattern color=putred] coordinates {
                ($p@1$,84.21)
                ($p@3$,72.86)
                ($p@5$,63.77)
            };%
\nextgroupplot[title=\deliciouslarge]%
            \addplot+[color=putblue] coordinates {
                ($p@1$,49.55)
                ($p@3$,43.08)
                ($p@5$,39.90)
            };%
            \addplot+[color=putblue, pattern=north east lines, pattern color=putblue] coordinates {
                ($p@1$,49.65)
                ($p@3$,43.18)
                ($p@5$,39.97)
            };%
            \addplot+[color=putblue, pattern=crosshatch dots, pattern color=putblue] coordinates {
                ($p@1$,49.69)
                ($p@3$,43.20)
                ($p@5$,39.98)
            };%
            \addplot+[color=putred] coordinates {
                ($p@1$,46.30)
                ($p@3$,39.76)
                ($p@5$,36.54)
            };%
            \addplot+[color=putred, pattern=north east lines, pattern color=putred] coordinates {
                ($p@1$,46.79)
                ($p@3$,40.21)
                ($p@5$,36.89)
            };%
            \addplot+[color=putred, pattern=crosshatch dots, pattern color=putred] coordinates {
                ($p@1$,46.90)
                ($p@3$,40.32)
                ($p@5$,36.98)
            };%
\nextgroupplot[title=\wikilshtc]%
            \addplot+[color=putblue] coordinates {
                ($p@1$,61.96)
                ($p@3$,40.77)
                ($p@5$,30.19)
            };%
            \addplot+[color=putblue, pattern=north east lines, pattern color=putblue] coordinates {
                ($p@1$,63.73)
                ($p@3$,42.12)
                ($p@5$,31.21)
            };%
            \addplot+[color=putblue, pattern=crosshatch dots, pattern color=putblue] coordinates {
                ($p@1$,64.00)
                ($p@3$,42.34)
                ($p@5$,31.39)
            };%
            \addplot+[color=putred] coordinates {
                ($p@1$,62.78)
                ($p@3$,41.17)
                ($p@5$,30.25)
            };%
            \addplot+[color=putred, pattern=north east lines, pattern color=putred] coordinates {
                ($p@1$,64.65)
                ($p@3$,42.87)
                ($p@5$,31.73)
            };%
            \addplot+[color=putred, pattern=crosshatch dots, pattern color=putred] coordinates {
                ($p@1$,64.82)
                ($p@3$,43.06)
                ($p@5$,31.91)
            };%
\nextgroupplot[title=\wikipedia]%
            \addplot+[color=putblue] coordinates {
                ($p@1$,66.20)
                ($p@3$,47.14)
                ($p@5$,36.83)
            };%
            \addplot+[color=putblue, pattern=north east lines, pattern color=putblue] coordinates {
                ($p@1$,67.86)
                ($p@3$,48.55)
                ($p@5$,37.95)
            };%
            \addplot+[color=putblue, pattern=crosshatch dots, pattern color=putblue] coordinates {
                ($p@1$,67.86)
                ($p@3$,48.55)
                ($p@5$,37.95)
            };%
            \addplot+[color=putred] coordinates {
                ($p@1$,66.77)
                ($p@3$,47.63)
                ($p@5$,36.94)
            };%
            \addplot+[color=putred, pattern=north east lines, pattern color=putred] coordinates {
                ($p@1$,68.63)
                ($p@3$,49.46)
                ($p@5$,38.62)
            };%
            \addplot+[color=putred, pattern=crosshatch dots, pattern color=putred] coordinates {
                ($p@1$,68.06)
                ($p@3$,49.00)
                ($p@5$,38.24)
            };%
\nextgroupplot[title=\amazon]%
            \addplot+[color=putblue] coordinates {
                ($p@1$,43.54)
                ($p@3$,38.71)
                ($p@5$,35.15)
            };%
            \addplot+[color=putblue, pattern=north east lines, pattern color=putblue] coordinates {
                ($p@1$,44.55)
                ($p@3$,39.74)
                ($p@5$,36.17)
            };%
            \addplot+[color=putblue, pattern=crosshatch dots, pattern color=putblue] coordinates {
                ($p@1$,44.76)
                ($p@3$,39.91)
                ($p@5$,36.35)
            };%
            \addplot+[color=putred] coordinates {
                ($p@1$,43.31)
                ($p@3$,38.19)
                ($p@5$,34.31)
            };%
            \addplot+[color=putred, pattern=north east lines, pattern color=putred] coordinates {
                ($p@1$,44.72)
                ($p@3$,39.67)
                ($p@5$,35.90)
            };%
            \addplot+[color=putred, pattern=crosshatch dots, pattern color=putred] coordinates {
                ($p@1$,44.88)
                ($p@3$,39.84)
                ($p@5$,36.09)
            };%
\nextgroupplot[title=\amazonlarge]%
            \addplot+[color=putblue] coordinates {
                ($p@1$,46.09)
                ($p@3$,43.11)
                ($p@5$,40.98)
            };%
            \addplot+[color=putblue, pattern=north east lines, pattern color=putblue] coordinates {
                ($p@1$,47.02)
                ($p@3$,44.00)
                ($p@5$,41.82)
            };%
            \addplot+[color=putblue, pattern=crosshatch dots, pattern color=putblue] coordinates {
                ($p@1$,46.99)
                ($p@3$,43.98)
                ($p@5$,41.79)
            };%
            \addplot+[color=putred] coordinates {
                ($p@1$,46.23)
                ($p@3$,43.48)
                ($p@5$,41.41)
            };%
            \addplot+[color=putred, pattern=north east lines, pattern color=putred] coordinates {
                ($p@1$,47.41)
                ($p@3$,44.62)
                ($p@5$,42.51)
            };%
            \addplot+[color=putred, pattern=crosshatch dots, pattern color=putred] coordinates {
                ($p@1$,47.45)
                ($p@3$,44.67)
                ($p@5$,42.57)
            };%
                \coordinate (bot) at (rel axis cs:1,0);
\end{groupplot}
\path (top|-current bounding box.north) -- node[above]{\ref{grouplegend}} (bot|-current bounding box.north);
\end{tikzpicture}
\caption{\Precatk{} for ensembles of $T=$ 1, 3, and 5 trees trained using either logistic loss (log) or squared hinge loss (s.h.).}
\label{fig:ensembles}
\end{figure}

%% file: tables/table-fmeasure.tex
\begin{table}[H]
\centering
\footnotesize{
\begin{tabular}{l|r@{}lr@{}l|r@{}lr@{}l}
\toprule
& \multicolumn{4}{c|}{Hamming loss$^1$} 
& \multicolumn{4}{c}{micro-$F_1$ [\%]} 
\\
& \multicolumn{2}{c}{\textbf{thr=0.5}} & \multicolumn{2}{c|}{micro-\Algo{OFO}}
& \multicolumn{2}{c}{thr=0.5} & \multicolumn{2}{c}{\textbf{micro-\Algo{OFO}}}
\\
\midrule
\eurlex & \boldmath$4.01$ & \boldmath$\pm 0.00$ & $7.42$ & $\pm 0.59$ & \boldmath$46.81$ & \boldmath$\pm 0.08$ & $40.30$ & $\pm 2.08$ \\
\amazoncatsmall & \boldmath$2.79$ & \boldmath$\pm 0.00$ & $2.94$ & $\pm 0.01$ & $67.00$ & $\pm 0.02$ & \boldmath$68.55$ & \boldmath$\pm 0.12$ \\
\wikiten & \boldmath$17.43$ & \boldmath$\pm 0.01$ & $23.60$ & $\pm 0.30$ & $24.95$ & $\pm 0.10$ & \boldmath$32.27$ & \boldmath$\pm 0.18$ \\
\deliciouslarge & \boldmath$97.15$ & \boldmath$\pm 0.00$ & $554.37$ & $\pm 0.35$ & $0.84$ & $\pm 0.01$ & \boldmath$13.08$ & \boldmath$\pm 0.00$ \\
\wikilshtc & \boldmath$2.94$ & \boldmath$\pm 0.00$ & $4.14$ & $\pm 0.08$ & $31.76$ & $\pm 0.07$ & \boldmath$32.73$ & \boldmath$\pm 0.44$ \\
\wikipedia & \boldmath$4.12$ & \boldmath$\pm 0.01$ & $4.60$ & $\pm 0.00$ & $31.88$ & $\pm 0.12$ & \boldmath$38.76$ & \boldmath$\pm 0.00$ \\
\amazon & \boldmath$4.69$ & \boldmath$\pm 0.00$ & $6.22$ & $\pm 0.13$ & $18.87$ & $\pm 0.03$ & \boldmath$28.95$ & \boldmath$\pm 0.48$ \\
\amazonlarge & \boldmath$34.87$ & \boldmath$\pm 0.00$ & $73.32$ & $\pm 0.00$ & $12.79$ & $\pm 0.02$ & \boldmath$28.20$ & \boldmath$\pm 0.00$ \\

\bottomrule
\end{tabular}
\vskip4pt
{\footnotesize $^1$ multiplied by the number of labels to avoid presentation of very small numbers}}

\caption{The results of \Algo{PLT}s with threshold-based predictions for Hamming loss, 
micro $F_1$-measures. 
The name of the column with the theoretically optimal tuning strategy for a given metric is given in bold.}
\label{tab:plt-fmeasure}
\end{table}

%% file: tables/table-hsm-vs-plt.tex
\begin{table}[h]
\centering
\footnotesize{
\begin{tabular}{l|r@{}l|r@{}l|r|r|r|r}
\toprule
& \multicolumn{2}{c}{HSM-POL} & \multicolumn{2}{|c}{PLT} 
& \multicolumn{1}{|c}{\# of losses} & \multicolumn{1}{|c}{\# of ties} & \multicolumn{1}{|c}{\# of wins} 
& \multicolumn{1}{|c}{$p$-value} \\
\midrule
% standard error
\multilabel{} dependent       & $71.29$ & $\pm 0.98$ & $72.31$ & $\pm 0.94$ & 5 & 0 & 45 & 4.21e-09 \\ 
\multilabel{} independent     & $32.66$ & $\pm 0.08$ & $32.64$ & $\pm 0.08$ & 25 & 3 & 22 & 0.4799 \\
\multiclass{}                 & $61.23$ & $\pm 1.14$ & $61.23$ & $\pm 1.14$ & 0 & 50 & 0 & - \\
\bottomrule
\end{tabular}}
\caption{\Precat{1} of \Algo{PLT}s and \Algo{HSM-POL} on synthetic data. 
Reported are mean values over 50 runs along with standard errors, 
the number of wins, ties, and losses of \Algo{PLT}s, 
and $p$-value of the sign test.}
\label{tab:synthetic}
\end{table}

%% file: tables/table-hsm-vs-plt-benchmark.tex
\begin{table}[h]
\centering
\footnotesize{
\resizebox{\textwidth}{!}{
\tabcolsep=5pt
\begin{tabular}{l|r@{}lr@{}l|r@{}lr@{}l|r@{}lr@{}l}
\toprule
& \multicolumn{4}{c|}{$p@1$ [\%]}
& \multicolumn{4}{c|}{$r@1$ [\%]}
& \multicolumn{4}{c}{$r@5$ [\%]}  
\\
& \multicolumn{2}{c}{\Algo{HSM}} & \multicolumn{2}{c|}{\Algo{PLT}}
& \multicolumn{2}{c}{\Algo{HSM}} & \multicolumn{2}{c|}{\Algo{PLT}}
& \multicolumn{2}{c}{\Algo{HSM}} & \multicolumn{2}{c}{\Algo{PLT}}
\\
\midrule
\eurlex & $67.89$ & $\pm 0.26$ & \boldmath$80.51$ & \boldmath$\pm 0.16$ & $13.64$ & $\pm 0.06$ & \boldmath$16.20$ & \boldmath$\pm 0.04$ & $44.64$ & $\pm 0.45$ & \boldmath$51.71$ & \boldmath$\pm 0.67$ \\
\amazoncatsmall & $88.19$ & $\pm 0.16$ & \boldmath$93.04$ & \boldmath$\pm 0.02$ & $24.73$ & $\pm 0.06$ & \boldmath$26.37$ & \boldmath$\pm 0.01$ & $69.38$ & $\pm 0.09$ & \boldmath$74.64$ & \boldmath$\pm 0.02$ \\
\wikiten & $54.69$ & $\pm 1.00$ & \boldmath$85.36$ & \boldmath$\pm 0.09$ & $3.18$ & $\pm 0.07$ & \boldmath$5.06$ & \boldmath$\pm 0.01$ & $12.58$ & $\pm 0.18$ & \boldmath$18.26$ & \boldmath$\pm 0.02$ \\
\wikilshtc & $58.35$ & $\pm 0.04$ & \boldmath$61.96$ & \boldmath$\pm 0.03$ & $26.41$ & $\pm 0.02$ & \boldmath$27.41$ & \boldmath$\pm 0.01$ & $49.81$ & $\pm 0.01$ & \boldmath$52.96$ & \boldmath$\pm 0.03$ \\
\wikipedia & $60.48$ & $\pm 0.09$ & \boldmath$66.20$ & \boldmath$\pm 0.05$ & $20.16$ & $\pm 0.03$ & \boldmath$21.50$ & \boldmath$\pm 0.01$ & $43.17$ & $\pm 0.03$ & \boldmath$47.12$ & \boldmath$\pm 0.03$ \\
\amazon & $40.38$ & $\pm 0.04$ & \boldmath$43.54$ & \boldmath$\pm 0.01$ & $8.53$ & $\pm 0.01$ & \boldmath$9.01$ & \boldmath$\pm 0.01$ & $29.52$ & $\pm 0.04$ & \boldmath$32.83$ & \boldmath$\pm 0.02$ \\

\bottomrule
\end{tabular}}}

\caption{Precision$@1$ and recall$@k$ of hierarchical softmax with pick-one-label heuristic (\Algo{HSM}) and \Algo{PLT} on benchmark datasets.}
\label{tab:hsm-benchmark}
\end{table}

%% file: tables/table-online-vs-batch-features.tex
\begin{table}[ht]
\centering
\footnotesize{
\begin{tabular}{l|r|rrr|rr}
\toprule
%Features
& \multicolumn{1}{c|}{\#features}
& \multicolumn{3}{c|}{\#hashed features} 

% & \multicolumn{1}{c}{\#hashed features} 

& \multicolumn{2}{c}{RAM [GB]} 
\\
%RAM
& %\multicolumn{1}{|c}{unlimited} \\
& \multicolumn{1}{c}{64GB}    
& \multicolumn{1}{c}{128GB}
& \multicolumn{1}{c|}{256GB}   
& \multicolumn{1}{c}{\Algo{Robin Hood}}
& \multicolumn{1}{c}{dense vector} \\   
\midrule
\eurlex             & 5000    & $\ast$ & $\ast$ & $\ast$     & 0.6 & 0.2  \\
\amazoncatsmall     & 203882  & $\ast$ & $\ast$ & $\ast$ & 9   & 60    \\ 
\wikiten            & 101938  & $\ast$ & $\ast$ & $\ast$ & 18  & 70    \\
\deliciouslarge     & 782585  & 13000  & 26000  & 52000  & 240 & 3593  \\
\wikilshtc          & 1617899 & 8000   & 16000  & 32000  & 30  & 11754 \\
\wikipedia          & 2381304 & 5000   & 10000  & 20000  & 240 & 26670 \\
\amazon             & 135909  & 4000   & 8000   & 16000  & 36  & 2035  \\
\amazonlarge        & 337067  & 1000   & 2000   & 4000   & 280 & 21187 \\
\bottomrule
\end{tabular}
}
\caption{Number of features, hashed features for \Algo{OPLT} with complete tree policy and memory required to train \Algo{OPLT} with complete tree policy with \Algo{Robin Hood} hash maps and dense vectors.
With symbol `$\ast$' we denote data sets for which feature hashing is not needed to fit the available memory.}
\label{tab:oplt-random-features}
\end{table}

% \begin{table}[ht]
% \centering
% \footnotesize{
% \begin{tabular}{l|rr}
% \toprule
% & \multicolumn{2}{c}{RAM [GB]} \\
% Representation
% & \multicolumn{1}{c}{\Algo{Robin Hood}}
% & \multicolumn{1}{c}{dense vector} \\      
% \midrule
% \eurlex              & 0.6 & 0.2   \\
% \amazoncatsmall      & 9   & 60    \\
% \wikiten             & 18  & 70    \\
% \deliciouslarge      & 240 & 3593  \\
% \wikilshtc           & 30  & 11754 \\
% \wikipedia           & 240 & 26670 \\
% \amazon              & 36  & 2035  \\
% \amazonlarge         & 280 & 21187 \\
% \bottomrule
% \end{tabular}
% }
% \caption{Memory required to train \Algo{OPLT} with complete tree policy
% with \Algo{Robin Hood} hash maps and dense vectors.}
% \label{tab:oplt-robin-hood-dense}
% \end{table}

%% file: tables/table-online-vs-batch.tex
\begin{table}[H]
% \resizebox{\textwidth}{!}{
\centering
\footnotesize{
\begin{tabular}{l|r@{}lr@{}lr@{}lr@{}l|r@{}l}
\toprule
Algorithm
& \multicolumn{8}{c|}{OPLT} 
& \multicolumn{2}{c}{PLT} \\
%Tree & \multicolumn{6}{c|}{complete online} & \multicolumn{2}{c}{\makecell{complete}} \\
Representation & \multicolumn{6}{c|}{feature hashing} & \multicolumn{2}{c|}{\Algo{Robin Hood}} \\
RAM
& \multicolumn{2}{c}{64GB}
& \multicolumn{2}{c}{128GB}
& \multicolumn{2}{c|}{256GB}
& \multicolumn{2}{c|}{unlimited}
& \multicolumn{2}{c}{unlimited} \\

\midrule

\eurlex & \multicolumn{2}{c}{$\ast$} & \multicolumn{2}{c}{$\ast$} & \multicolumn{2}{c}{$\ast$} & $76.69$ & $\pm 0.21$ & \boldmath$76.82$ & \boldmath$\pm 0.35$ \\

\amazoncatsmall & \multicolumn{2}{c}{$\ast$} & \multicolumn{2}{c}{$\ast$} & \multicolumn{2}{c}{$\ast$} & \boldmath$91.35$ & $\pm 0.06$ & \boldmath$91.20$ & $\pm 0.05$ \\

\wikiten & \multicolumn{2}{c}{$\ast$} & \multicolumn{2}{c}{$\ast$} & \multicolumn{2}{c}{$\ast$} & \boldmath$84.41$ & \boldmath$\pm 0.19$ & $82.74$ & $\pm 0.14$\\

\deliciouslarge & $44.61$ & $\pm 0.03$ & $44.52$ & $\pm 0.04$ & $44.58$ & $\pm 0.06$ & $46.29$ & $\pm 0.05$ & \boldmath$47.81$ & \boldmath$\pm 0.03$  \\

\wikilshtc & $32.18$ & $\pm 0.04$ & $34.71$ & $\pm 0.01$ & $36.90$ & $\pm 0.00$ & \boldmath$44.42$ & \boldmath$\pm 0.03$ & $43.91$ & $\pm 0.02$ \\

\wikipedia & $25.39$ & $\pm 0.03$ & $29.04$ & $\pm 0.02$ & $32.90$ & $\pm 0.06$ & \boldmath$49.35$ & \boldmath$\pm 0.03$ & $47.25$ & $\pm 0.02$ \\

\amazon & $23.18$ & $\pm 0.05$ & $26.64$ & $\pm 0.02$ & $29.53$ & $\pm 0.05$ & \boldmath$37.02$ & \boldmath$\pm 0.01$ & $35.12$ & $\pm 0.03$  \\

\amazonlarge & $10.31$ & $\pm 0.01$ & $14.18$ & $\pm 0.01$ & $18.49$ & $\pm 0.03$ & $37.80$ & $\pm 0.01$ & \boldmath$38.05$ & \boldmath$\pm 0.02$ \\
\bottomrule
\end{tabular}
}
\caption{Performance of \Algo{OPLT} with different memory management strategies: 
feature hashing of 64GB, 128GB and 256GB, and \Algo{Robin Hood} hash maps. 
Results of a batch counterpart are given for reference. 
With symbol `$\ast$' we denote data sets where feature hashing is not needed to fit the available memory.}
\label{tab:oplt-random}
\end{table}

%% file: tables/table-best-plt-vs-sota.tex
\begin{center}

{\footnotesize \tabcolsep=5pt
\begin{longtable}[h]{ l|r@{}lr@{}lr@{}l|r@{}lr@{}lr@{}l }

\specialrule{0.94pt}{0.4ex}{0.65ex}
\multicolumn{1}{c|}{} 
& \multicolumn{2}{c}{$p@1$ [\%]} 
& \multicolumn{2}{c}{$p@3$ [\%]} 
& \multicolumn{2}{c|}{$p@5$ [\%]} 
& \multicolumn{2}{c}{$T_{\textrm{train}}$ [h]} 
& \multicolumn{2}{c}{$T/N_{\textrm{test}}$ [ms]} 
& \multicolumn{2}{c}{$M_{\textrm{size}}$ [GB]} \\

\specialrule{0.94pt}{0.4ex}{0.65ex}
\multicolumn{13}{c}{\eurlex} \\
\midrule
\Algo{FastXML} & $71.26$ & $\pm 0.19$ & $59.80$ & $\pm 0.12$ & $50.28$ & $\pm 0.02$ & $0.07$ & $\pm 0.00$ & $0.97$ & $\pm 0.15$ & $0.22$ & $\pm 0.00$ \\
\Algo{PfastreXML} & $70.21$ & $\pm 0.09$ & $59.26$ & $\pm 0.10$ & $50.59$ & $\pm 0.08$ & $0.08$ & $\pm 0.00$ & $1.30$ & $\pm 0.09$ & $0.26$ & $\pm 0.00$ \\
\Algo{PPDSparse} & \multicolumn{2}{c}{\boldmath$83.83$} & \multicolumn{2}{c}{\boldmath$70.72$} & \multicolumn{2}{c|}{\boldmath$59.21$} & \multicolumn{2}{c}{$\approx0.02$} & \multicolumn{2}{c}{$\approx0.70$} & \multicolumn{2}{c}{$0.07$} \\
\Algo{DiSMEC} & \multicolumn{2}{c}{$83.67$} & \multicolumn{2}{c}{$70.70$} & \multicolumn{2}{c|}{$59.14$} & \multicolumn{2}{c}{$\approx0.70$} & \multicolumn{2}{c}{$\approx4.60$} & \multicolumn{2}{c}{$0.04$} \\

\midrule
\Algo{Parabel}-T=3 & $81.80$ & $\pm 0.10$ & $68.67$ & $\pm 0.03$ & $57.45$ & $\pm 0.06$ & \boldmath$0.02$ & \boldmath$\pm 0.00$ & \boldmath$0.93$ & \boldmath$\pm 0.04$ & $0.03$ & $\pm 0.00$ \\
\Algo{nXC}-T=3 & $81.94$ & $\pm 0.24$ & $68.94$ & $\pm 0.07$ & $57.49$ & $\pm 0.14$ & $0.03$ & $\pm 0.00$ & $0.97$ & $\pm 0.06$ & \boldmath$0.02$ & \boldmath$\pm 0.00$ \\

\specialrule{0.94pt}{0.4ex}{0.65ex}
\multicolumn{13}{c}{\amazoncatsmall} \\
\midrule
\Algo{FastXML} & $93.03$ & $\pm 0.00$ & $78.22$ & $\pm 0.01$ & $63.38$ & $\pm 0.00$ & $5.53$ & $\pm 0.15$ & $1.06$ & $\pm 0.08$ & $18.35$ & $\pm 0.00$ \\
\Algo{PfastreXML} & $85.62$ & $\pm 0.01$ & $75.31$ & $\pm 0.00$ & $62.83$ & $\pm 0.01$ & $5.45$ & $\pm 0.12$ & \boldmath$0.99$ & \boldmath$\pm 0.06$ & $19.01$ & $\pm 0.00$ \\
\Algo{PPDSparse} & \multicolumn{2}{c}{$92.72$} & \multicolumn{2}{c}{$78.14$} & \multicolumn{2}{c|}{$63.41$} & \multicolumn{2}{c}{$\approx2.97$} & \multicolumn{2}{c}{$\approx1.20$} & \multicolumn{2}{c}{\boldmath$0.50$} \\
\Algo{DiSMEC} & \multicolumn{2}{c}{$92.72$} & \multicolumn{2}{c}{$78.11$} & \multicolumn{2}{c|}{$63.40$} & \multicolumn{2}{c}{$\approx138.60$} & \multicolumn{2}{c}{$\approx2.9$} & \multicolumn{2}{c}{$1.50$} \\

\midrule
\Algo{Parabel}-T=3 & $93.24$ & $\pm 0.01$ & $79.17$ & $\pm 0.00$ & $64.51$ & $\pm 0.00$ & \boldmath$0.64$ & \boldmath$\pm 0.03$ & $1.05$ & $\pm 0.04$ & $0.62$ & $\pm 0.00$ \\
\Algo{nXC}-T=3 & \boldmath$93.37$ & \boldmath$\pm 0.05$ & \boldmath$79.01$ & \boldmath$\pm 0.03$ & \boldmath$64.27$ & \boldmath$\pm 0.04$ & $2.30$ & $\pm 0.13$ & \boldmath$0.99$ & \boldmath$\pm 0.10$ & $1.01$ & $\pm 0.00$ \\

\ifjmlr

\else
\specialrule{0.94pt}{0.4ex}{0.65ex}
\pagebreak

\specialrule{0.94pt}{0.4ex}{0.65ex}
\multicolumn{1}{c|}{} 
& \multicolumn{2}{c}{$p@1$ [\%]} 
& \multicolumn{2}{c}{$p@3$ [\%]} 
& \multicolumn{2}{c|}{$p@5$ [\%]} 
& \multicolumn{2}{c}{$T_{\textrm{train}}$ [h]} 
& \multicolumn{2}{c}{$T/N_{\textrm{test}}$ [ms]} 
& \multicolumn{2}{c}{$M_{\textrm{size}}$ [GB]} \\

\fi

\specialrule{0.94pt}{0.4ex}{0.65ex}
\multicolumn{13}{c}{\wikiten} \\
\midrule
\Algo{FastXML} & $82.97$ & $\pm 0.02$ & $67.58$ & $\pm 0.07$ & $57.68$ & $\pm 0.03$ & $0.23$ & $\pm 0.01$ & $8.21$ & $\pm 0.52$ & $0.54$ & $\pm 0.00$ \\
\Algo{PfastreXML} & $75.58$ & $\pm 0.07$ & $64.38$ & $\pm 0.11$ & $57.25$ & $\pm 0.07$ & $0.23$ & $\pm 0.00$ & $10.40$ & $\pm 0.41$ & $1.13$ & $\pm 0.00$ \\
\Algo{PPDSparse} & \multicolumn{2}{c}{$73.80$} & \multicolumn{2}{c}{$60.90$} & \multicolumn{2}{c|}{$50.40$} & \multicolumn{2}{c}{$\approx1.20$} & \multicolumn{2}{c}{$\approx22.00$} & \multicolumn{2}{c}{$0.80$} \\
\Algo{DiSMEC} & \multicolumn{2}{c}{$85.20$} & \multicolumn{2}{c}{\boldmath$74.60$} & \multicolumn{2}{c|}{\boldmath$65.90$} & \multicolumn{2}{c}{$\approx26.80$} & \multicolumn{2}{c}{$\approx112.40$} & \multicolumn{2}{c}{$2.40$} \\

\midrule
\Algo{Parabel}-T=3 & $84.49$ & $\pm 0.05$ & $72.57$ & $\pm 0.04$ & $63.66$ & $\pm 0.10$ & \boldmath$0.20$ & \boldmath$\pm 0.00$ & \boldmath$2.67$ & \boldmath$\pm 0.06$ & \boldmath$0.18$ & \boldmath$\pm 0.00$ \\
\Algo{nXC}-T=3 & \boldmath$85.90$ & \boldmath$\pm 0.02$ & $74.45$ & $\pm 0.11$ & $64.84$ & $\pm 0.09$ & $0.39$ & $\pm 0.01$ & $11.76$ & $\pm 0.19$ & $2.16$ & $\pm 0.00$ \\

\specialrule{0.94pt}{0.4ex}{0.65ex}
\multicolumn{13}{c}{\deliciouslarge} \\
\midrule
\Algo{FastXML} & $43.17$ & $\pm 0.03$ & $38.70$ & $\pm 0.01$ & $36.22$ & $\pm 0.02$ & $3.86$ & $\pm 0.09$ & $12.27$ & $\pm 0.36$ & $6.95$ & $\pm 0.00$ \\
\Algo{PfastreXML} & $17.44$ & $\pm 0.02$ & $17.28$ & $\pm 0.01$ & $17.19$ & $\pm 0.01$ & $3.71$ & $\pm 0.02$ & $19.64$ & $\pm 0.35$ & $15.34$ & $\pm 0.00$ \\
\Algo{PPDSparse} & \multicolumn{2}{c}{$45.05$} & \multicolumn{2}{c}{$38.34$} & \multicolumn{2}{c|}{$34.90$} & \multicolumn{2}{c}{$\approx17.00$} & \multicolumn{2}{c}{$\approx64.00$} & \multicolumn{2}{c}{$3.40$} \\
\Algo{DiSMEC} & \multicolumn{2}{c}{$45.50$} & \multicolumn{2}{c}{$38.70$} & \multicolumn{2}{c|}{$35.50$} & \multicolumn{2}{c}{$\approx24000.00$} & \multicolumn{2}{c}{$\approx68.20$} & \multicolumn{2}{c}{$160.10$} \\

\midrule
\Algo{Parabel}-T=3 & $46.62$ & $\pm 0.02$ & $39.78$ & $\pm 0.04$ & $36.37$ & $\pm 0.04$ & $9.01$ & $\pm 0.20$ & \boldmath$2.61$ & \boldmath$\pm 0.03$ & $6.36$ & $\pm 0.00$ \\
\Algo{nXC}-T=3 & \boldmath$49.65$ & \boldmath$\pm 0.03$ & \boldmath$43.18$ & \boldmath$\pm 0.02$ & \boldmath$39.97$ & \boldmath$\pm 0.01$ & \boldmath$7.90$ & \boldmath$\pm 0.48$ & $31.10$ & $\pm 2.87$ & \boldmath$2.86$ & \boldmath$\pm 0.00$ \\

% \specialrule{0.94pt}{0.4ex}{0.65ex}
% \pagebreak

% \specialrule{0.94pt}{0.4ex}{0.65ex}
% \multicolumn{1}{c|}{} 
% & \multicolumn{2}{c}{$p@1$ [\%]} 
% & \multicolumn{2}{c}{$p@3$ [\%]} 
% & \multicolumn{2}{c|}{$p@5$ [\%]} 
% & \multicolumn{2}{c}{$T_{\textrm{train}}$ [h]} 
% & \multicolumn{2}{c}{$T/N_{\textrm{test}}$ [ms]} 
% & \multicolumn{2}{c}{$M_{\textrm{size}}$ [GB]} \\

\specialrule{0.94pt}{0.4ex}{0.65ex}\multicolumn{13}{c}{\wikilshtc} \\
\midrule
\Algo{FastXML} & $49.85$ & $\pm 0.00$ & $33.16$ & $\pm 0.01$ & $24.49$ & $\pm 0.01$ & $6.41$ & $\pm 0.13$ & $4.10$ & $\pm 0.04$ & $12.93$ & $\pm 0.00$ \\
\Algo{PfastreXML} & $58.50$ & $\pm 0.02$ & $37.69$ & $\pm 0.01$ & $27.57$ & $\pm 0.01$ & $6.25$ & $\pm 0.13$ & $4.00$ & $\pm 0.20$ & $14.20$ & $\pm 0.00$ \\
\Algo{PPDSparse} & \multicolumn{2}{c}{$64.13$} & \multicolumn{2}{c}{$42.10$} & \multicolumn{2}{c|}{$31.14$} & \multicolumn{2}{c}{$\approx16.00$} & \multicolumn{2}{c}{$\approx51.00$} & \multicolumn{2}{c}{$5.10$} \\
\Algo{DiSMEC} & \multicolumn{2}{c}{$64.94$} & \multicolumn{2}{c}{$42.71$} & \multicolumn{2}{c|}{$31.50$} & \multicolumn{2}{c}{$\approx2320.00$} & \multicolumn{2}{c}{$\approx340.00$} & \multicolumn{2}{c}{$3.80$} \\

\midrule
\Algo{Parabel}-T=3 & $64.95$ & $\pm 0.02$ & $43.21$ & $\pm 0.02$ & $32.01$ & $\pm 0.01$ & \boldmath$0.81$ & \boldmath$\pm 0.02$ & \boldmath$1.27$ & \boldmath$\pm 0.03$ & $3.10$ & $\pm 0.00$ \\
\Algo{nXC}-T=3 & \boldmath$65.57$ & \boldmath$\pm 0.10$ & \boldmath$43.64$ & \boldmath$\pm 0.11$ & \boldmath$32.33$ & \boldmath$\pm 0.11$ & $7.10$ & $\pm 0.13$ & $1.70$ & $\pm 0.13$ & \boldmath$2.68$ & \boldmath$\pm 0.00$ \\

\ifjmlr
\specialrule{0.94pt}{0.4ex}{0.65ex}
\pagebreak

\specialrule{0.94pt}{0.4ex}{0.65ex}
\multicolumn{1}{c|}{} 
& \multicolumn{2}{c}{$p@1$ [\%]} 
& \multicolumn{2}{c}{$p@3$ [\%]} 
& \multicolumn{2}{c|}{$p@5$ [\%]} 
& \multicolumn{2}{c}{$T_{\textrm{train}}$ [h]} 
& \multicolumn{2}{c}{$T/N_{\textrm{test}}$ [ms]} 
& \multicolumn{2}{c}{$M_{\textrm{size}}$ [GB]} \\

\fi

\specialrule{0.94pt}{0.4ex}{0.65ex}
\multicolumn{13}{c}{\wikipedia} \\
\midrule
\Algo{FastXML} & $49.32$ & $\pm 0.03$ & $33.48$ & $\pm 0.03$ & $25.84$ & $\pm 0.01$ & $51.48$ & $\pm 0.65$ & $15.35$ & $\pm 0.56$ & $59.69$ & $\pm 0.01$ \\
\Algo{PfastreXML} & $59.58$ & $\pm 0.02$ & $40.26$ & $\pm 0.01$ & $30.73$ & $\pm 0.01$ & $51.07$ & $\pm 0.92$ & $15.24$ & $\pm 0.24$ & $69.33$ & $\pm 0.01$ \\
\Algo{PPDSparse} & \multicolumn{2}{c}{$70.16$} & \multicolumn{2}{c}{$50.57$} & \multicolumn{2}{c|}{$39.66$} & \multicolumn{2}{c}{$\approx26.00$} & \multicolumn{2}{c}{$\approx130.00$} & \multicolumn{2}{c}{$4.00$} \\
\Algo{DiSMEC} & \multicolumn{2}{c}{\boldmath$70.20$} & \multicolumn{2}{c}{\boldmath$50.60$} & \multicolumn{2}{c|}{\boldmath$39.70$} & \multicolumn{2}{c}{$\approx26800.00$} & \multicolumn{2}{c}{$\approx1200.00$} & \multicolumn{2}{c}{$14.80$} \\

\midrule
\Algo{Parabel}-T=3 & $68.66$ & $\pm 0.06$ & $49.48$ & $\pm 0.05$ & $38.60$ & $\pm 0.04$ & \boldmath$7.33$ & \boldmath$\pm 0.12$ & \boldmath$3.44$ & \boldmath$\pm 0.13$ & $5.69$ & $\pm 0.00$ \\
\Algo{nXC}-T=3 & $69.24$ & $\pm 0.20$ & $49.82$ & $\pm 0.16$ & $38.81$ & $\pm 0.14$ & $41.11$ & $\pm 1.34$ & $5.53$ & $\pm 0.10$ & \boldmath$4.68$ & \boldmath$\pm 0.01$ \\

\specialrule{0.94pt}{0.4ex}{0.65ex}
\multicolumn{13}{c}{\amazon} \\
\midrule
\Algo{FastXML} & $36.90$ & $\pm 0.02$ & $33.22$ & $\pm 0.01$ & $30.44$ & $\pm 0.01$ & $2.80$ & $\pm 0.03$ & $8.57$ & $\pm 0.20$ & $9.54$ & $\pm 0.00$ \\
\Algo{PfastreXML} & $36.97$ & $\pm 0.02$ & $34.18$ & $\pm 0.01$ & $32.05$ & $\pm 0.01$ & $3.01$ & $\pm 0.03$ & $9.96$ & $\pm 0.14$ & $10.98$ & $\pm 0.00$ \\
\Algo{PPDSparse} & \multicolumn{2}{c}{$45.32$} & \multicolumn{2}{c}{$40.37$} & \multicolumn{2}{c|}{$36.92$} & \multicolumn{2}{c}{$\approx2.00$} & \multicolumn{2}{c}{$\approx90.00$} & \multicolumn{2}{c}{$6.00$} \\
\Algo{DiSMEC} & \multicolumn{2}{c}{\boldmath$45.37$} & \multicolumn{2}{c}{\boldmath$40.40$} & \multicolumn{2}{c|}{\boldmath$36.96$} & \multicolumn{2}{c}{$\approx1830.00$} & \multicolumn{2}{c}{$\approx380.00$} & \multicolumn{2}{c}{$3.80$} \\

\midrule
\Algo{Parabel}-T=3 & $44.70$ & $\pm 0.04$ & $39.66$ & $\pm 0.04$ & $35.85$ & $\pm 0.04$ & \boldmath$0.39$ & \boldmath$\pm 0.00$ & \boldmath$1.57$ & \boldmath$\pm 0.05$ & $1.95$ & $\pm 0.00$ \\
\Algo{nXC}-T=3 & $45.10$ & $\pm 0.11$ & $40.00$ & $\pm 0.12$ & $36.22$ & $\pm 0.13$ & $2.17$ & $\pm 0.10$ & $1.84$ & $\pm 0.42$ & \boldmath$1.66$ & \boldmath$\pm 0.00$ \\

\specialrule{0.94pt}{0.4ex}{0.65ex}
\multicolumn{13}{c}{\amazonlarge} \\
\midrule
\Algo{FastXML} & $45.26$ & $\pm 0.01$ & $41.96$ & $\pm 0.00$ & $39.80$ & $\pm 0.01$ & $18.19$ & $\pm 1.01$ & $68.77$ & $\pm 4.16$ & $30.70$ & $\pm 0.00$ \\
\Algo{PfastreXML} & $32.62$ & $\pm 0.01$ & $32.67$ & $\pm 0.01$ & $32.35$ & $\pm 0.01$ & $19.07$ & $\pm 0.92$ & $78.83$ & $\pm 3.93$ & $41.88$ & $\pm 0.00$ \\
\Algo{PPDSparse} & \multicolumn{2}{c}{-} & \multicolumn{2}{c}{-} & \multicolumn{2}{c|}{-} & \multicolumn{2}{c}{-} & \multicolumn{2}{c}{-} & \multicolumn{2}{c}{-} \\
\Algo{DiSMEC} & \multicolumn{2}{c}{$47.77$} & \multicolumn{2}{c}{$44.96$} & \multicolumn{2}{c|}{$42.80$} & \multicolumn{2}{c}{$\approx18800.00$} & \multicolumn{2}{c}{$\approx2050.00$} & \multicolumn{2}{c}{$39.70$} \\

\midrule
\Algo{Parabel}-T=3 & $47.52$ & $\pm 0.01$ & $44.69$ & $\pm 0.01$ & $42.57$ & $\pm 0.00$ & \boldmath$5.20$ & \boldmath$\pm 0.01$ & \boldmath$1.53$ & \boldmath$\pm 0.02$ & $31.43$ & $\pm 0.00$ \\
\Algo{nXC}-T=3 & \boldmath$47.83$ & \boldmath$\pm 0.09$ & \boldmath$45.08$ & \boldmath$\pm 0.09$ & \boldmath$42.98$ & \boldmath$\pm 0.09$ & $25.43$ & $\pm 1.02$ & $4.93$ & $\pm 0.60$ & \boldmath$28.08$ & \boldmath$\pm 0.00$ \\

\specialrule{0.94pt}{0.4ex}{0.65ex}
\caption{\Algo{PLT}s compared to state-of-the-art algorithms.}
\label{tab:plt-vs-sota}
\end{longtable}
}

\end{center}

%% file: 09-summary.tex
\section{Summary}
\label{sec:summary}

We presented and investigated probabilistic label trees, 
a computationally efficient and statistically well-justified model 
for solving extreme \multilabel{} classification problems. 
%We performed an  of their statistical properties.
%
The in-depth analysis shows that \Algo{PLT}s can scale logarithmically with the number of labels
and are suitable for optimizing a wide spectrum of performance metrics 
commonly used in extreme \multilabel{} classification. 
%The analysis of the computational complexity shows that \Algo{PLT}s are an efficient solution to this kind of classification problems.
We also considered a fully online algorithm which incrementally trains both 
node classifiers and tree structure in a streaming setting, 
without any prior knowledge of training examples and the set of labels.
The presented discussion on existing implementations of a \Algo{PLT} model 
systematizes the knowledge about different design choices
and allows for a better understanding of their performance. 
We introduced \Algo{napkinXC}, a new modular implementation of \Algo{PLT}s.
It supports batch and incremental training, the fully online setting,
sparse and dense representation of features, 
as well as prediction suited for various performance metrics.
Thanks to this it can be easily tailored for a wide spectrum of applications. 
In the comprehensive experimental study, 
we showed that \Algo{PLT}s are indeed a state-of-the-art method 
performing on par with the \Algo{1-vs-All} approaches, 
but being at the same time orders of magnitude faster.
The existing implementations of the \Algo{PLT} model along with the empirical results
indicate that this approach is currently the most successful for extreme  \multilabel{} classification.

We hope that our work will contribute to development of label tree methods, 
by creating a common basis for this approach. 
Nevertheless, there exist many open problems related to probabilistic label trees. 
The tree structure learning is one of them. 
The $k$-means trees have made a significant contribution, 
but certainly they are not the ultimate solution to this problem.
The ideal tree structure should also take training and prediction costs into account,
as well as the long-tail labels.
Another challenge are new tree building policies for online probabilistic label trees.
An open problem is to make \Algo{PLT}s suitable for such performance metrics as 
recall$@k$ and NDCG$@k$. 
Another interesting research direction is the use of \Algo{PLT}s as an efficient index structure 
for negative sampling in the \Algo{1-vs-All} approaches.

%% file: 10-appendix.tex
\appendix
%%%%%%%%%%%%%%%%%%%%%%%%%%%%%%%%%%%%%%%%%%%%%%%%%%%%%%%%%%%%%%%%%%%%%%%%%%%%%%%%
%%%%%%%%%%%%%%%%%%%%%%%%%%%%%%%%%%%%%%%%%%%%%%%%%%%%%%%%%%%%%%%%%%%%%%%%%%%%%%%%
\renewcommand{\thesection}{\Alph{section}}
\renewcommand{\thesubsection}{\Alph{section}.\arabic{subsection}}
%%%%%%%%%%%%%%%%%%%%%%%%%%%%%%%%%%%%%%%%%%%%%%%%%%%%%%%%%%%%%%%%%%%%%%%%%%%%%%%%
%%%%%%%%%%%%%%%%%%%%%%%%%%%%%%%%%%%%%%%%%%%%%%%%%%%%%%%%%%%%%%%%%%%%%%%%%%%%%%%%

\section{Proofs of the results from Section~\ref{subsec:complexity_analysis}}
\label{app:complexity_analysis}

We present proofs of Proposition~\ref{prop:cost_upperbound} and Theorem~\ref{thm:comp_tresh_pred}
concerning the computational costs of \Algo{PLT}s.
The first proof has been originally published in~\citep{Busa-Fekete_et_al_2019}.
The second is based on a proof of a more general result concerning actual label probabilities 
from the same paper.
%In this section we include the proofs related to the cost of training $c(T, \by)$ and the cost of prediction $c_{\tau}(T, \bx)$ of a \Algo{PLT}.}

\costupperbound* 
\begin{proof}
First notice that a training example is always used in the root node, either as a positive example $(\bx, 1)$, if $\|\by\|_1 > 0$, 
or as a negative example $(\bx, 0)$, if  $\|\by\|_1 = 0$. Therefore the cost is bounded by 1.
% If $\|\by\|_1 > 0$, the training example is also used as a positive example in each node on the path from the root to a corresponding leaf, for each $j$ for which $y_j = 1$ in $\by$. 
If $\|\by\|_1 > 0$, the training example is also used as a positive example in all the nodes on paths from the root to leaves corresponding to labels $j$ for which $y_j = 1$ in $\by$. 
As the root has been already counted, we have at most $\depth_T = \max_v \lenpath_v - 1$ such nodes for each positive label in $\by$. 
Moreover, the training example is used as a negative example in all siblings of the nodes on the paths determined above, 
unless it is already a positive example in the sibling node. % because of any other label assigned to this example. 
The highest degree of a node in the tree is $\deg_T$.
Taking the above into account, the cost $c(T, \by)$ is bounded from above by $1 + \|\by\|_1 \cdot \depth_T \cdot \deg_T$. 
The bound is tight, for example, if $\|\by\|_1 = 1$ and $T$ is a perfect $\deg_T$-ary tree 
(all non-leaf nodes have an equal degree and the paths to the root from all leaves are of the same length).
\end{proof}

\compthreshpred*
\begin{proof}
The proof is similar to the one of Theorem~6.1 in~\citep{Busa-Fekete_et_al_2019}. 
{As stated before the theorem, we assume that the estimates are properly normalized, 
that is, they satisfy:}
\begin{equation}
\heta_v(\bx) \le \min \left \{1, \sum_{v' \in \childs{v}} \heta_{v'}(\bx)  \right \} \,,
\label{eqn:node_heta_prob_ub}
\end{equation}
and
\begin{equation}
\max \left \{\heta_{v'}(\bx), v' \in \childs{v} \right \} \le \heta_v(\bx)  \,.
\label{eqn:node_heta_prob_lb}
\end{equation}
Consider the subtree $T'$ of $T$, which consists of all nodes $\node \in \nodes_T$ 
for which $\heta_v(\bx) \ge \tau$. 
If there are no such nodes, from the pseudocode of Algorithm~\ref{alg:tbs_prediction}, 
we see that only the root classifier is called. 
The upperbound~(\ref{eq:comp_tresh_pred}) in this case obviously holds. 
However, it might not be tight 
as $\heta_v(\bx) < \tau$ does not imply $\hat P \le \tau$ because of (\ref{eqn:node_heta_prob_ub}). 

If $T'$ has at least one node, Algorithm~\ref{alg:tbs_prediction} visits each node of $T'$ 
(calls a corresponding classifier and add the node to a stack), 
since for each parent node we have (\ref{eqn:node_heta_prob_lb}). 
Moreover, Algorithm~\ref{alg:tbs_prediction} visits all children of nodes $T'$ 
(some of them are already in $T'$). 
Let the subtree $T''$ consist of all nodes of $T'$ and their child nodes. 
Certainly $T' \subseteq T'' \subseteq T$.
To prove the theorem we count first the number of nodes in $T'$ and then the number of nodes in $T''$, 
which gives as the final result. 

If the number of nodes in $T'$ is greater than or equal to 1, then certainly $r_{T}$ is in $T'$. 
Let us consider next the number of leaves of $T'$. 
Observe that $\sum_{v \in L_{T'}} \heta_v(\bx) \le \hat P$. 
This is because $\sum_{v \in L_{T'}} \heta_v(\bx) \le \sum_{v \in L_{T}} \heta_v(\bx) \le \hat P$, 
that is, $v \in L_{T'}$ might be an internal node in $T$ 
and its $\heta_v(\bx)$ is at most the sum of probability estimates of the leaves 
underneath $v$ according to (\ref{eqn:node_heta_prob_ub}). 
From this we get the following upper bound on the number of leaves in $T'$:
\begin{equation}
|L_{T'}| \le \lfloor \hat P/\tau \rfloor \,.
\label{eqn:num_leaves_ub}
\end{equation}
Since the degree of internal nodes in $T'$ might be $1$, 
to upperbound the number of all nodes in $T'$ 
we count the number of nodes on all paths from leaves to the root, 
but counting the root node only once:
$$
|V_{T'}| \le 1 + \sum_{v \in L_{T'}} (\lenpath_v - 1) \,.
$$
Next, notice that for each $v \in T'$ its all siblings are in $T''$ unless $v$ is the root node. 
This is because if non-root node $v$ is in $T'$ then its parent is also in $T'$ according to (\ref{eqn:node_heta_prob_lb}) 
and $T''$ contains all child nodes of nodes in $T'$.  
The rest of nodes in $T''$ are the child nodes of leaves of $T'$, 
unless a leaf of $T'$ is also a leaf of $T$. 
Therefore, we have 
$$
|V_{T''}| \le 1 +  \sum_{v \in L_{T'}} \deg_T (\lenpath_v - 1) + \sum_{v \in L_{T'}} \deg_T \assert{v \not \in L_{T}} \,,
$$
with $\deg_T$ being the highest possible degree of a node.
Since (\ref{eqn:num_leaves_ub}) and
$$
\lenpath_v - 1 + \assert{v \not \in L_{T}} \le \depth_T \,,
$$ 
that is, the longest path cannot be longer than the depth of the tree plus 1, 
we finally get:
$$
|V_{T''}| \le  1 +  \lfloor \hat P/\tau \rfloor \cdot \depth_T \cdot \deg_T \,. 
$$
This ends the proof as the number of nodes in $T''$ is equivalent 
to the number of calls to the node classifiers, that is, $c_{\tau}(T,\bx)$.
 \end{proof}

\section{Proofs of the results from Section~\ref{subsec:analysis-marginal}}
\label{app:analysis-marginal}

We prove here Theorem~\ref{thm:total_estimation_bound}, 
the main result of~Section~\ref{subsec:analysis-marginal}.
To this end we first show two additional results. 
The first lemma concerns expectation
of $\eta_{\pa{v'}}(\bx) \left | \eta(\bx,v') - \heta(\bx, v') \right |$, 
the weighted $L_1$ error in node $v$ used in upper bounds 
from Lemma~\ref{lma:node_estimation_regret} and Corollary~\ref{cor:marginal-conditional-rec}.
We express this expectation by the expected $L_1$ error in node $v$ 
multiplied by $\prob(z_{\pa{\bx}} = 1)$. 
Based on this result, we proof the second lemma in which 
we bound the expected $L_1$-estimation error of label $j$ 
by a weighted sum of expected $L_1$ errors on $\Path{\leafnode_j}$. 

%\todo{
%Before we prove the Theorem~\ref{thm:total_estimation_bound}, we present two additional results used in the later proofs. The first result transforms the expected $L_1$ error of a node $v$ multiplied by the intermediate product $\eta_{\pa{v}}(\bx)$ to the expected $L_1$ error of the classifier in node $v$ scaled by the $\prob(z_{\pa{\bx}} = 1)$. The expected $L_1$ error of node $v$ is expressed according to a conditional distribution $\prob(\bx | z_{\pa{v}} = 1)$.}
%
\begin{lemma}
\label{lma:node_estimation_equality}
For any tree $T$ and distribution $\prob(\bx,\by)$, the following holds for $v \in V_T$:
\begin{equation*}
 \mathbb{E}_{\bx \sim \prob(\bx)} \left [ \eta_{\pa{v}}(\bx) \left | \eta(\bx,v) - \heta(\bx, v)  \right | \right ] = \prob(z_{\pa{v}} = 1) \mathbb{E}_{\bx \sim \prob(\bx \vert z_{\pa{v}} =1)} \left [ \left | \eta(\bx,v) - \heta(\bx, v) \right | \right ] \,,
% \label{eqn:node_estimation_equality}
\end{equation*}
where for the root node $\prob(z_{\pa{r_T}} = 1) = 1$.
\end{lemma}
\begin{proof}
By using the definition of expectation, replacing $\eta_{\pa{v}}(\bx)$ by its definition, applying Bayes' theorem, and rearranging terms, we obtain: 
\begin{eqnarray*}
\mathbb{E}_{\bx \sim \prob(\bx)}\!\! \left [ \eta_{\pa{v}}(\bx) \left | \eta(\bx,v) - \heta(\bx, v)  \right | \right ] 
& \!\!\!\!\!=\!\!\!\! & \int \!\! \prob(\bx) \eta_{\pa{v}}(\bx) \left | \eta(\bx,v) - \heta(\bx, v)  \right | d\bx \\
& \!\!\!\!=\!\!\!\! & \int \!\! \prob(\bx) \prob(z_{\pa{v}}=1 \vert \bx) \left | \eta(\bx,v) - \heta(\bx, v)  \right | d\bx \\
& \!\!\!\!=\!\!\!\! & \prob(z_{\pa{v}}=1) \!\! \int \!\! \prob(\bx \vert z_{\pa{v}}=1) \left | \eta(\bx,v) - \heta(\bx, v)  \right | d\bx \,.
\end{eqnarray*}
Since 
$$
\int \prob(\bx \vert z_{\pa{v}}=1) \left | \eta(\bx,v) - \heta(\bx, v)  \right | d\bx
$$
is nothing else than the expected $L_1$ estimation error in node $\node$, denoted by 
$$
\mathbb{E}_{\bx \sim \prob(\bx \vert z_{\pa{v}}=1)} \left [ \left | \eta(\bx,v) - \heta(\bx, v)  \right | \right ] \,,
$$
we obtain the final result.
\end{proof}
\begin{lemma}
\label{lma:label_estimation_bound}
For any tree $T$ and distribution $\prob(\bx,\by)$ the following holds for $j \in \labels$:
\begin{equation*}
\mathbb{E}_{\bx \sim \prob(\bx)} \left [ \left | \eta_j(\bx) - \heta_j(\bx) \right | \right ]  \leq 
\!\!\! \sum_{v \in \Path{\leafnode_j}} \!\!\! \prob(z_{\pa{v}} = 1) \mathbb{E}_{\bx \sim \prob(\bx \vert z_{\pa{v}} =1)} \left [ \left | \eta(\bx,v) - \heta(\bx, v) \right | \right ] \,,
% \label{eqn:label_estimation_bound}
\end{equation*}
where for the root node $\prob(z_{\pa{r_T}} = 1) = 1$.
\end{lemma}
\begin{proof}
Take expectation of both hand sides of (\ref{eqn:estimation_bound}) and use linearity of expectation for the right hand side:
\begin{eqnarray*}
& & \mathbb{E}_{\bx \sim \prob(\bx)} \left [ \left | \eta_j(\bx) - \heta_j(\bx) \right | \right ]  \\
& & \quad\quad \leq \mathbb{E}_{\bx \sim \prob(\bx)} \left [ \sum_{v \in \Path{\leafnode_j}} \eta_{\pa{v}}(\bx) \left | \eta(\bx,v) - \heta(\bx, v)  \right | \right ] \\
& & \quad\quad = \sum_{v \in \Path{\leafnode_j}} \mathbb{E}_{\bx \sim \prob(\bx)} \left [ \eta_{\pa{v}}(\bx) \left | \eta(\bx,v) - \heta(\bx, v)  \right | \right ]
\end{eqnarray*}
The rest follows from Lemma~\ref{lma:node_estimation_equality}.
\end{proof}

Using the above results, we finally prove Theorem~\ref{thm:total_estimation_bound}.
It bounds the expectation of the $L_1$-estimation error averaged over all labels
by the expected $L_1$-estimation errors of node classifiers.
The expectation is defined over the entire distribution $\Pr(\bx)$.
We present the result in a general form of a weighted average
as such form is used later in proofs for the generalized performance metrics.

%\todo{With the above results, we finally prove Theorem~\ref{thm:total_estimation_bound}. Recall that this theorem bounds the expected $L_1$ estimation error, weighted with weights $W_j$, $j \in \{1, \ldots, m\}$, by a weighted sum of expected $L_1$ errors of node classifiers divided by the number of labels $m$. As in the previous lemmas, the node $L_1$ errors are scaled by the $\prob(z_{\pa{v}} = 1)$. The main difference between this theorem and previous lemmas is the introduction of the $W$-weights to express the expected error for all labels, not a single label.}

\thmtotalestimationbound*

\begin{proof}
From Lemma~\ref{lma:label_estimation_bound} we obtain:
\begin{eqnarray*}
& & \frac{1}{m} \sum_{j=1}^m  W_j \mathbb{E}_{\bx \sim \prob(\bx)} \left [ \left | \eta_j(\bx) - \heta_j(\bx) \right | \right ] \leq \\ 
& & \quad \quad \frac{1}{m} \sum_{j=1}^m W_j \!\!\! \sum_{v \in \Path{\leafnode_j}} \!\!\! \prob(z_{\pa{v}} = 1) \mathbb{E}_{\bx \sim \prob(\bx \vert z_{\pa{v}} =1)} \left [ \left | \eta(\bx,v) - \heta(\bx, v) \right | \right ] \,.
\end{eqnarray*}
The RHS can be further transformed to:
\begin{eqnarray*}
& & \quad \quad \frac{1}{m} \sum_{j=1}^m \sum_{v \in \Path{\leafnode_j}} \!\!\! W_j \prob(z_{\pa{v}} = 1) \mathbb{E}_{\bx \sim \prob(\bx \vert z_{\pa{v}} =1)} \left [ \left | \eta(\bx,v) - \heta(\bx, v) \right | \right ] = \\
& & \quad \quad \frac{1}{m} \sum_{\node \in \nodes} \prob(z_{\pa{v}} = 1) \mathbb{E}_{\bx \sim \prob(\bx \vert z_{\pa{v}} =1)} \left [ \left | \eta(\bx,v) - \heta(\bx, v) \right | \right ] \sum_{j \in \leaves_\node} W_j \,.
\end{eqnarray*}
where the last equation follows from the fact that each $\node \in \nodes_T$ appears in the double sum $|L_v|$ times (where $|L_v|$ is the number of leaves in a subtree rooted in $\node$; in other words, this is the number of paths from leaves to the root that contain node $\node$). Changing the double sum to sum over all nodes and multiplying the expected $L_1$ estimation error for $\node$ by $\sum_{j \in \leaves_\node} W_j$ gives the final result.
\end{proof}

\section{Proofs of the results from Section~\ref{subsec:strongly_proper_composite_losses}}
\label{app:strongly_proper_composite_losses}

We present the proof of Theorem~\ref{thm:total_regret_bound}.
It expresses the bound from Theorem~\ref{thm:total_estimation_bound}
in terms of node regrets of a strongly proper composite loss function.

%\todo{Theorem~\ref{thm:total_regret_bound} connects the weighted sum of expected $L_1$ errors of estimation of probabilities of labels with the regret of node classifiers with respect to a strongly proper composite loss function. The proof uses the definition of $\lambda$-strongly proper composite loss $\ell_c$ given in (\ref{eqn:regret_bound_for_spc_losses}), and Lemma~\ref{lma:node_estimation_equality}. Then transforms the result to obtain a bound using the {$\ell_c$-regret} of the expected $L_1$ estimation error of node, expressed as before for $\bx \sim \prob(\bx | z_{\pa{v} = 1})$, scaled by $\prob(z_{\pa{v}} = 1)$. Finally, it uses the obtained bound and Theorem~\ref{thm:total_estimation_bound} to get a result for all labels.}

\thmtotalregretbound*

\begin{proof}
%The results follow from Lemma~\ref{lma:total_estimation_bound} and (\ref{eqn:regret_bound_for_spc_losses}). 
As $\psi$ is an invertible function satisfying $f(\bx) = \psi(\prob(y = 1\given \bx))$, we can assume that $\heta(\bx, \node) = \psi^{-1}(f_v(\bx))$. We then obtain from~(\ref{eqn:regret_bound_for_spc_losses}):

\begin{equation}
\eta_{\pa{\node}}(\bx) \left | \eta(\bx, \node)  - \heta(\bx, \node) \right | \le \eta_{\pa{\node}}(\bx) \sqrt{ \frac{2}{\lambda}} \sqrt{\reg_{\ell_c}(f_v \given \bx)} \,,
\label{eqn:node_regret_bound_for_spc_losses}
\end{equation}
for any $\node \in \nodes_T$.
We take the expectation with respect to $\prob(\bx)$ of (\ref{eqn:node_regret_bound_for_spc_losses}). Based on Lemma~\ref{lma:node_estimation_equality},
given in Appendix~\ref{app:analysis-marginal},
the left hand side is equal to: 
$$
 \prob(z_{\pa{v}} = 1) \mathbb{E}_{\bx \sim \prob(\bx \vert z_{\pa{v}} =1)} \left [ \left | \eta(\bx,v) - \heta(\bx, v) \right | \right ] \,.
$$
For the left hand side we obtain the following upper bound: 
\begin{eqnarray*}
\mathbb{E}_{\bx \sim \prob(\bx)} \left [ \eta_{\pa{\node}}(\bx) \sqrt{ \frac{2}{\lambda}} \sqrt{\reg_{\ell_c}(f_v \given \bx)} \right ] 
& = & \sqrt{ \frac{2}{\lambda}} \mathbb{E}_{\bx \sim \prob(\bx)} \left [ \prob(z_{\pa{v}}=1 \vert \bx) \sqrt{\reg_{\ell_c}(f_v \given \bx)} \right ] \\
& = & \sqrt{ \frac{2}{\lambda}} \mathbb{E}_{\bx \sim \prob(\bx)} \left [ \sqrt{\prob(z_{\pa{v}}=1 \vert \bx)^2 \reg_{\ell_c}(f_v \given \bx)} \right ] \\
& \le & \sqrt{ \frac{2}{\lambda}} \mathbb{E}_{\bx \sim \prob(\bx)} \left [ \sqrt{ \prob(z_{\pa{v}}=1 \vert \bx) \reg_{\ell_c}(f_v \given \bx)} \right ] \,.
\end{eqnarray*}
Using Jensen's inequality we further get:
$$
\sqrt{ \frac{2}{\lambda}} \mathbb{E}_{\bx \sim \prob(\bx)} \left [ \sqrt{ \prob(z_{\pa{v}}=1 \vert \bx) \reg_{\ell_c}(f_v \given \bx)} \right ] 
 \le \sqrt{\frac{2}{\lambda} \mathbb{E}_{\bx \sim \prob(\bx)} \left [ \prob(z_{\pa{v}}=1 \vert \bx) \reg_{\ell_c}(f_v \given \bx) \right ] }
$$
The next step is similarly to the proof of Lemma~\ref{lma:node_estimation_regret}. 
We use first the definition of expectation, then Bayes' theorem, and finally we rearrange the terms:
\begin{eqnarray*}
\sqrt{ \frac{2}{\lambda} \mathbb{E}_{\bx \sim \prob(\bx)} \left [ \prob(z_{\pa{v}}\!=\!1 \vert \bx) \reg_{\ell_c}(f_v \given \bx) \right ] } 
& \!\!\!\! = \!\!\!\! &  \sqrt{ \frac{2}{\lambda} \! \int \!\! \prob(\bx) \prob(z_{\pa{v}}\!=\!1 \vert \bx) \reg_{\ell_c}(f_v \given \bx)  d\bx }  \\
& \!\!\!\! = \!\!\!\! &  \sqrt{ \frac{2}{\lambda} \! \int \!\! \prob(z_{\pa{v}}=1) \prob(\bx \vert z_{\pa{v}}\!=\!1) \reg_{\ell_c}(f_v \given \bx) d\bx }  \\
& \!\!\!\! = \!\!\!\! &  \sqrt{ \frac{2}{\lambda} \prob(z_{\pa{v}}=1) \!\! \int \!\! \prob(\bx \vert z_{\pa{v}}\!=\!1) \reg_{\ell_c}(f_v \given \bx) d\bx }  \,. 
\end{eqnarray*}
\noindent
Notice that:
$$
\int \prob(\bx \vert z_{\pa{v}}=1) \reg_\ell(f_v \given \bx) d\bx = 
\mathbb{E}_{\bx \sim \prob(\bx \vert z_{\pa{v}=1})} \left [  \reg_{\ell_c}(f_v \given \bx) \right ]
$$
\noindent
This is the expected regret of $f_\node$ taken over $\prob(\bx, z_\node \given z_{\pa{\node}} = 1)$, denoted by $\reg_{\ell_c}(f_v)$. 
We thus obtain the following by taking the expectation of (\ref{eqn:node_regret_bound_for_spc_losses}):
$$
\prob(z_{\pa{v}} = 1) \mathbb{E}_{\bx \sim \prob(\bx \vert z_{\pa{v}} =1)} \left [ | \eta(\bx,v') - \heta(\bx, v') | \right ] \le 
\sqrt{ \frac{2}{\lambda}} \sqrt{ \prob(z_{\pa{v}}=1)\reg_{\ell_c}(f_v) } 
$$
By using the above in Eq.~(\ref{eqn:total_estimation_bound}) from Theorem~\ref{thm:total_estimation_bound}, we obtain the final result.
\end{proof}

\section{Proofs of the results from Section~\ref{subsec:analyis-generalized}}
\label{app:analysis-generalized}

This appendix contains proofs of Theorems~\ref{thm:psi_macro_regret}~and~\ref{thm:psi_micro_regret}
which state the regret bounds of \Algo{PLT}s for generalized performance metrics. 
We start with presenting two other results being the building blocks of the main proofs.
Both are based on~\citep{Kotlowski_Dembczynski_2017}. 
The first one states that the regret for a cost-sensitive binary classification 
can be upperbounded by $L_1$ estimation error of the conditional probabilities 
by using a proper threshold that corresponds to the misclassification cost.
The second one shows that the regret of the generic function $\Psi(\FP, \FN)$ can be upperbounded
by the regret of the cost-sensitive binary classification 
with costs being a function of the optimal value of $\Psi(\FP, \FN)$. 

Given a real number $\alpha \in [0,1]$, 
let us first define an $\alpha$-cost-sensitive loss function for a single binary label $y$, 
$\ell_\alpha: \{0,1\} \times \{0,1\} \rightarrow [0,2]$, as:
$$
\ell_\alpha(y, \hy) = 2 \alpha \assert{y = 0} \assert{\hy = 1} + 2 (1 - \alpha) \assert{y = 1} \assert{\hy = 0}
$$
The cost-sensitive loss assigns different costs of misclassification depending on whether the label is relevant ($y=1$) or not ($y=-1$).
The multiplier of 2 makes $\ell_{0.5}(y, \hy)$ to be the typical binary 0/1 loss.
Given classifier $h$, the $\alpha$-cost-sensitive risk of $h$ is:
\begin{equation}
R_{\alpha}(h) = \mathbb{E}_{(y,\bx ) \sim \prob(y,\bx)}[\ell_\alpha(y, \hy)] = 2\alpha\FP(h) + 2(1-\alpha)\FN(h)
\label{eqn:alpha-risk}
\end{equation}
The $\alpha$-cost-sensitive regret of $h$ is then: 
\begin{equation}
\reg_{\alpha}(h) = R_{\alpha}(h) - R_{\alpha}(h^*_\alpha) \,,
\label{eqn:alpha-regret}
\end{equation}
where $h^*_\alpha = \argmin_h R_{\alpha}(h)$. 

\begin{restatable}{proposition}{generalizedalphamarginals}
%[based on proof of Proposition 2. from \cite{Kotlowski_Dembczynski_2017}]
\label{prop:alpha_marginals}
For any distribution $\prob$ over $(y, \bx) \in \{0,1\} \times \calX$, with $\eta(\bx) = \prob(y = 1 \given \bx)$, any $\alpha \in [0, 1]$, and classifier $h$, such that $h(\bx) = \assert{\heta(\bx) > \alpha}$ with $\heta(\bx) \in [0,1]$, the following holds:
$$
\reg_\alpha(h) \le 2 \mathbb{E}_{\bx \sim \prob(\bx))} [ |\eta(\bx) - \heta(\bx)| ]
$$
\end{restatable}
\begin{proof}
The proof is a part of the derivation of the bound from Proposition~2 in~\cite{Kotlowski_Dembczynski_2017}.
Given $\eta \in [0, 1]$ and $h \in [0, 1]$, the conditional $\alpha$-cost-sensitive risk is:
$$
R_\alpha(h \given \bx) = \mathbb{E}_{y \sim \prob(y \given \bx)} \left [ \ell_\alpha(y, h) \right ] = 
2\alpha(1 - \eta) \assert{h = 1} + 2(1 - \alpha)\eta \assert{h = 0} \,.
$$
Let $h^*_\alpha \in \argmin_h \riskcond_\alpha(\eta, h)$. It is easy to check that one of possible solutions is 
\begin{equation}
h^*_\alpha = \assert{\eta > \alpha} \,. 
\label{eqn:optimal-alpha}
\end{equation}
The $\alpha$-conditional cost-sensitive regret is 
$$
\reg_\alpha(h \given \bx) = R_\alpha(h \given \bx) - R_\alpha(h^*_\alpha \given \bx).
$$
If $h = h^*_\alpha$, then $\reg_\alpha(\eta, h) = 0$, otherwise, $\reg_\alpha(\eta, h) = 2|\eta - \alpha|$, so
$$
\reg_\alpha(h \given \bx) = 2 \assert{h \neq h^*_\alpha}|\eta - \alpha|.
$$
In the statement of the theorem, we assume $h(\bx) = \assert{\heta(\bx) > \alpha}$, for some $\heta(\bx) \in [0,1]$, 
that is, $h(\bx)$ has the same form as $h^*_\alpha(\bx)$ in (\ref{eqn:optimal-alpha}).
For such $h(\bx)$ we have:
\begin{equation*}
\reg_\alpha(h \given \bx) \le 2 |\eta - \heta| \,.
\end{equation*}
This statement trivially holds when $h = h^*_\alpha$. 
If $h \neq h^*_\alpha$, then $\eta$ and $\hat \eta$ are on the opposite sides of $\alpha$, 
hence $|\eta - \alpha| \le |\eta - \hat \eta|$.

The unconditional statement is obtained by taking the expectation with respect to $\bx$ of both sides of the above equation:
$$
\reg_\alpha(h) 
= \mathbb{E}_{\bx \sim \prob(\bx)} [ \reg_\alpha(h \given \bx) ]  
\le  2 \mathbb{E}_{\bx \sim \prob(\bx)} [ |\eta(\bx) - \hat \eta(\bx)| ] \,.
$$
\end{proof}

The second result is a modified version of Proposition~1 from \citep{Kotlowski_Dembczynski_2017},
which in turn generalizes Proposition~6 in \citep{Parambath_et_al_2014}.
%, used therein in the proof of Theorem~4. 
\begin{proposition}
\label{prop:psi_alpha}
Let $\Psi$ be a linear-factorial function as defined in (\ref{eqn:psi}) 
with the denominator bounded away from 0 by $\gamma$ as in (\ref{eqn:gamma}).
Take any real values $\FP$, $\FN$ and $\FP^*$, $\FN^*$ in the domain of $\Psi$ such that: 
$$
\Psi(\FP^*, \FN^*) - \Psi(\FP, \FN) \ge 0 \,.
$$
Then, we obtain:
$$
\Psi(\FP^*, \FN^*) - \Psi(\FP, \FN) \le C(\alpha^*_\Psi(\FP - \FP^*) + (1 - \alpha^*_\Psi) (\FN - \FN^*)) \,
$$
where:
$$
\alpha^*_\Psi = \frac{\Psi(\FP^*, \FN^*)  b_1 - a_1}{\Psi(\FP^*, \FN^*) (b_1 + b_2) - (a_1 + a_2)} \,,
$$
and
$$
C = \frac{1}{\gamma} \left( \Psi(\FP^*, \FN^*)  \left(b_1 + b_2\right) - \left(a_1 + a_2\right) \right ) > 0 \,.
$$
\end{proposition}
\begin{proof}
For the sake of clarity, we use a shorthand notation 
$\Psi^* = \Psi(\FP^*, \FN^*)$, $\Psi = \Psi(\FP, \FN)$, 
$A = a_0 + a_1\FP + a_2\FN$, $B = b_0 + b_1\FP + b_2\FN$, 
for the numerator and denominator of $\Psi$, 
and analogously $A^*$ and $B^*$ for $\Psi^*$. 
With this notation, we have:
\begin{eqnarray}
\Psi^* - \Psi & = & \frac{\Psi^* B - A}{B}
= \frac{\Psi^* B - A -  \overbrace{\left ( \Psi^* B^* - A^* \right )}^{=0} }{B} \nonumber \\
& = & \frac{\Psi^* (B - B^*) - (A - A^*) }{B} \nonumber \\ 
& = & \frac{(\Psi^*b_1 - a_1) (\FP - \FP^*) + (\Psi^*b_2 - a_2) (\FN - \FN^*)}{B} \nonumber \\
& \le & \frac{(\Psi^*b_1 - a_1) (\FP - \FP^*) + (\Psi^*b_2 - a_2) (\FN - \FN^*)}{\gamma} \,,
\label{eqn:psi_regret_upper_bound}
\end{eqnarray}
where the last inequality follows from the assumptions that 
$B \ge \gamma$ and $\Psi^* - \Psi \ge 0$. 
Since $\Psi$ is non-increasing in $\FP$ and $\FN$, we have:
$$
\frac{\partial \Psi^*}{\partial \FP^*} 
= \frac{a_1 B^* - b_1 A^*}{(B^*)^2} 
= \frac{a_1 - b_1 \Psi^*}{B^*} \le 0
$$
and similarly $\frac{\partial \Psi^*}{\partial \FN^*} = \frac{a_2 - b_2 \Psi^*}{B^*} \le 0$.
This and the assumption $B^* \ge \gamma$ implies that 
both $\Psi\*b_1 - a_1$ and $\Psi^*b_2 - a_2$ are non-negative. 
If we normalize them by defining:
$$
\alpha^*_\Psi = \frac{\Psi^*  b_1 - a_1}{\Psi^* (b_1 + b_2) - (a_1 + a_2)} \,,
$$
we obtain then from (\ref{eqn:psi_regret_upper_bound}):
$$
\Psi^* - \Psi \le C(\alpha^*_\Psi(\FP - \FP^*) + (1 - \alpha^*_\Psi) (\FN - \FN^*))\,
$$
with $C$ being $\frac{1}{\gamma} \left( \Psi^* \left(b_1 + b_2\right) - \left(a_1 + a_2\right) \right )$.
\end{proof}

%\subsection{Macro-averaging}

With the above results we can prove the main theorems of Section~\ref{subsec:analyis-generalized}.

\psimacroregret*

\begin{proof}
From the definitions of the macro-average performance measure (\ref{eqn:psimacro}) 
and the regret of $\Psi_{\textrm{macro}}$ (\ref{eqn:macro-regret}), 
as well as from Proposition~\ref{prop:psi_alpha} we have for any $\bh(\bx) = (h_1(\bx), h_2(\bx), \ldots, h_m(\bx))$ that:
$$
\reg_{\Psi_{\mathrm{macro}}}(\bh) \le \frac{1}{m} \sum_{j=1}^m C_j (\alpha^*_\Psi(\FP_j - \FP_j^*) + 
(1 - \alpha^*_\Psi) (\FN_j - \FN_j^*)) \,, 
$$
with $\FP_j$ and $\FN_j$ being the false positives and false negatives of $h_{j}$. 
It can be easily notice $(\alpha^*_\Psi(\FP_j - \FP_j^*) + (1 - \alpha^*_\Psi) (\FN_j - \FN_j^*))$ is half of the $\alpha^*_\Psi$-regret~(\ref{eqn:alpha-regret}) for label $j$. Therefore, we can write:
$$
\reg_{\Psi_{\mathrm{macro}}}(\bh) \le \frac{1}{2m} \sum_{j=1}^m C_j \reg_{\alpha^*_\Psi}(h_j) \,.
$$
If we now take $h_j = h_{j,\alpha^*_{\Psi,j}}$, 
then by using Proposition~\ref{prop:alpha_marginals} 
and the bound~(\ref{eqn:generalized_total_regret_bound}) from Theorem~\ref{thm:total_regret_bound} we obtain:
\begin{eqnarray*}
\reg_{\Psi_{\mathrm{macro}}}(\bh_{\balpha^*_\Psi}) & \le &  \frac{1}{m} \sum_{j=1}^m C_j \mathbb{E}_{\bx \sim \prob(\bx))} [ |\eta(\bx) - \heta(\bx)| ] \\
 & \le & \frac{\sqrt{2}}{m\sqrt{\lambda}} \sum_{\node \in \nodes} \sqrt{ \prob(z_{\pa{v}} = 1) \reg_{\lossfunc_c}(f_v)} \sum_{j \in L_v} C_j \,.
\end{eqnarray*}
Finally, since 
$$
\btau^* = \argmax_{\btau} \Psi_{\textrm{macro}}(\bh_{\btau}) = \argmin_{\btau} \reg_{\Psi_{\mathrm{macro}}}(\bh_{\btau})\,,
$$ 
we have that $\reg_{\Psi_{\mathrm{macro}}}(\bh_{\btau^*}) \le \reg_{\Psi_{\mathrm{macro}}}(\bh_{\balpha^*_\Psi})$.
\end{proof}

\psimicroregret*
\begin{proof}
Using Proposition~\ref{prop:psi_alpha}, which applies to any real values $\FP$, $\FN$, $\FP^*$, $\FN^*$, 
and from definitions of the micro-averaged performance measure~(\ref{eqn:psimicro}) 
and the regret of $\Psi_{\textrm{micro}}$~(\ref{eqn:micro-regret}), 
we can write for any $\bh(\bx) = (h_1(\bx), h_2(\bx), \ldots, h_m(\bx))$ that: 
\begin{eqnarray*}
\reg_{\Psi_{\mathrm{micro}}}(\bh) & = & \Psi(\bar \FP(\bh^*_{\alpha^*_{\Psi}}), \bar \FN(\bh^*_{\alpha^*_{\Psi}})) - \Psi(\bar \FP(\bh), \bar \FN(\bh)) \\
& \le &  C\big( \alpha^*_\Psi(\bar \FP(\bh) - \bar \FP (\bh^*_{\alpha^*_{\Psi}})) + (1 - \alpha^*_\Psi)(\bar \FN(\bh) - \bar \FN (\bh^*_{\alpha^*_{\Psi}}))  \big) \\
& = & \frac{C}{m} \sum_{j = 1}^{m} \alpha^*_\Psi( \FP_j(h_j) - \FP_j (\bh^*_{\alpha^*_{\Psi}})) + (1 - \alpha^*_\Psi)(\FN_j(h_j) - \bar \FN_j (\bh^*_{\alpha^*_{\Psi}})) \,.
\end{eqnarray*}
Further from $\alpha$-cost-sensitive risk (\ref{eqn:alpha-risk}) and regret (\ref{eqn:alpha-regret}), we have:
\begin{eqnarray*}
\reg_{\Psi_{\mathrm{micro}}}(\bh) & \le & \frac{C}{2 m} \sum_{j = 1}^{m} \left ( R_\alpha(h_j) -  R_\alpha(\bh^*_{\alpha^*_{\Psi}})\right ) \\
& = & \frac{C}{2m} \sum_{j = 1}^{m} \reg_\alpha(h_j) \,.
\end{eqnarray*}
If we now take $h_j = h_{j,\alpha^*_{\Psi}}$, for all $j \in \labels$, 
then by using Proposition~\ref{prop:alpha_marginals} we obtain:
$$
\reg_{\Psi_{\mathrm{micro}}}(\bh_{\alpha^*_{\Psi}}) \le \frac{C}{m} \sum_{j = 1}^{m} \mathbb{E}_{\bx} [ |\eta(\bx) - \hat \eta(\bx)| ].
$$
By using the bound~(\ref{eqn:total_regret_bound}) from Theorem~\ref{thm:total_regret_bound} we have:
$$
\reg_{\Psi_{\mathrm{micro}}}(\bh) \le 
\frac{C}{m} \sqrt{\frac{2}{\lambda}} \sum_{\node \in \nodes} |L_\node| \sqrt{ \prob(z_{\pa{v}} = 1) \reg_{\lossfunc_c}(f_v)} \,.
$$
The theorem now follows from noticing that
$$
\tau^* = \argmax_{\tau} \Psi_{\textrm{micro}}(\bh_{\tau}) = \argmin_{\tau} \reg_{\Psi_{\mathrm{micro}}}(\bh_{\tau})\,,
$$ 
we have that $\reg_{\Psi_{\mathrm{micro}}}(\bh_{\tau^*}) \le \reg_{\Psi_{\mathrm{micro}}}(\bh_{\alpha^*_\Psi})$.
\end{proof}

\section{Proof of the result from Section~\ref{subsec:hsm}}
\label{app:hsm}

This appendix shows the full proof of Proposition~\ref{prop:hsm-independent}.

% \todo{
% The proof of Proposition~\ref{prop:hsm-independent} does not use other results from this paper. It starts from the definition of the pick-one-label heuristic from Equation~(\ref{eq:heuristic}). Then it shows that given conditionally independent labels, for which $\prob(\by \given \bx) = \prod_{j=1}^m \prob(y_i \given \bx)$, the order of
% labels sorted according to the values obtained with the heuristic
% gives the same quality of precision$@k$ as the order of
% labels sorted according to their marginal probability.
% }

\hsmindependent*
\begin{proof}
To proof the proposition it suffices to show that for conditionally independent labels the order of 
labels induced by the marginal probabilities $\eta_j(\bx)$ is the same as the order induced by 
the values of $\eta_j'(\bx)$ obtained by the pick-one-label heuristic (\ref{eq:heuristic}):
\begin{equation*}
\eta_j'(\bx) = \prob'(y_j = 1 \given \bx) = \sum_{\by \in \calY}  \frac{y_j}{\sum_{j'=1}^m y_{j'}}\prob(\by \given \bx).
\end{equation*}
In other words, for any two labels $i, j \in \{1, \dots ,m\}$, $i \neq j$, $\eta_i(\bx) \ge \eta_j(\bx) \Leftrightarrow \eta_i'(\bx) \ge \eta_j'(\bx)$.

Let $\eta_i(\bx) \ge \eta_j(\bx)$. The summation over all $\by$ in (\ref{eq:heuristic}) can be written in the following way:
$$
\eta_j'(\bx) = \sum_{\by \in \calY}  y_j N(\by) \prob(\by | \bx)\,,
$$
where $N(\by) = (\sum_{i=1}^m y_{i})^{-1}$ is a value that depends only on the number of positive labels in $\by$. In this summation we consider four subsets of $\mathcal{Y}$, creating a partition of this set: 
$$
    \mathcal{S}^{u,w}_{i,j} = \{ \by \in \mathcal{Y}: y_i = u \land y_j = w \}, \quad u,w \in \{0, 1\}.
$$
The subset $\mathcal{S}^{0,0}_{i,j}$ does not play any role because $y_i = y_j = 0$ and therefore do not contribute to the final sum.
Then (\ref{eq:heuristic}) can be written in the following way for the $i$-th and $j$-th label:
\begin{eqnarray}
\eta_i'(\bx)  & = & \sum_{\by : \mathcal{S}^{1,0}_{i,j}}{ N(\by) \prob(\by | \bx) } +  \sum_{\by \in \mathcal{S}^{1,1}_{i,j}}{ N(\by) \prob(\by | \bx) } \label{eqn:eta_i}\\
\eta_j'(\bx) & = & \sum_{\by : \mathcal{S}^{0,1}_{i,j}}{ N(\by) \prob(\by | \bx) } +  \sum_{\by \in \mathcal{S}^{1,1}_{i,j}}{ N(\by) \prob(\by | \bx) }
\label{eqn:eta_j}
\end{eqnarray}
The contribution of elements from $\mathcal{S}^{1,1}_{i,j}$ is equal for both $\eta_i'(\bx)$ and $\eta_j'(\bx)$.
It is so because the value of $N(\by) \prob(\by | \bx)$ is the same for all $\by \in \mathcal{S}^{1,1}_{i,j}$: the conditional joint probabilities $\prob(\by | \bx)$ are fixed and they are multiplied by the same factors $N(\by)$.

Consider now the contributions of  $\mathcal{S}^{1,0}_{i,j}$ and  $\mathcal{S}^{0,1}_{i,j}$
%$\sum_{\by : \mathcal{S}^{1,0}_{i,j}}{ N_k(\by) \prob(\by | \bx) }$ and $\sum_{\by : \mathcal{S}^{0,1}_{i,j}}{ N_k(\by) \prob(\by | \bx) }$
to the relevant sums. 
By the definition of $\mathcal{Y}$, $\mathcal{S}^{1,0}_{i,j}$, and $\mathcal{S}^{0,1}_{i,j}$, there exists bijection $b_{i,j}: \mathcal{S}^{1,0}_{i,j} \rightarrow \mathcal{S}^{0,1}_{i,j}$, such that for each $\by' \in \mathcal{S}^{1,0}_{i,j}$ there exists $\by'' \in \mathcal{S}^{0,1}_{i,j}$ equal to $\by'$ except on the $i$-th and the $j$-th position.

%between elements $\by^{1,0} \in \mathcal{S}^{1,0}_{i,j}$ and  $\by^{0,1} \in \mathcal{S}^{0,1}_{i,j}$
%such that $\forall_{\by' \in \mathcal{S}^{1,0}} \exists_{\by'' \in \mathcal{S}^{0,1}}  \Big( \by'' = b_{i,j}(\by') \land \forall_{l \in \mathcal{L} \setminus \{i,j\}} y_l; = y_j'' \land y_i' = 1 \land y_i'' = 0 \land y_j' = 0 \land y_j'' = 1 \Big)$.

%such that $\forall_{l \in \{1,\ldots, m\}, l \neq i, l\neq j} y^{1,0}_l = b(y^{1,0})_l$, $y^{1,0}_i \neq b(y^{1,0})_i$, $y^{1,0}_j \neq b(y^{1,0})_j$.

Notice that because of the conditional independence assumption the joint probabilities of elements in $\mathcal{S}^{1,0}_{i,j}$ and $\mathcal{S}^{0,1}_{i,j}$ are related to each other. Let $\by'' = b_{i,j}(\by')$, where $\by' \in \mathcal{S}^{1,0}_{i,j}$ and $\by'' \in \mathcal{S}^{0,1}_{i,j}$. The joint probabilities are:
$$
\prob(\by'| \bx) = \eta_i(\bx)(1 - \eta_j(\bx)) \prod_{l \in \mathcal{L} \setminus \{i,j\}} \eta_l(\bx)^{y_l} (1 - \eta_l(\bx))^{1 - y_l}
$$
and
$$
\prob(\by''| \bx) = (1 - \eta_i(\bx)) \eta_j(\bx) \prod_{l \in \mathcal{L} \setminus \{i,j\}} \eta_l(\bx)^{y_l}(1 - \eta_l(\bx))^{1 - y_l}.
$$
One can easily notice the relation between these probabilities: 
$$
\prob(\by'| \bx) = \eta_i(\bx)(1 - \eta_j(\bx)) q_{i,j} \quad \textrm{and} \quad 
\prob(\by''| \bx) = (1 - \eta_i(\bx)) \eta_j(\bx)q_{i,j},
$$
where $q_{i,j} = \prod_{l \in \mathcal{L} \setminus \{i,j\}}\eta_l(\bx)^{y_l} (1 - \eta_l(\bx))^{1 - y_l} \ge 0$.
Consider now the difference of these two probabilities: %values $\prob(\by'| \bx) - \prob(\by''| \bx)$.
%Let $p_i = \prob(y_i = 1 | \bx)$ and $p_j = \prob(y_j = 1| \bx)$. 
%The difference is %now
\begin{eqnarray*}
\prob(\by'| \bx) - \prob(\by''| \bx) &=&  \eta_i(\bx)(1 - \eta_j(\bx)) q_{i,j} - (1 - \eta_i(\bx)) \eta_j(\bx)q_{i,j}\\
&=& q_{i,j}( \eta_i(\bx)(1 - \eta_j(\bx)) - (1 - \eta_i(\bx))\eta_j(\bx) ) \\
%&=& q_{i,j}(\eta_i(\bx) - \eta_i(\bx) \eta_j(\bx) - \eta_j(\bx) + \eta_i(\bx) \eta_j(\bx) ) \\
&=& q_{i,j}(\eta_i(\bx) - \eta_j(\bx)).
\end{eqnarray*}
From the above we see that $\eta_i(\bx) \ge \eta_j(\bx) \Rightarrow  \prob(\by'| \bx) \ge \prob(\by''| \bx)$.
Due to the properties of the bijection $b_{i,j}$, the number of positive labels in $\by'$ and $\by''$ is the same and $N(\by') = N(\by'')$, therefore we also get $\eta_i(\bx) \ge \eta_j(\bx) \Rightarrow \sum_{\by : \mathcal{S}^{1,0}_{i,j}}{ N(\by) \prob(\by | \bx) }  \ge  \sum_{\by : \mathcal{S}^{0,1}_{i,j}}{ N(\by) \prob(\by | \bx) }$, which by (\ref{eqn:eta_i}) and (\ref{eqn:eta_j}) gives us finally $\eta_i(\bx) \ge \eta_j(\bx) \Rightarrow \eta_i'(\bx) \ge \eta_j'(\bx)$.

The implication in the other direction, that is, $\eta_i(\bx) \ge \eta_j(\bx) \Leftarrow  \prob(\by'| \bx) \ge \prob(\by''| \bx)$ holds obviously for $q_{i,j} > 0$. 
For $q_{i,j} = 0$, we can notice, however, that $\prob(\by'| \bx)$ and $\prob(\by''| \bx)$ do not contribute to the appropriate sums as they are zero, and therefore we can follow a similar reasoning as above, concluding that $\eta_i(\bx) \ge \eta_j(\bx) \Leftarrow \eta_i'(\bx) \ge \eta_j'(\bx)$. 

Thus for conditionally independent labels, the order of labels induced by marginal probabilities $\eta_j(\bx)$ is equal to the order induced by $\eta_j'(\bx)$.
As the precision@$k$ is optimized by $k$ labels with the highest marginal probabilities, we have  
%By the optimality of the order induced by $\eta_j(\bx)$, discussed in \ref{app:prec@k}, 
%for conditionally independent labels 
that prediction consisted of $k$ labels with highest $\eta_j'(\bx)$ has zero regret for precision@$k$.
\end{proof}

%%%%%%%%%%%%%%%%%%%%%%%%%%%%%%%%%%%%%%%%%%%%%%%%%%%%%%%%%%%%%%%%%%%%%%%%%%%%%%%%
%%%%%%%%%%%%%%%%%%%%%%%%%%%%%%%%%%%%%%%%%%%%%%%%%%%%%%%%%%%%%%%%%%%%%%%%%%%%%%%%

\section{The proof of the result from Section~\ref{sec:oplt}}
\label{app:oplt}

Theorem~\ref{thm:oplt} concerns two properties, the properness and the efficency, of an \Algo{OPLT} algorithm.
We first prove that the \Algo{OPLT} algorithm satisfies each of the properties in two separate lemmas. 
The final proof of the theorem is then straight-forward. 
%We start with properness, and then prove the efficiency of \Algo{OPLT}.

\begin{lemma}
\label{lem:proper}
\Algo{OPLT} is a proper \Algo{OPLT} algorithm.
\end{lemma}
\begin{proof}
We need to show that for any $\calS$ and $t$ the two of the following hold.
Firstly, that the set $L_{T_t}$ of leaves of tree $T_t$ built by \Algo{OPLT} correspond to $\calL_t$, 
the set of all labels observed in $S_t$.
Secondly, that the set $H_t$ of classifiers trained by \Algo{OPLT} 
is exactly the same as $H = \textsc{IPLT.Train}(T_t, A_{\textrm{online}}, \mathcal{S}_t)$, 
that is, the set of node classifiers trained incrementally 
by Algorithm~\ref{alg:plt-incremental-learning} 
on $\calD = \calS_t$ and tree $T_t$ given as input parameter.
We will prove it by induction with the base case for $\calS_0$ and 
the induction step for $\calS_t$, $t \ge 1$, with the assumption that the statement holds for $\calS_{t-1}$.

For the base case of $\calS_0$, tree $T_0$ is initialized with the root node $r_T$ with no label assigned 
and set $H_0$ of node classifiers with a single classifier assigned to the root. 
As there are no observations, this classifier receives no updates. 
Now, notice that $\textsc{IPLT.Train}$, run on $T_0$ and $\calS_0$,
returns the same set of classifiers $H$ that contains solely the initialized root node classifier
without any updates  (assuming that initialization procedure is always the same).
There are no labels in any sequence of 0 observations and also $T_0$ has no label assigned.

The induction step is more involved as we need to take into account
the internal loop which extends the tree with new labels. 
Let us consider two cases. 
In the first one, observation $(\bx_t, \calL_{\bx_t})$ does not contain any new label. 
This means that that the tree $T_{t-1}$ will not change, that is, $T_{t-1} = T_t$.
Moreover, node classifiers from $H_{t-1}$ will get the same updates for $(\bx_t, \calL_{\bx_t})$
as classifiers in \Algo{IPLT.Train}, therefore $H_t =  \Algo{IPLT.Train}(T_t, A_{\textrm{online}}, \mathcal{S}_t)$.
It also holds that $l_j \in L_{T_t}$ iff $j \in \labels_t$, since $\labels_{t-1} = \labels_t$.
In the second case, observation $(\bx_t, \calL_{\bx_t})$ has $m' = | \calL_{\bx_t} \setminus \labels_{t-1}|$ new labels. 
Let us make the following assumption for the \Algo{UpdateTree} procedure, 
which we later prove that it indeed holds. 
Namely, we assume that the set $H_{t'}$ of classifiers after calling the \Algo{UpdateTree} procedure 
is the same as the one being returned by $\Algo{IPLT.Train}(T_t, A_{\textrm{online}}, \mathcal{S}_{t-1})$, 
where $T_t$ is the extended tree. 
Moreover, leaves of $T_t$ correspond to all observed labels seen so far.
If this is the case, the rest of the induction step is the same as in the first case. 
All updates to classifiers in $H_{t'}$ for $(\bx_t, \calL_{\bx_t})$ are the same as in \Algo{IPLT.Train}. 
Therefore $H_t =  \Algo{IPLT.Train}(T_t, A_{\textrm{online}}, \mathcal{S}_t)$.

Now, we need to show that the assumption for the \Algo{UpdateTree} procedure holds.
To this end, we also use induction, this time on the number $m'$ of new labels. 
For the base case, we take $m' = 1$. 
The induction step is proved for $m' > 1$ with the assumption that the statement holds for $m' - 1$.

For $m' = 1$ we need consider two scenarios. 
In the first scenario, the new label is the first label in the sequence. 
This label will be then assigned to the root node $r_T$. 
So, the structure of the tree does not change, that is, $T_{t-1} = T_t$. 
Furthermore, the set of classifiers also does not change, 
since the root classifier has already been initialized. 
It might be negatively updated by previous observations. 
Therefore, we have $H_{t'} = \Algo{IPLT.Train}(T_t, A_{\textrm{online}}, \mathcal{S}_{t-1})$.
Furthermore, all observed labels are appropriately assigned to the leaves of $T_t$.
In the second scenario, set $\labels_{t-1}$ is not empty.
We need to consider in this scenario the three variants of tree extension illustrated in Figure~\ref{fig:oplt-tree_building}. 

In the first variant, tree $T_{t-1}$ is extended by one leaf node only without any additional ones. 
\Algo{AddNode} creates a new leaf node $v''$ with the new label assigned to the tree. 
After this operation, the tree contains all labels from $\calS_t$. 
The new leaf $v''$ is added as a child of the selected node $v$. 
This new node is initialized as $\heta(v'') = \textsc{InverseClassifier}(\hat\theta(v))$. 
Recall that \Algo{InverseClassifier} creates a wrapper that inverts the behavior of the base classifier.
It predicts $1 - \heta$, where $\heta$ is the prediction of the base classifier,
and flips the updates, that is, positive updates become negative and negative updates become positive. 
From the definition of the auxiliary classifier,
we know that $\hat\theta(v)$ has been trained on all positives updates of $\heta(v)$.
So, $\heta(v'')$  is initialized with a state as 
if it was updated negatively each time $\heta(v)$ was updated positively in sequence $S_{t-1}$.
Notice that in $S_{t-1}$ there is no observation labeled with the new label.
Therefore $\heta(v'')$ is the same as if it was created and updated using $\textsc{IPLT.Train}$.
There are no other operations on $T_{t-1}$, 
so we have that $H_{t'} = \Algo{IPLT.Train}(T_t, A_{\textrm{online}}, \mathcal{S}_{t-1})$.

In the second variant, tree $T_{t-1}$ is extended by internal node $v'$ and leaf node $v''$.
The internal node $v'$ is added in \Algo{InsertNode}. 
It becomes a parent of all child nodes of the selected node $v$ and the only child of this node. 
Thus, all leaves of the subtree of $v$ do not change.
Since $v'$ is the root of this subtree,
its classifier $\heta(v')$ should be initialized as a copy of the auxiliary classifier $\hat\theta(v)$,
which has accumulated all updates from and only from observations with labels assigned to the leaves of this subtree.
The addition of the leaf node $v''$ can be analyzed as in the first variant.
Since nothing else has changed in the tree and in the node classifiers,
we have that $H_{t'} = \Algo{IPLT.Train}(T_t, A_{\textrm{online}}, \mathcal{S}_{t-1})$.
Moreover, the tree contains the new label, so the statement holds. 

The third variant is similar to the second one. 
Tree $T_{t-1}$ is extended by two leaf nodes $v'$ and $v''$ 
being children of the selected node $v$.
Insertion of leaf $v'$ is similar to the insertion of node $v'$ in the second variant,
with the difference that $v$ does not have any children and 
its label has to be reassigned to $v'$. 
The new classifier in $v'$ is initialized as a copy of the auxiliary classifier $\hat\theta(v)$,
which contains all updates from and only from observations with the label assigned previously to $v$. 
Insertion of $v''$ is exactly the same as in the second variant.
From the above, we conclude that $H_{t'} = \Algo{IPLT.Train}(T_t, A_{\textrm{online}}, \mathcal{S}_{t-1})$
and that $T_t$ contains all labels from $T_{t-1}$ and the new label.
In this way we prove the base case.

The induction step is similar to the second scenario of the base case.
The only difference is that we do not extent tree $T_{t-1}$, 
but an intermediate tree with $m'-1$ new labels already added. 
Because of the induction hypothesis,
the rest of the analysis of the three variants of tree extension is exactly the same.
This ends the proof that the assumption for the inner loop holds. 
At the same time, it finalizes the entire proof. 
\end{proof}

\begin{lemma}
\label{lem:efficient}
\Algo{OPLT} is an efficient \Algo{OPLT} algorithm.
\end{lemma}
\begin{proof}
The \Algo{OPLT} maintains one additional classifier per each node in comparison to \Algo{IPLT}.
Hence, for a single observation there is at most one update more for each positive node. 
Furthermore, the time and space cost of the complete tree building policy is constant per a single label, 
if implemented with an array list.
In this case, insertion of any new node can be made in amortized constant time, 
and the space required by the array list is linear in the number of nodes. 
Concluding the above, the time and space complexity of \Algo{OPLT} is in constant factor of $C_t$ and $C_s$, 
the time and space complexity of \Algo{IPLT} respectively. 
This proves that \Algo{OPLT} is an efficient \Algo{OPLT} algorithm.
\end{proof}

\thmoplt*
\begin{proof}
The theorem directly follows from Lemma~\ref{lem:proper} and Lemma~\ref{lem:efficient}.
\end{proof}%
%\vspace{-2em} % Coś tutaj odległość po proof przenosi się na kolejną stronę

%%%%%%%%%%%%%%%%%%%%%%%%%%%%%%%%%%%%%%%%%%%%%%%%%%%%%%%%%%%%%%%%%%%%%%%%%%%%%%%%%
%%%%%%%%%%%%%%%%%%%%%%%%%%%%%%%%%%%%%%%%%%%%%%%%%%%%%%%%%%%%%%%%%%%%%%%%%%%%%%%%%

\section{Synthetic data}
\label{app:synthetic}

All synthetic models use linear models parametrized by a weight vector $\bw$ of size $d$. 
The values of the vector are sampled uniformly from a $d$-dimensional sphere of radius 1. 
Each observation $\bx$ is a vector sampled from a $d$-dimensional disc of the same radius. 

To create the \multiclass{} data, 
we associate a weight vector $\bw_j$ with each label $j \in \{1, \ldots, m\}$. 
This model assigns probabilities to labels at point $\bx$ using softmax,
\begin{equation*}
\eta_j(\bx) = \frac{\exp(\bw_j^\top \bx)}{ \sum_{j' = 1}^m {\exp(\bw_{j'}^\top \bx)}} \,,
\end{equation*}
and draws the positive label according to this probability distribution. % over labels. 

The \multilabel{} data  with conditionally independent labels are created 
similarly to the multi-class data. 
The difference lays is normalization as the marginal probabilities do not have to sum up to 1. 
To get a probability of the $j$-th label, we use the logistic transformation:
$$
\eta_j(\bx) = \frac{\exp(\bw_j^\top \bx)}{1 + \exp(\bw_j^\top \bx)}.
$$
Then, we assign a label to an observation by:
$$
y_j = \assert{ r < \eta_j(\bx) },
$$
where the random value $r$ is sampled uniformly and independently from range $[0,1]$,
for each instance $\bx$ and label $j \in \{1, \ldots, m\}$.

Generation of the \multilabel{} data with conditionally dependent labels is more involved.
We follow the mixing matrix model previously used 
to a similar purpose in~\citep{Dembczynski_et_al_2012b}. 
This model is based on $m$ latent scoring functions generated by $\mW = (\bw_1, \ldots, \bw_m)$.
The $m \times m$ mixing matrix $\mM$ introduces dependencies between noise $\boldsymbol{\epsilon}$,
which stands for the source of randomness in the model. 
The models $\bw_j$ are sampled from a sphere of radius 1, as in previous cases. 
The values in the mixing matrix $\mM$ are sampled uniformly and independently from $[-1, 1]$. 
The random noise vector $\boldsymbol{\epsilon}$ is sampled from $N(0, 0.25)$. 
The label vector $\by$ is then obtained by element-wise evaluation of the following expression:
$$
\by =  \assert {\mM(\mW^\top\bx + \boldsymbol{\epsilon}) > 0} 
$$
Notice that if $\mM$ was an identity matrix the model would generate independent labels.

In the experiments, we used the following parameters of the synthetic models: 
$d = 3$, 
$n = 100000$ instances (with a $1:1$ split to training and test subsets), 
and $m = 32$ labels. 
%We report mean of 50 random runs of each experiment.

%\clearpage

%%%%%%%%%%%%%%%%%%%%%%%%%%%%%%%%%%%%%%%%%%%%%%%%%%%%%%%%%%%%%%%%%%%%%%%%%%%%%%%%
%%%%%%%%%%%%%%%%%%%%%%%%%%%%%%%%%%%%%%%%%%%%%%%%%%%%%%%%%%%%%%%%%%%%%%%%%%%%%%%%

%%%%%%%%%%%%%%%%%%%%%%%%%%%%%%%%%%%%%%%%%%%%%%%%%%%%%%%%%%%%%%%%%%%%%%%%%%%%%%%%
%%%%%%%%%%%%%%%%%%%%%%%%%%%%%%%%%%%%%%%%%%%%%%%%%%%%%%%%%%%%%%%%%%%%%%%%%%%%%%%%

%\input{05-theory-files/complexity/complexity.tex}

%%%%%%%%%%%%%%%%%%%%%%%%%%%%%%%%%%%%%%%%%%%%%%%%%%%%%%%%%%%%%%%%%%%%%%%%%%%%%%%%
%%%%%%%%%%%%%%%%%%%%%%%%%%%%%%%%%%%%%%%%%%%%%%%%%%%%%%%%%%%%%%%%%%%%%%%%%%%%%%%%
% \section{Weights pruning}
% \label{app:pruning}

% \input{tables/table-weights-thr.tex}

\section{Hyperparamters}
\label{app:hiperparams}

\ifjmlr
In Table~\ref{tab:hyperparameters},
\else
In Tables~\ref{tab:hyperparameters-fastxml},~\ref{tab:hyperparameters-pdsparse},~\ref{tab:hyperparameters-dismec}, and~\ref{tab:hyperparameters-plt},
\fi
we report values of hyperparameters used in all experiments. 
For the state-of-the-art algorithms, used in Section~\ref{sec:plt-vs-sota}, 
we took values recommended in the original articles or default values from the provided implementations. 
For \Algo{PLT}s we tune the following parameters: $c$, $\eta$, \textit{AdaGrad's} $\epsilon$, $\lambda_2$, and \textit{epochs}. 
The setting of the other parameters depends on a given experiment
and is discussed in the main text.
To replicate our experimental results 
we added corresponding scripts to the \Algo{napkinXC} repository.%
\footnote{\url{https://github.com/mwydmuch/napkinXC/experiments}}
Those scripts contain exact values of hyperparameters used.

\ifjmlr

    \begin{table}[H]
    
    \begin{subtable}{\textwidth}
    \begin{center}
    \footnotesize
    \begin{tabular}{c|p{0.58\textwidth}|c}
    \toprule
    hyperparameter & desription & values \\
    \midrule
    $t$ & number of trees & $\{50\}$ \\
    $c$ & SVM weight co-efficient & $\{1.0\}$ \\
    $l$ & number of label-probability pairs to retrain in a leaf & $\{100\}$ \\
    $m$ & maximum allowed instances in a lead node & $\{10\}$ \\
    \midrule
    $\gamma$ & $\gamma$ parameter in tail label classifier (\Algo{PfastreXML} only) & $\{30\}$ \\
    $\alpha$ & trade-off parameter between \Algo{PfastXML} and tail classifier scores (\Algo{PfastreXML} only) & $\{0.8\}$ \\
    $A$ & parameter of the propensity model (\Algo{PfastreXML} only) & $\{0.5, 0.55, 0.6\}$ \\
    $B$ & parameter of the propensity model (\Algo{PfastreXML} only) & $\{0.4, 1.5, 2.6\}$ \\
    \bottomrule
    \end{tabular}
    \caption{\Algo{FastXML} and \Algo{PfastreXML} hyperparamters.}
    \end{center}
    \end{subtable}
    
    \vspace{6pt}
    
    \begin{subtable}{\textwidth}
    \begin{center}
    \footnotesize
    \begin{tabular}{c|p{0.58\textwidth}|c}
    \toprule
    hyperparameter & desription & values \\
    \midrule
    $\lambda_1$ & L1 regularization weight & $\{0.01, 0.1, 1\}$ \\
    $c$ & cost of each sample & $\{1\}$ \\
    $\tau$ & degree of asynchronization & $\{0.1, 1, 10\}$ \\
    $m$ & maximum number of iterations allowed & $\{30\}$ \\
    \bottomrule
    \end{tabular}
    \caption{\Algo{PPDSparse} hyperparamters.}
    \end{center}
    \end{subtable}
    
    \vspace{6pt}
    
    \begin{subtable}{\textwidth}
    \begin{center}
    \footnotesize
    \begin{tabular}{c|p{0.58\textwidth}|c}
    \toprule
    hyperparameter & desription & values \\
    \midrule
    $c$ & \Algo{LIBLINEAR} cost co-efficient, inverse regularization & $\{1\}$ \\
    $\epsilon$ & \Algo{LIBLINEAR} tolerance of termination criterion & $\{0.01\}$ \\
    $\Delta$ & threshold value for pruning linear classifiers weights & $\{0.01\}$ \\
    \bottomrule
    \end{tabular}
    \caption{\Algo{DiSMEC} hyperparamters.}
    \end{center}
    \end{subtable}
    
    \vspace{6pt}
    
    \begin{subtable}{\textwidth}
    \begin{center}
    \footnotesize
    \begin{tabular}{c|p{0.58\textwidth}|c}
    \toprule
    hyperparameter & desription & values \\
    \midrule
    $ensemble$ & number of trees in ensemble & $\{1,3\}$ \\
    \textit{k-means} $\epsilon$ & tolerance of termination criterion of the k-means clustering used for the tree building procedure & $\{0.0001\}$ \\
    \textit{max leaves} & maximum degree of pre-leaf nodes & $\{25, 100, 400\}$ \\
    \midrule
    $c$ & \Algo{LIBLINEAR} cost co-efficient, inverse regularization strength (\Algo{Parabel}, \Algo{nXC} only) & $\{1, 8, 12, 16, 32\}$ \\
    $\epsilon$ & \Algo{LIBLINEAR} tolerance of termination criterion (\Algo{Parabel}, \Algo{nXC} only) & $\{0.1\}$ \\
    $\Delta$ & threshold value for pruning weights (\Algo{Parabel}, \Algo{nXC} only) & $\{0.1, 0.2, 0.3, 0.5\}$ \\
    \midrule
    \textit{max iter} & maximum iterations of \Algo{LIBLINEAR} (\Algo{Parabel} only) & $\{20\}$ \\
    \midrule
    $arity$ & arity of tree nodes, k for k-means clustering (\Algo{XT}, \Algo{nXC} only) & $\{2, 16, 64\}$ \\
    $\eta$ & learning rate for SGD or Adagrad (\Algo{XT}, \Algo{nXC} only) & $\{0.02, 0.2, 0.5, 1\}$ \\
    $epochs$ & number of passes over dataset when training with incremental algorithm (\Algo{XT}, \Algo{nXC} only) & $\{1, 3, 10\}$ \\
    \midrule
    \textit{AdaGrad's} $\epsilon$ & determines initial learning rate (\Algo{nXC} only) & $\{0.01, 0.001\}$ \\
    \midrule
    $\lambda_2$ & L2 regularization weight (\Algo{XT} only) & $\{0.001, 0.002, 0.003\}$ \\
    \textit{dim} & size of hidden representation (\Algo{XT} only) & $\{500\}$ \\
    \bottomrule
    \end{tabular}
    \caption{Hyperparamters of different \Algo{PLT}s implementations: \Algo{Parabel}, \Algo{extremeText} (\Algo{XT}) and \Algo{napkinXC} \Algo{PLT}, \Algo{OPLT} and \Algo{HSM} (\Algo{nXC}).}
    
    \end{center}
    \end{subtable}
    
    \caption{Hyperparameters of algorithms used in experiments of Section~\ref{sec:experiments}}.
    \label{tab:hyperparameters}
    \end{table}

\else

    \begin{table}[H]
    
    \begin{center}
    \footnotesize
    \begin{tabular}{c|p{0.58\textwidth}|c}
    \toprule
    hyperparameter & desription & values \\
    \midrule
    $t$ & number of trees & $\{50\}$ \\
    $c$ & SVM weight co-efficient & $\{1.0\}$ \\
    $l$ & number of label-probability pairs to retrain in a leaf & $\{100\}$ \\
    $m$ & maximum allowed instances in a lead node & $\{10\}$ \\
    \midrule
    $\gamma$ & $\gamma$ parameter in tail label classifier (\Algo{PfastreXML} only) & $\{30\}$ \\
    $\alpha$ & trade-off parameter between \Algo{PfastXML} and tail classifier scores (\Algo{PfastreXML} only) & $\{0.8\}$ \\
    $A$ & parameter of the propensity model (\Algo{PfastreXML} only) & $\{0.5, 0.55, 0.6\}$ \\
    $B$ & parameter of the propensity model (\Algo{PfastreXML} only) & $\{0.4, 1.5, 2.6\}$ \\
    \bottomrule
    \end{tabular}
    \caption{\Algo{FastXML} and \Algo{PfastreXML} hyperparamters.}
    \label{tab:hyperparameters-fastxml}
    \end{center}
    \end{table}

    \begin{table}[H]
    \begin{center}
    \footnotesize
    \begin{tabular}{c|p{0.58\textwidth}|c}
    \toprule
    hyperparameter & desription & values \\
    \midrule
    $\lambda_1$ & L1 regularization weight & $\{0.01, 0.1, 1\}$ \\
    $c$ & cost of each sample & $\{1\}$ \\
    $\tau$ & degree of asynchronization & $\{0.1, 1, 10\}$ \\
    $m$ & maximum number of iterations allowed & $\{30\}$ \\
    \bottomrule
    \end{tabular}
    \caption{\Algo{PPDSparse} hyperparamters.}
    \label{tab:hyperparameters-pdsparse}
    \end{center}
    \end{table}

    \begin{table}[H]
    \begin{center}
    \footnotesize
    \begin{tabular}{c|p{0.58\textwidth}|c}
    \toprule
    hyperparameter & desription & values \\
    \midrule
    $c$ & \Algo{LIBLINEAR} cost co-efficient, inverse regularization & $\{1\}$ \\
    $\epsilon$ & \Algo{LIBLINEAR} tolerance of termination criterion & $\{0.01\}$ \\
    $\Delta$ & threshold value for pruning linear classifiers weights & $\{0.01\}$ \\
    \bottomrule
    \end{tabular}
    \caption{\Algo{DiSMEC} hyperparamters.}
    \label{tab:hyperparameters-dismec}
    \end{center}
    \end{table}
    
    \begin{table}[H]
    \begin{center}
    \footnotesize
    \begin{tabular}{c|p{0.58\textwidth}|c}
    \toprule
    hyperparameter & desription & values \\
    \midrule
    $ensemble$ & number of trees in ensemble & $\{1,3\}$ \\
    \textit{k-means} $\epsilon$ & tolerance of termination criterion of the k-means clustering used for the tree building procedure & $\{0.0001\}$ \\
    \textit{max leaves} & maximum degree of pre-leaf nodes & $\{25, 100, 400\}$ \\
    \midrule
    $c$ & \Algo{LIBLINEAR} cost co-efficient, inverse regularization strength (\Algo{Parabel}, \Algo{nXC} only) & $\{1, 8, 12, 16, 32\}$ \\
    $\epsilon$ & \Algo{LIBLINEAR} tolerance of termination criterion (\Algo{Parabel}, \Algo{nXC} only) & $\{0.1\}$ \\
    $\Delta$ & threshold value for pruning weights (\Algo{Parabel}, \Algo{nXC} only) & $\{0.1, 0.2, 0.3, 0.5\}$ \\
    \midrule
    \textit{max iter} & maximum iterations of \Algo{LIBLINEAR} (\Algo{Parabel} only) & $\{20\}$ \\
    \midrule
    $arity$ & arity of tree nodes, k for k-means clustering (\Algo{XT}, \Algo{nXC} only) & $\{2, 16, 64\}$ \\
    $\eta$ & learning rate for SGD or Adagrad (\Algo{XT}, \Algo{nXC} only) & $\{0.02, 0.2, 0.5, 1\}$ \\
    $epochs$ & number of passes over dataset when training with incremental algorithm (\Algo{XT}, \Algo{nXC} only) & $\{1, 3, 10\}$ \\
    \midrule
    \textit{AdaGrad's} $\epsilon$ & determines initial learning rate (\Algo{nXC} only) & $\{0.01, 0.001\}$ \\
    \midrule
    $\lambda_2$ & L2 regularization weight (\Algo{XT} only) & $\{0.001, 0.002, 0.003\}$ \\
    \textit{dim} & size of hidden representation (\Algo{XT} only) & $\{500\}$ \\
    \bottomrule
    \end{tabular}
    \caption{Hyperparamters of different \Algo{PLT}s implementations: \Algo{Parabel}, \Algo{extremeText} (\Algo{XT}) and \Algo{napkinXC} \Algo{PLT}, \Algo{OPLT} and \Algo{HSM} (\Algo{nXC}).}
    \label{tab:hyperparameters-plt}
    \end{center}
    \end{table}
    
    % \caption{Hyperparameters of algorithms used in experiments of Section~\ref{sec:experiments}}.
    % \label{tab:hyperparameters}
    % \end{table}

\fi

\section{Weight pruning}
\label{app:pruning}

In all the experiments, we used a threshold of 0.1 for weight pruning. 
We present results for higher values of threshold 
and analyze their impact on the predictive and computational performance of \Algo{PLT}s.
Table~\ref{tab:weights-thr} reports results for logistic and squared hinge loss.
We observe that for logistic loss a more aggressive pruning can be beneficial.
\Precatk{} decreases only slightly, while testing time can be reduced almost by two, 
and the model size even by 4. 
For squared hinge loss, \precatk{} drops more substantially, 
but the model size can be even reduced by a factor of 10.
Let us also recall that weight pruning has also been investigated by \citet{Prabhu_et_al_2018},
with similar outcomes to those presented here.

\input{tables/table-weights-thr.tex}

\section{Tree depth impact for the squared hinge loss}
\label{app:treehinge}

We present additional results concerning different tree shapes,
namely the tree depth, for the squared hinge loss.
\input{tables/table-trees-arity-maxleaves-sh}

%% file: tables/table-weights-thr.tex
\begin{table}[h]
\centering

\begin{subtable}{1\textwidth}
\footnotesize{
\resizebox{\textwidth}{!}{
\tabcolsep=5pt
\begin{tabular}{l|r@{}lr@{}lr@{}l|r@{}lr@{}lr@{}l}
\specialrule{0.94pt}{0.4ex}{0.65ex}
& \multicolumn{2}{c}{0.1} & \multicolumn{2}{c}{0.3} & \multicolumn{2}{c|}{0.5} 
& \multicolumn{2}{c}{0.1} & \multicolumn{2}{c}{0.3} & \multicolumn{2}{c}{0.5} \\
\specialrule{0.94pt}{0.4ex}{0.65ex}
& \multicolumn{6}{c|}{$p@1$ [\%]}
& \multicolumn{6}{c}{$p@5$ [\%]} \\
\midrule
% \eurlex & \boldmath$80.51$ & \boldmath$\pm 0.16$ & $80.48$ & $\pm 0.13$ & $80.13$ & $\pm 0.17$ & \boldmath$53.33$ & \boldmath$\pm 0.68$ & $53.23$ & $\pm 0.69$ & $52.80$ & $\pm 0.69$ \\
% \amazoncatsmall & \boldmath$93.04$ & \boldmath$\pm 0.02$ & $93.02$ & $\pm 0.02$ & $92.94$ & $\pm 0.03$ & \boldmath$63.70$ & \boldmath$\pm 0.02$ & $63.67$ & $\pm 0.01$ & $63.58$ & $\pm 0.02$ \\
% \wikiten & $85.36$ & $\pm 0.09$ & \boldmath$85.46$ & \boldmath$\pm 0.07$ & $85.35$ & $\pm 0.06$ & \boldmath$63.84$ & \boldmath$\pm 0.07$ & $63.75$ & $\pm 0.06$ & $63.62$ & $\pm 0.09$ \\
% \deliciouslarge & \boldmath$49.55$ & \boldmath$\pm 0.05$ & $48.60$ & $\pm 0.11$ & $46.35$ & $\pm 0.18$ & \boldmath$39.90$ & \boldmath$\pm 0.02$ & $39.55$ & $\pm 0.04$ & $38.31$ & $\pm 0.15$ \\
\wikilshtc & \boldmath$61.96$ & \boldmath$\pm 0.03$ & $61.95$ & $\pm 0.03$ & $61.82$ & $\pm 0.03$ & \boldmath$30.19$ & \boldmath$\pm 0.02$ & $30.18$ & $\pm 0.02$ & $30.11$ & $\pm 0.02$ \\
\wikipedia & \boldmath$66.20$ & \boldmath$\pm 0.05$ & $65.95$ & $\pm 0.11$ & $65.52$ & $\pm 0.07$ & \boldmath$36.83$ & \boldmath$\pm 0.01$ & $36.65$ & $\pm 0.04$ & $36.40$ & $\pm 0.02$ \\
\amazon & \boldmath$43.54$ & \boldmath$\pm 0.01$ & $43.23$ & $\pm 0.02$ & $42.67$ & $\pm 0.02$ & \boldmath$35.15$ & \boldmath$\pm 0.03$ & $34.81$ & $\pm 0.03$ & $34.16$ & $\pm 0.02$ \\
\amazonlarge & \boldmath$46.09$ & \boldmath$\pm 0.02$ & $45.94$ & $\pm 0.01$ & $45.57$ & $\pm 0.01$ & \boldmath$40.98$ & \boldmath$\pm 0.01$ & $40.82$ & $\pm 0.01$ & $40.35$ & $\pm 0.01$ \\

\specialrule{0.94pt}{0.4ex}{0.65ex}
& \multicolumn{6}{c|}{$T/N_{\textrm{test}}$ [ms]}
& \multicolumn{6}{c}{$M_{\textrm{size}}$ [GB]} \\
\midrule
% \eurlex & $0.39$ & $\pm 0.03$ & \boldmath$0.38$ & \boldmath$\pm 0.05$ & $0.38$ & $\pm 0.04$ & $0.02$ & $\pm 0.00$ & \boldmath$0.01$ & \boldmath$\pm 0.00$ & \boldmath$0.01$ & \boldmath$\pm 0.00$ \\
% \amazoncatsmall & $0.32$ & $\pm 0.03$ & $0.21$ & $\pm 0.01$ & \boldmath$0.17$ & \boldmath$\pm 0.01$ & $0.35$ & $\pm 0.00$ & $0.17$ & $\pm 0.00$ & \boldmath$0.10$ & \boldmath$\pm 0.00$ \\
% \wikiten & $5.35$ & $\pm 0.32$ & $3.95$ & $\pm 0.14$ & \boldmath$2.80$ & \boldmath$\pm 0.11$ & $0.58$ & $\pm 0.00$ & $0.17$ & $\pm 0.00$ & \boldmath$0.09$ & \boldmath$\pm 0.00$ \\
% \deliciouslarge & $9.89$ & $\pm 0.89$ & $5.25$ & $\pm 0.27$ & \boldmath$3.03$ & \boldmath$\pm 0.30$ & $0.95$ & $\pm 0.00$ & $0.18$ & $\pm 0.00$ & \boldmath$0.06$ & \boldmath$\pm 0.00$ \\
\wikilshtc & $1.77$ & $\pm 0.11$ & $1.21$ & $\pm 0.09$ & \boldmath$1.02$ & \boldmath$\pm 0.04$ & $2.73$ & $\pm 0.00$ & $1.38$ & $\pm 0.00$ & \boldmath$0.89$ & \boldmath$\pm 0.00$ \\
\wikipedia & $6.67$ & $\pm 0.23$ & $4.71$ & $\pm 0.15$ & \boldmath$4.15$ & \boldmath$\pm 0.03$ & $8.89$ & $\pm 0.00$ & $2.90$ & $\pm 0.00$ & \boldmath$1.55$ & \boldmath$\pm 0.00$ \\
\amazon & $4.13$ & $\pm 0.28$ & $2.53$ & $\pm 0.06$ & \boldmath$2.20$ & \boldmath$\pm 0.11$ & $2.26$ & $\pm 0.00$ & $0.77$ & $\pm 0.00$ & \boldmath$0.42$ & \boldmath$\pm 0.00$ \\
\amazonlarge & $3.26$ & $\pm 0.08$ & $2.54$ & $\pm 0.11$ & \boldmath$2.03$ & \boldmath$\pm 0.09$ & $20.84$ & $\pm 0.00$ & $8.54$ & $\pm 0.00$ & \boldmath$4.61$ & \boldmath$\pm 0.00$ \\

\specialrule{0.94pt}{0.4ex}{0.65ex}
\end{tabular}}}
\caption{Results for logistic loss.}
\end{subtable}

\vspace{12pt}

\begin{subtable}{\textwidth}
\footnotesize{
\resizebox{\textwidth}{!}{
\tabcolsep=5pt
\begin{tabular}{l|r@{}lr@{}lr@{}l|r@{}lr@{}lr@{}l}
\specialrule{0.94pt}{0.4ex}{0.65ex}
& \multicolumn{2}{c}{0.1} & \multicolumn{2}{c}{0.2} & \multicolumn{2}{c|}{0.3} 
& \multicolumn{2}{c}{0.1} & \multicolumn{2}{c}{0.2} & \multicolumn{2}{c}{0.3} \\
\specialrule{0.94pt}{0.4ex}{0.65ex}
& \multicolumn{6}{c|}{$p@1$ [\%]}
& \multicolumn{6}{c}{$p@5$ [\%]} \\
\midrule
% \eurlex & \boldmath$80.17$ & \boldmath$\pm 0.27$ & $78.32$ & $\pm 0.19$ & $74.52$ & $\pm 0.44$ & \boldmath$53.01$ & \boldmath$\pm 0.87$ & $50.45$ & $\pm 0.97$ & $45.99$ & $\pm 0.96$ \\
% \amazoncatsmall & \boldmath$92.40$ & \boldmath$\pm 0.04$ & $92.11$ & $\pm 0.03$ & $90.99$ & $\pm 0.13$ & \boldmath$63.88$ & \boldmath$\pm 0.02$ & $63.37$ & $\pm 0.02$ & $61.49$ & $\pm 0.05$ \\
% \wikiten & \boldmath$84.17$ & \boldmath$\pm 0.10$ & $79.33$ & $\pm 0.24$ & $76.47$ & $\pm 0.47$ & \boldmath$63.12$ & \boldmath$\pm 0.04$ & $59.75$ & $\pm 0.16$ & $49.12$ & $\pm 0.44$ \\
% \deliciouslarge & \boldmath$46.30$ & \boldmath$\pm 0.07$ & $43.63$ & $\pm 0.19$ & $37.75$ & $\pm 0.40$ & \boldmath$36.54$ & \boldmath$\pm 0.07$ & $35.14$ & $\pm 0.03$ & $31.65$ & $\pm 0.64$ \\
\wikilshtc & \boldmath$62.78$ & \boldmath$\pm 0.03$ & $60.77$ & $\pm 0.06$ & $56.83$ & $\pm 0.13$ & \boldmath$30.25$ & \boldmath$\pm 0.02$ & $29.12$ & $\pm 0.02$ & $27.05$ & $\pm 0.05$ \\
\wikipedia & \boldmath$66.77$ & \boldmath$\pm 0.08$ & $63.88$ & $\pm 0.00$ & $59.62$ & $\pm 0.13$ & \boldmath$36.94$ & \boldmath$\pm 0.02$ & $34.95$ & $\pm 0.00$ & $32.25$ & $\pm 0.05$ \\
\amazon & \boldmath$43.31$ & \boldmath$\pm 0.03$ & $40.59$ & $\pm 0.02$ & $35.67$ & $\pm 0.06$ & \boldmath$34.31$ & \boldmath$\pm 0.03$ & $31.14$ & $\pm 0.04$ & $26.49$ & $\pm 0.04$ \\
\amazonlarge & \boldmath$46.23$ & \boldmath$\pm 0.01$ & $44.74$ & $\pm 0.00$ & $39.87$ & $\pm 0.02$ & \boldmath$41.41$ & \boldmath$\pm 0.01$ & $39.67$ & $\pm 0.00$ & $35.16$ & $\pm 0.02$ \\

\specialrule{0.94pt}{0.4ex}{0.65ex}
& \multicolumn{6}{c|}{$T/N_{\textrm{test}}$ [ms]}
& \multicolumn{6}{c}{$M_{\textrm{size}}$ [GB]} \\
\midrule

% \eurlex & \boldmath$0.24$ & \boldmath$\pm 0.02$ & $0.24$ & $\pm 0.02$ & $0.28$ & $\pm 0.02$ & \boldmath$0.00$ & \boldmath$\pm 0.00$ & $0.00$ & $\pm 0.00$ & $0.00$ & $\pm 0.00$ \\
% \amazoncatsmall & $0.19$ & $\pm 0.00$ & $0.14$ & $\pm 0.02$ & \boldmath$0.11$ & \boldmath$\pm 0.01$ & $0.19$ & $\pm 0.00$ & $0.07$ & $\pm 0.00$ & \boldmath$0.03$ & \boldmath$\pm 0.00$ \\
% \wikiten & $2.87$ & $\pm 0.08$ & $1.69$ & $\pm 0.05$ & \boldmath$1.17$ & \boldmath$\pm 0.06$ & $0.06$ & $\pm 0.00$ & \boldmath$0.01$ & \boldmath$\pm 0.00$ & $0.01$ & $\pm 0.00$ \\
% \deliciouslarge & $10.07$ & $\pm 0.23$ & $4.06$ & $\pm 0.25$ & \boldmath$1.59$ & \boldmath$\pm 0.18$ & $1.82$ & $\pm 0.00$ & $0.38$ & $\pm 0.00$ & \boldmath$0.15$ & \boldmath$\pm 0.00$ \\
\wikilshtc & $0.86$ & $\pm 0.06$ & $0.55$ & $\pm 0.02$ & \boldmath$0.45$ & \boldmath$\pm 0.01$ & $0.97$ & $\pm 0.00$ & $0.24$ & $\pm 0.00$ & \boldmath$0.10$ & \boldmath$\pm 0.00$ \\
\wikipedia & $2.86$ & $\pm 0.07$ & \boldmath$1.99$ & \boldmath$\pm 0.00$ & $2.05$ & $\pm 0.07$ & $1.78$ & $\pm 0.00$ & $0.39$ & $\pm 0.00$ & \boldmath$0.17$ & \boldmath$\pm 0.00$ \\
\amazon & $1.32$ & $\pm 0.08$ & \boldmath$1.01$ & \boldmath$\pm 0.02$ & $1.17$ & $\pm 0.02$ & $0.63$ & $\pm 0.00$ & $0.18$ & $\pm 0.00$ & \boldmath$0.10$ & \boldmath$\pm 0.00$ \\
\amazonlarge & $1.96$ & $\pm 0.05$ & $1.09$ & $\pm 0.00$ & \boldmath$0.95$ & \boldmath$\pm 0.02$ & $9.86$ & $\pm 0.00$ & $2.31$ & $\pm 0.00$ & \boldmath$0.89$ & \boldmath$\pm 0.00$ \\

\specialrule{0.94pt}{0.4ex}{0.65ex}
\end{tabular}}}
\caption{Results for squared hinge loss.}
\end{subtable}

\caption{\Precatk{} for $k = 1,\,5$, average prediction times, 
and model sizes with different thresholds of weight pruning for different losses.}
\label{tab:weights-thr}
\end{table}

%% file: tables/table-trees-arity-maxleaves-sh.tex
\begin{table}[H]
\centering

\begin{subtable}{1\textwidth}

\footnotesize{
\resizebox{\textwidth}{!}{
\tabcolsep=5pt
\begin{tabular}{l|r@{}lr@{}lr@{}l|r@{}lr@{}lr@{}l}
\specialrule{0.94pt}{0.4ex}{0.65ex}
Arity & \multicolumn{2}{c}{2} & \multicolumn{2}{c}{16} & \multicolumn{2}{c|}{64}
& \multicolumn{2}{c}{2} & \multicolumn{2}{c}{16} & \multicolumn{2}{c}{64} \\
\specialrule{0.94pt}{0.4ex}{0.65ex}
& \multicolumn{6}{c|}{$p@1$ [\%]}
& \multicolumn{6}{c}{$T/N_{\textrm{test}}$ [ms]} \\
\midrule
%\eurlex & $80.17$ & $\pm 0.27$ & $81.62$ & $\pm 0.15$ & \boldmath$81.68$ & \boldmath$\pm 0.24$ & $0.24$ & $\pm 0.02$ & \boldmath$0.18$ & \boldmath$\pm 0.01$ & $0.27$ & $\pm 0.02$ \\
%\amazoncatsmall & $92.40$ & $\pm 0.04$ & $92.39$ & $\pm 0.06$ & \boldmath$92.46$ & \boldmath$\pm 0.07$ & \boldmath$0.19$ & \boldmath$\pm 0.00$ & $0.23$ & $\pm 0.01$ & $0.31$ & $\pm 0.03$ \\
%\wikiten & $84.17$ & $\pm 0.10$ & $84.30$ & $\pm 0.07$ & \boldmath$84.62$ & \boldmath$\pm 0.07$ & \boldmath$2.87$ & \boldmath$\pm 0.08$ & $2.95$ & $\pm 0.20$ & $4.14$ & $\pm 0.28$ \\
%\deliciouslarge & $46.30$ & $\pm 0.07$ & $46.19$ & $\pm 0.08$ & \boldmath$46.31$ & \boldmath$\pm 0.06$ & \boldmath$10.07$ & \boldmath$\pm 0.23$ & $12.59$ & $\pm 0.51$ & $11.83$ & $\pm 0.48$ \\
\wikilshtc & $62.78$ & $\pm 0.03$ & $64.17$ & $\pm 0.05$ & \boldmath$64.61$ & \boldmath$\pm 0.04$ & \boldmath$0.86$ & \boldmath$\pm 0.06$ & $0.90$ & $\pm 0.07$ & $1.42$ & $\pm 0.05$ \\
\wikipedia & $66.77$ & $\pm 0.08$ & \boldmath$68.16$ & \boldmath$\pm 0.10$ & $68.02$ & $\pm 0.01$ & \boldmath$2.86$ & \boldmath$\pm 0.07$ & $4.41$ & $\pm 0.12$ & $5.55$ & $\pm 0.00$ \\
\amazon & $43.31$ & $\pm 0.03$ & $43.88$ & $\pm 0.05$ & \boldmath$44.03$ & \boldmath$\pm 0.05$ & \boldmath$1.32$ & \boldmath$\pm 0.08$ & $1.73$ & $\pm 0.15$ & $2.68$ & $\pm 0.17$ \\
\amazonlarge & $46.23$ & $\pm 0.01$ & $46.98$ & $\pm 0.01$ & \boldmath$47.33$ & \boldmath$\pm 0.00$ & \boldmath$1.96$ & \boldmath$\pm 0.05$ & $2.39$ & $\pm 0.09$ & $2.56$ & $\pm 0.00$ \\

\specialrule{0.94pt}{0.4ex}{0.65ex}

& \multicolumn{6}{c|}{$T_{\textrm{train}}$ [h]}
& \multicolumn{6}{c}{$M_{\textrm{size}}$ [GB]} \\
\midrule
%\eurlex & \boldmath$0.01$ & \boldmath$\pm 0.00$ & $0.01$ & $\pm 0.00$ & $0.01$ & $\pm 0.00$ & \boldmath$0.00$ & \boldmath$\pm 0.00$ & $0.00$ & $\pm 0.00$ & $0.01$ & $\pm 0.00$ \\
%\amazoncatsmall & \boldmath$0.29$ & \boldmath$\pm 0.00$ & $0.38$ & $\pm 0.04$ & $0.47$ & $\pm 0.03$ & $0.19$ & $\pm 0.00$ & \boldmath$0.18$ & \boldmath$\pm 0.00$ & $0.18$ & $\pm 0.00$ \\
%\wikiten & \boldmath$0.11$ & \boldmath$\pm 0.00$ & $0.16$ & $\pm 0.00$ & $0.36$ & $\pm 0.02$ & $0.06$ & $\pm 0.00$ & $0.05$ & $\pm 0.00$ & \boldmath$0.04$ & \boldmath$\pm 0.00$ \\
%\deliciouslarge & \boldmath$5.51$ & \boldmath$\pm 0.29$ & $9.07$ & $\pm 0.76$ & $12.23$ & $\pm 0.78$ & $1.82$ & $\pm 0.00$ & $1.77$ & $\pm 0.00$ & \boldmath$1.74$ & \boldmath$\pm 0.00$ \\
\wikilshtc & \boldmath$1.60$ & \boldmath$\pm 0.06$ & $2.18$ & $\pm 0.07$ & $3.85$ & $\pm 0.17$ & $0.97$ & $\pm 0.00$ & $0.90$ & $\pm 0.00$ & \boldmath$0.88$ & \boldmath$\pm 0.00$ \\
\wikipedia & \boldmath$9.48$ & \boldmath$\pm 0.33$ & $15.91$ & $\pm 0.55$ & $28.68$ & $\pm 0.74$ & $1.78$ & $\pm 0.00$ & $1.52$ & $\pm 0.00$ & \boldmath$1.49$ & \boldmath$\pm 0.00$ \\
\amazon & \boldmath$0.40$ & \boldmath$\pm 0.01$ & $0.64$ & $\pm 0.02$ & $1.57$ & $\pm 0.04$ & $0.63$ & $\pm 0.00$ & $0.55$ & $\pm 0.00$ & \boldmath$0.52$ & \boldmath$\pm 0.00$ \\
\amazonlarge & \boldmath$5.44$ & \boldmath$\pm 0.13$ & $9.82$ & $\pm 0.23$ & $20.82$ & $\pm 0.00$ & $9.86$ & $\pm 0.00$ & $9.36$ & $\pm 0.00$ & \boldmath$9.24$ & \boldmath$\pm 0.00$ \\
\specialrule{0.94pt}{0.4ex}{0.65ex}
\end{tabular}}}

\caption{Results for arity equal to 2, 16 or 64 and pre-leaf node degree equal to 100.}
\label{tab:k-means-tree-arity-sh}
\end{subtable}

\vspace{12pt}

\begin{subtable}{\textwidth}

\footnotesize{
\resizebox{\textwidth}{!}{
\tabcolsep=5pt
\begin{tabular}{l|r@{}lr@{}lr@{}l|r@{}lr@{}lr@{}l}
\specialrule{0.94pt}{0.4ex}{0.65ex}
Pre-leaf degree & \multicolumn{2}{c}{25} & \multicolumn{2}{c}{100} & \multicolumn{2}{c|}{400} 
& \multicolumn{2}{c}{25} & \multicolumn{2}{c}{100} & \multicolumn{2}{c}{400} \\
\specialrule{0.94pt}{0.4ex}{0.65ex}
& \multicolumn{6}{c|}{$p@1$ [\%]}
& \multicolumn{6}{c}{$T/N_{\textrm{test}}$ [ms]} \\
\midrule
% \eurlex & \boldmath$80.31$ & \boldmath$\pm 0.19$ & $80.17$ & $\pm 0.27$ & $79.21$ & $\pm 0.56$ & \boldmath$0.09$ & \boldmath$\pm 0.01$ & $0.24$ & $\pm 0.02$ & $0.54$ & $\pm 0.02$ \\
% \amazoncatsmall & $92.36$ & $\pm 0.04$ & $92.40$ & $\pm 0.04$ & \boldmath$92.53$ & \boldmath$\pm 0.08$ & \boldmath$0.12$ & \boldmath$\pm 0.02$ & $0.19$ & $\pm 0.00$ & $0.51$ & $\pm 0.06$ \\
% \wikiten & $83.72$ & $\pm 0.12$ & \boldmath$84.17$ & \boldmath$\pm 0.10$ & $83.90$ & $\pm 0.14$ & \boldmath$1.92$ & \boldmath$\pm 0.17$ & $2.87$ & $\pm 0.08$ & $6.46$ & $\pm 0.27$ \\
% \deliciouslarge & $46.24$ & $\pm 0.06$ & $46.30$ & $\pm 0.07$ & \boldmath$46.37$ & \boldmath$\pm 0.10$ & \boldmath$4.53$ & \boldmath$\pm 0.46$ & $10.07$ & $\pm 0.23$ & $17.12$ & $\pm 1.12$ \\
\wikilshtc & $61.96$ & $\pm 0.03$ & $62.78$ & $\pm 0.03$ & \boldmath$63.19$ & \boldmath$\pm 0.03$ & \boldmath$0.51$ & \boldmath$\pm 0.03$ & $0.86$ & $\pm 0.06$ & $1.78$ & $\pm 0.03$ \\
\wikipedia & $66.01$ & $\pm 0.10$ & $66.77$ & $\pm 0.08$ & \boldmath$66.90$ & \boldmath$\pm 0.06$ & \boldmath$2.29$ & \boldmath$\pm 0.03$ & $2.86$ & $\pm 0.07$ & $6.14$ & $\pm 0.11$ \\
\amazon & $42.93$ & $\pm 0.02$ & \boldmath$43.31$ & \boldmath$\pm 0.03$ & $43.25$ & $\pm 0.04$ & \boldmath$1.12$ & \boldmath$\pm 0.06$ & $1.32$ & $\pm 0.08$ & $2.43$ & $\pm 0.08$ \\
\amazonlarge & $45.84$ & $\pm 0.02$ & $46.23$ & $\pm 0.01$ & \boldmath$46.72$ & \boldmath$\pm 0.01$ & \boldmath$1.27$ & \boldmath$\pm 0.10$ & $1.96$ & $\pm 0.05$ & $5.61$ & $\pm 0.23$ \\

\specialrule{0.94pt}{0.4ex}{0.65ex}
& \multicolumn{6}{c|}{$T_{\textrm{train}}$ [h]}
& \multicolumn{6}{c}{$M_{\textrm{size}}$ [GB]} \\
\midrule
% \eurlex & \boldmath$0.00$ & \boldmath$\pm 0.00$ & $0.01$ & $\pm 0.00$ & $0.02$ & $\pm 0.00$ & \boldmath$0.00$ & \boldmath$\pm 0.00$ & $0.00$ & $\pm 0.00$ & $0.00$ & $\pm 0.00$ \\
% \amazoncatsmall & \boldmath$0.17$ & \boldmath$\pm 0.02$ & $0.29$ & $\pm 0.00$ & $0.70$ & $\pm 0.05$ & $0.21$ & $\pm 0.00$ & $0.19$ & $\pm 0.00$ & \boldmath$0.18$ & \boldmath$\pm 0.00$ \\
% \wikiten & \boldmath$0.10$ & \boldmath$\pm 0.00$ & $0.11$ & $\pm 0.00$ & $0.30$ & $\pm 0.02$ & $0.12$ & $\pm 0.00$ & \boldmath$0.06$ & \boldmath$\pm 0.00$ & $0.08$ & $\pm 0.00$ \\
% \deliciouslarge & \boldmath$3.08$ & \boldmath$\pm 0.13$ & $5.51$ & $\pm 0.29$ & $8.03$ & $\pm 0.57$ & $2.30$ & $\pm 0.00$ & $1.82$ & $\pm 0.00$ & \boldmath$1.56$ & \boldmath$\pm 0.00$ \\
\wikilshtc & \boldmath$1.54$ & \boldmath$\pm 0.03$ & $1.60$ & $\pm 0.06$ & $2.28$ & $\pm 0.10$ & $1.20$ & $\pm 0.00$ & $0.97$ & $\pm 0.00$ & \boldmath$0.84$ & \boldmath$\pm 0.00$ \\
\wikipedia & \boldmath$7.41$ & \boldmath$\pm 0.14$ & $9.48$ & $\pm 0.33$ & $19.32$ & $\pm 0.31$ & $2.23$ & $\pm 0.00$ & $1.78$ & $\pm 0.00$ & \boldmath$1.49$ & \boldmath$\pm 0.00$ \\
\amazon & \boldmath$0.39$ & \boldmath$\pm 0.01$ & $0.40$ & $\pm 0.01$ & $0.68$ & $\pm 0.02$ & $0.80$ & $\pm 0.00$ & $0.63$ & $\pm 0.00$ & \boldmath$0.51$ & \boldmath$\pm 0.00$ \\
\amazonlarge & \boldmath$4.80$ & \boldmath$\pm 0.29$ & $5.44$ & $\pm 0.13$ & $12.45$ & $\pm 0.74$ & $11.80$ & $\pm 0.00$ & $9.86$ & $\pm 0.00$ & \boldmath$8.60$ & \boldmath$\pm 0.00$ \\

\specialrule{0.94pt}{0.4ex}{0.65ex}
\end{tabular}}}
\caption{Results arity equal to 2 and pre-leaf node degree equal to 25, 100, or 400.}
\label{tab:k-means-tree-maxleaves-sh}
\end{subtable}

\caption{\Precatk{1}, average prediction time per example, training time and model size for $k$-means trees if different depths with squared hinge loss.}
\label{tab::k-means-tree-depths-sh}
\end{table}